\documentclass[11pt, reqno]{amsart}
\usepackage{txfonts}
\usepackage{amsmath,amssymb,amsthm,mathrsfs,enumerate,bm,xcolor,multirow,pbox}
\usepackage{mathrsfs}
\usepackage{algorithm,algorithmic,float}
\usepackage{graphicx,color,framed,tikz,caption,subcaption}
\usepackage{enumitem}
\setlist{leftmargin=9mm}
\usepackage[colorlinks,linkcolor=black,citecolor=black,urlcolor=black]{hyperref}
\allowdisplaybreaks[4]
\numberwithin{equation}{section}
\newcommand{\N}{\mathbb{N}}
\newcommand{\R}{\mathbb{R}}

\newcommand{\pnorm}[2]{\lVert #1\rVert_{#2}}
\newcommand{\bigpnorm}[2]{\big\lVert#1\big\rVert_{#2}}
\newcommand{\biggpnorm}[2]{\bigg\lVert#1\bigg\rVert_{#2}}
\newcommand{\abs}[1]{\lvert#1\rvert}
\newcommand{\bigabs}[1]{\big\lvert#1\big\rvert}
\newcommand{\biggabs}[1]{\bigg\lvert#1\bigg\rvert}
\newcommand{\iprod}[2]{\langle#1,#2\rangle}
\newcommand{\bigiprod}[2]{\big\langle#1,#2\big\rangle}
\newcommand{\biggiprod}[2]{\bigg\langle#1,#2\bigg\rangle}

\renewcommand{\epsilon}{\varepsilon}

\renewcommand{\d}[1]{\mathrm{d}#1}

\newcommand{\smallo}{\mathfrak{o}}
\newcommand{\bigo}{\mathcal{O}}

\renewcommand{\tilde}{\widetilde}

\DeclareMathOperator{\E}{\mathbb{E}}
\DeclareMathOperator{\Prob}{\mathbb{P}}

\DeclareMathOperator{\sign}{\texttt{sgn}}

\DeclareMathOperator{\var}{Var}
\DeclareMathOperator{\cov}{Cov}

\DeclareMathOperator{\op}{op}

\DeclareMathOperator{\err}{\texttt{err}}

\DeclareMathOperator{\proj}{\mathsf{P}}

\DeclareMathOperator*{\argmin}{arg\,min\,}

\theoremstyle{definition}\newtheorem{problem}{Problem}[section]
\theoremstyle{definition}\newtheorem{definition}[problem]{Definition}
\theoremstyle{remark}\newtheorem{assumption}{Assumption}

\theoremstyle{remark}\newtheorem{remark}{Remark}
\theoremstyle{definition}
\theoremstyle{plain}\newtheorem{theorem}[problem]{Theorem}
\theoremstyle{plain}\newtheorem{question}{Question}
\theoremstyle{plain}\newtheorem{lemma}[problem]{Lemma}
\theoremstyle{plain}\newtheorem{proposition}[problem]{Proposition}
\theoremstyle{plain}\newtheorem{corollary}[problem]{Corollary}
\theoremstyle{plain}

\AtBeginDocument{%
	\def\MR#1{}
}

\begin{document}

\title[Nonconvex gradient descent]{Long-time dynamics and universality of nonconvex gradient descent}

\author[Q. Han]{Qiyang Han}

\address[Q. Han]{
Department of Statistics, Rutgers University, Piscataway, NJ 08854, USA.
}
\email{qh85@stat.rutgers.edu}

\date{\today}


\begin{abstract}
Nonconvex gradient descent is ubiquitous in modern statistical learning.  Yet, a precise long-time characterization of its full trajectory until algorithmic convergence or divergence remains elusive.

This paper develops a general approach to characterize the long-time trajectory behavior of nonconvex gradient descent in generalized single-index models in the large aspect ratio regime. In this regime, we show that for each iteration the gradient descent iterate concentrates around a deterministic vector called the \emph{Gaussian theoretical gradient descent}, whose dynamics can be tracked by a state evolution system of two recursive equations for two scalars. Our concentration guarantees hold universally for a broad class of design matrices and remain valid over long time horizons until algorithmic convergence or divergence occurs. Moreover, our approach reveals that gradient descent iterates are in general approximately independent of the data and strongly incoherent with the feature vectors, a phenomenon previously known as the `implicit regularization' effect of gradient descent in specific models under Gaussian data.

As an illustration of the utility of our general theory, we present two applications of different natures in the regression setting. In the first, we prove global convergence of nonconvex gradient descent with general independent initialization for a broad class of structured link functions, and establish universality of randomly initialized gradient descent in phase retrieval for large aspect ratios. In the second, we develop a data-free iterative algorithm for estimating state evolution parameters along the entire gradient descent trajectory, thereby providing a low-cost yet statistically valid tool for practical tasks such as hyperparameter tuning and runtime determination.

As a by-product of our analysis, we show that in the large aspect ratio regime, the Gaussian theoretical gradient descent coincides with a recent line of dynamical mean-field theory for gradient descent over the constant-time horizon.
\end{abstract}

\maketitle

\setcounter{tocdepth}{1}
\tableofcontents

\sloppy

\section{Introduction}

\subsection{Overview}
Suppose we observe i.i.d. data pairs $(X_i,Y_i)\in \R^{n}\times \R$ from the generalized single-index model
\begin{align}\label{def:model}
Y_i = \mathcal{F}\big(\iprod{X_i}{\mu_\ast},\xi_i\big),\quad i \in [m].
\end{align}
Here $\mathcal{F}:\R^2\to \R$ is a model function and the $\xi_i$'s are unobserved statistical errors. The model (\ref{def:model}) encompasses a wide range of concrete statistical models, including the commonly studied linear regression, nonlinear single-index regression, logistic regression, etc.; see, e.g., \cite{bellec2025observable} for more details. To fix notation, we reserve $n$ for the signal dimension and $m$ for the sample size throughout this paper. 

Given a loss function $\mathsf{L}:\R^2\to \R$, consider the standard empirical risk minimization problem
\begin{align}\label{def:ERM}
\hat{\mu}\in   \argmin_{\mu \in \R^n } \frac{1}{m}\sum_{i=1}^m \mathsf{L}\big(\iprod{X_i}{\mu},Y_i\big).
\end{align}
When the loss function $\mathsf{L}$ is designed so that $\mu_\ast$ is the global minimizer of the population loss corresponding to (\ref{def:ERM}), it is by now textbook knowledge that, in the large aspect ratio regime\footnote{Here $m\wedge n\geq 16$ is assumed merely for notational simplicity to avoid degeneracy in the notation $\log \log n$ and $\log \log m$.}
\begin{align}\label{def:aspect_ratio}
\phi\equiv m/n\gg 1,\quad m\wedge n\geq 16,
\end{align}
the empirical risk minimizer $\hat{\mu}$ will be close to the unknown signal $\mu_\ast$, and is sometimes even optimal in a suitable statistical sense; cf. \cite{van1996weak,van1998asymptotic,van2000empirical,wainwright2019high}.

On the other hand, as the optimization problem (\ref{def:ERM}) generally does not admit a closed form, the empirical risk minimizer $\hat{\mu}$ is usually computed by iterative algorithms. One of the most widely used iterative algorithms for this purpose is the gradient descent algorithm: with initialization $\mu^{(0)}\in \R^n$ and step size $\eta>0$, we compute, for $t=1,2,\ldots$,
\begin{align}\label{def:intro_grad_des}
\mu^{(t)}=\mu^{(t-1)}-\frac{\eta}{m}\cdot X^\top \partial_1 \mathsf{L}\big(X\mu^{(t-1)},Y\big) \in \R^n.
\end{align}
Here $\partial_1 \mathsf{L}(x,y)\equiv (\partial/\partial x)\mathsf{L}(x,y)$.

Our main interest in this paper lies in challenging situations where the model function $\mathcal{F}$ is highly nonlinear, so that the loss landscape $x\mapsto \partial_1\mathsf{L}(x,y)$ may exhibit arbitrary nonconvexity even if the loss function $\mathsf{L}$ itself is convex. In such scenarios, algorithmic convergence of the gradient descent iterates to $\mu_\ast$ is highly nontrivial, and the problem may even seem hopeless in the worst case due to this arbitrarily nonconvex nature of the loss landscape. 

Fortunately, in the statistical setting (\ref{def:model}), where there is a large amount of randomness in the observed data, the gradient descent iterates (\ref{def:intro_grad_des}) often exhibit a certain deterministic pattern in an `averaged sense' that performs substantially better than the worst-case scenario. 

By now, a common strategy in the literature for understanding such averaged-case behavior of gradient descent and its convergence properties in the regime (\ref{def:aspect_ratio}) is a two-step approach: 
\begin{enumerate}
	\item[(L1)] Establish concentration of the empirical loss (\ref{def:ERM}) around its population counterpart.
	\item[(L2)] Study possible benign landscape properties of the population loss corresponding to the empirical loss function.
\end{enumerate}
See, for example, \cite{loh2012high,candes2015phase,loh2015regularized,balakrishnan2017statistical,daskalakis2017ten,tian2017analytical,jain2017non,mei2018landscape,sun2018geometric,xu2018benefits,chi2019nonconvex,dwivedi2020singularity,chen2021spectral,zhang2022symmetry} for a partial list of papers and surveys with concrete implementations in various specific problems. 

While conceptually both natural and appealing, the aforementioned approach in general does not provide immediate insights into the behavior of nonconvex gradient descent iterates when they stay away from (local) benign regions that enforce fast algorithmic convergence.  

An interesting example pertinent to our model (\ref{def:model}) is randomly initialized gradient descent for (real) phase retrieval ($\mathcal{F}(z_0,\xi_0)=z_0^2+\xi_0$) with the squared loss $\mathsf{L}(x)=x^2/2$; cf. \cite{chen2019gradient}. In this case, the iterates start far from the benign region around $\mu_\ast$, and it is not clear apriori that they will ever enter it. Nonetheless, \cite{chen2019gradient} shows that under the specific Gaussian design ($X$ with i.i.d. $\mathcal{N}(0,1)$ entries), fast convergence still occurs in the regime (\ref{def:aspect_ratio}), up to logarithmic factors. Their analysis departs from the usual population landscape approach (L1)-(L2) and instead relies on a Gaussian-distribution-specific trajectory analysis.

This naturally leads to the following question:

\begin{question}\label{ques:per_ite_char}
	Can we develop a general approach to characterize the precise trajectory behavior of nonconvex gradient descent beyond Gaussian designs until algorithmic convergence or divergence in the regime (\ref{def:aspect_ratio})?
\end{question}

The goal of this paper is to provide an affirmative solution to Question \ref{ques:per_ite_char} for the gradient descent iterates (\ref{def:intro_grad_des}) in the generalized single-index model (\ref{def:model}) under the large aspect ratio regime (\ref{def:aspect_ratio}). In particular, we show that gradient descent (\ref{def:intro_grad_des}) concentrates around a much simpler deterministic vector, which we call the `\emph{Gaussian theoretical gradient descent}', whose dynamics can be tracked by a system of two recursive equations for two scalar parameters. Crucially, our concentration guarantee holds universally for a broad class of design matrices with independent entries and remains valid over long time horizons until algorithmic convergence or divergence occurs.

\subsection{Long-time dynamics and universality}
To formally describe our concentration guarantee for the gradient descent in (\ref{def:intro_grad_des}), let $\{\mathrm{u}_\ast^{(t)}\}\subset \R^n$ be defined recursively as follows: with the initialization $\mathrm{u}_\ast^{(0)}\equiv \mu^{(0)}$, for $t=1,2,\ldots$, let 
\begin{align}\label{def:intro_u_ast}
\mathrm{u}_\ast^{(t)}\equiv \big(1-\eta  \tau_\ast^{(t-1)}\big)\cdot \mathrm{u}_\ast^{(t-1)} + \eta \delta_\ast^{(t-1)}\cdot \mu_\ast \in \R^n,
\end{align}
where the scalar parameters $(\tau_\ast^{(t-1)},\delta_\ast^{(t-1)})\in \R^2$ are defined via $\mathrm{u}_\ast^{(t-1)}$ according to formula (\ref{def:tau_delta}) (with $\eta_t\equiv \eta$ therein). As will be shown in Proposition \ref{prop:u_Z_u_ast}, $\mathrm{u}_\ast^{(t)}$ coincides with the gradient descent obtained under the population loss with Gaussian data $X_i\equiv \mathsf{Z}_i\sim \mathcal{N}(0,I_n)$. It is therefore referred to as the \emph{Gaussian theoretical gradient descent} in what follows.

With $\{\mathrm{u}_\ast^{(t)}\}$ defined in (\ref{def:intro_u_ast}), we show in Theorem \ref{thm:grad_des_se} that the gradient descent iterate $\mu^{(t)}$ concentrates around $\mathrm{u}_\ast^{(t)}$: suppose that the data matrix $X$ has i.i.d. mean-zero, unit-variance entries; then, for some constant $M_t>0$ depending on the iteration number $t$,
\begin{align}\label{eqn:intro_gd_conc}
\pnorm{\mu^{(t)}-\mathrm{u}_\ast^{(t)}}{}\leq \frac{\mathrm{polylog}(n)}{(n\wedge \phi)^{1/2}}\cdot M_t,\quad \hbox{with high probability}.
\end{align}
As a consequence of (\ref{eqn:intro_gd_conc}), the gradient descent iterate $\mu^{(t)}$ concentrates around the deterministic vector $\mathrm{u}_\ast^{(t)}$ regardless of the detailed distribution of the entries of $X$, a phenomenon usually referred to as \emph{universality} in random matrix theory \cite{tao2012topics}. In particular, this universality principle allows us to capture the behavior of gradient descent $\mu^{(t)}$ for general design matrices $X$, under an idealized Gaussian data setting in which the dynamics can be computed precisely via Gaussian-distribution-specific methods.  

The proximity of $\mu^{(t)}$ and $\mathrm{u}_\ast^{(t)}$ in (\ref{eqn:intro_gd_conc}) can be understood from a different angle. Indeed, since the Gaussian theoretical gradient descent $\mathrm{u}_\ast^{(t)}$ is obtained under the population loss, (\ref{eqn:intro_gd_conc}) naturally suggests that the gradient descent iterate $\mu^{(t)}$ is almost `independent' of the data $X$. We prove in Theorem \ref{thm:grad_des_incoh} that such approximate independence holds in the sense that 
\begin{align}\label{def:incoherence}
\max_{i \in [m]} \abs{\iprod{X_i}{\mu^{(t)}}}\leq \mathrm{polylog}(n)\cdot M_t,\quad\hbox{with high probability}.
\end{align}
Interestingly, the quantitative formulation (\ref{def:incoherence}) of approximate independence for gradient descent has been confirmed and serves as a key analytic component in the Gaussian analysis of the aforementioned phase retrieval example in \cite{chen2019gradient}, and other matrix models in \cite{ma2020implicit}, under the terminology of `incoherence' or `implicit regularization'. Here our (\ref{def:incoherence}) shows that such approximate independence/incoherence is a much more general phenomenon for gradient descent that is not specific to particular models and the Gaussian design as in \cite{chen2019gradient,ma2020implicit}.

As mentioned earlier, an important technical aspect of our theory in (\ref{eqn:intro_gd_conc})-(\ref{def:incoherence}) lies in its long-time validity through the control of $M_t$. In particular, in the absence of algorithmic convergence, we show that $M_t \leq c^t$ for some constant $c>1$, so our theory (\ref{eqn:intro_gd_conc})-(\ref{def:incoherence}) remains valid for up to $\epsilon \log n$ iterations for sufficiently small $\epsilon>0$ in the general nonconvex gradient descent (\ref{def:intro_grad_des}). This represents the best attainable threshold for the permissible range of iterations, since nonconvex gradient descent may quickly diverge once $t \gg \log n$.

A more interesting scenario arises when algorithmic convergence eventually holds, but nonconvex gradient descent may spend an excessive amount of time searching before entering the final convergence stage. In such cases, we show that $M_t$ can be controlled via the cumulative progress of the Gaussian theoretical gradient descent ${\mathrm{u}_\ast^{(t)}}$ and the associated accumulated universality cost. A detailed formulation is given in Theorem~\ref{thm:nonconvex_generic_conv_long}.

As an illustration of the utility of our theory \eqref{eqn:intro_gd_conc}-\eqref{def:incoherence} and its long-time validity mentioned above, we specialize to the regression setting from \eqref{def:model} and develop two applications of distinct natures:
\begin{itemize}
	\item (\emph{Theoretical application}). Our theory \eqref{eqn:intro_gd_conc}-\eqref{def:incoherence} is used to establish global convergence of gradient descent with general independent initialization for a broad class of structured link functions. In particular, we prove universality of randomly initialized gradient descent in phase retrieval~\cite{chen2019gradient} for large aspect ratios $\phi\gg n$, whereas numerical experiments indicate certain unique subtleties of algorithmic (non-)universality in the moderate aspect ratio regime $1\ll \phi\ll n$.
	\item (\emph{Statistical application}). Our theory \eqref{eqn:intro_gd_conc}-\eqref{def:incoherence} is used to develop a data-free iterative algorithm for estimating the state evolution parameters $\{(\tau_\ast^{(t-1)}, \delta_\ast^{(t-1)})\}$ along the entire gradient descent trajectory. The proposed algorithm has low computational complexity that is independent of the problem size and requires no detailed knowledge of the unknown signal $\mu_\ast$. From a practical standpoint, these low-cost estimates of the state evolution parameters can be readily applied to tasks such as hyperparameter tuning of the step size and determining the algorithmic running time.	
\end{itemize}
From a technical perspective, the key ingredient of our approach to \eqref{eqn:intro_gd_conc}-\eqref{def:incoherence} relies on control of the smallest eigenvalues of the following two $n\times n$ symmetric matrices\footnote{Here the expectation is taken over independent variables $U\sim \mathrm{Unif}[0,1]$, $\mathsf{Z}_n\sim \mathcal{N}(0,I_n)$, and $\pi_m\sim \mathrm{Unif}[1:m]$. }
\begin{align}\label{def:intro_dyn_mat}
\begin{cases}
\mathbb{M}_{ \mathrm{u}_{\ast}^{(t)},\mu_\ast }^{\mathsf{Z}}\equiv \E^{(0)}   \,\partial_{11}\mathsf{L}\big(\bigiprod{\mathsf{Z}_n}{U\mathrm{u}_{\ast}^{(t)}+(1-U) \mu_\ast }, \mathcal{F}(\iprod{\mathsf{Z}_n}{\mu_\ast},\xi_{\pi_m})\big) \mathsf{Z}_n\mathsf{Z}_n^\top,\\
\mathbb{M}_{ \mathrm{u}_{\ast}^{(t)},\mathrm{u}_{\ast}^{(t)} }^{\mathsf{Z}}\equiv \E^{(0)}   \,\partial_{11}\mathsf{L}\big(\bigiprod{\mathsf{Z}_n}{\mathrm{u}_{\ast}^{(t)}}, \mathcal{F}(\iprod{\mathsf{Z}_n}{\mu_\ast},\xi_{\pi_m})\big) \mathsf{Z}_n\mathsf{Z}_n^\top,
\end{cases}
\end{align}
during the initial gradient descent search phase, up to $t_n=\bigo(\mathrm{polylog}(n))$ iterations. 

The relevance of the matrices $\big\{\mathbb{M}_{ \mathrm{u}_{\ast}^{(t)},\mu_\ast }^{\mathsf{Z}}\big\}$ in (\ref{def:intro_dyn_mat}) is easier to explain: under mild conditions on the loss function $\mathsf{L}$, it holds for $t=1,2,\ldots$ that 
\begin{align*}
\mathrm{u}_\ast^{(t)}-\mu_\ast = \big(I_n-\eta\cdot \mathbb{M}_{ \mathrm{u}_{\ast}^{(t-1)},\mu_\ast }^{\mathsf{Z}}\big)\, (\mathrm{u}_\ast^{(t-1)}-\mu_\ast).
\end{align*}
In words, control of the smallest eigenvalues of the matrices $\big\{\mathbb{M}_{ \mathrm{u}_{\ast}^{(t)},\mu_\ast }^{\mathsf{Z}}\big\}$ is equivalent to requiring that $\{\mathrm{u}_\ast^{(t)}\}$ do not blow up during the initial search phase $t \in [1:t_n]$. On the other hand, the role of the matrices $\big\{\mathbb{M}_{ \mathrm{u}_{\ast}^{(t)},\mathrm{u}_{\ast}^{(t)} }^{\mathsf{Z}}\big\}$ is more technical and, as will become clear below, provides a quantitative measure of the universality cost.

\subsection{Further related literature}

\subsubsection{Relation to dynamical mean-field theory}
A recent line of research in dynamical mean-field theory, rooted in statistical physics, develops iteration-wise characterizations of first-order methods in the mean-field regime $\phi = m/n \asymp 1$; cf. \cite{celentano2021high,gerbelot2024rigorous,han2025entrywise,fan2025dynamical,reeves2025dimension}. For gradient descent (\ref{def:intro_grad_des}), these works show that, if $X$ has i.i.d.\ mean-zero, unit-variance entries, then for $t=\bigo(1)$ or growing very slowly,
\begin{align}\label{eqn:intro_gd_mean_field}
\big(n^{1/2}\mu_{\ell}^{(t)}\big)_{\ell \in [n]} \stackrel{d}{\approx} \big(\Omega_{t;\ell}(\mathfrak{W}^{([1:t])})\big)_{\ell \in [n]},
\end{align}
where $\{\Omega_{t;\ell}:\R^t\to\R\}_{\ell\in[n]}$ are deterministic functions and $\{\mathfrak{W}^{([1:t])}\}$ are Gaussian random variables. 

As will be clear in Section \ref{section:mean_field_gd}, the dynamical mean-field theory (\ref{eqn:intro_gd_mean_field}) coincides with our theory \eqref{eqn:intro_gd_conc} for large aspect ratios $\phi\gg 1$ whenever $t=\bigo(1)$. Unfortunately, this short-time validity, in which the prescribed distributional approximation holds when $t=\bigo(1)$ or grows very slowly, presents a major limitation of \eqref{eqn:intro_gd_mean_field}. For instance, for gradient descent in phase retrieval, the random initialization $\mu^{(0)}\sim \mathcal{N}(0,I_n/n)$ is effectively treated as the trivial initialization $\mu^{(0)}=\bm{0}_n$, thereby leading to degeneracy in the state evolution and failing to capture the dynamics of gradient descent. Identifying conditions under which mean-field theory (\ref{eqn:intro_gd_mean_field}) yields long-time validity for gradient descent remains an open problem.

\subsubsection{Precise trajectory characterizations for other iterative algorithms}

A number of recent works have also obtained precise trajectory characterizations for various related models and iterative algorithms. Notable examples include:
\begin{itemize}
	\item other gradient descent methods in (generalized) linear models \cite{celentano2021high,fan2025dynamical,han2025entrywise};
	\item gradient descent methods for low-rank matrices \cite{chen2019gradient,ma2020implicit};
	\item gradient descent/flow for neural network training \cite{bordelon2023self,bordelon2024dynamics,han2025precise};
	\item stochastic gradient descent methods \cite{mignacco2020dynamical,benarous2021online,benarous2024high,gerbelot2024rigorous,collins2024hitting};
	\item approximate message passing algorithms \cite{bayati2011dynamics,bayati2015universality,rush2018finite,berthier2020state,chen2021universality,fan2022approximate,li2022non,dudeja2023universality,li2023approximate,wang2024universality,bao2025leave,lovig2025universality,reeves2025dimension};
	\item iterative algorithms beyond first-order methods \cite{chandrasekher2023sharp,chandrasekher2024alternating,lou2024hyperparameter,celentano2025state,okajima2025asymptotic};
	\item EM algorithms for Gaussian mixture models \cite{wu2021randomly}.
\end{itemize}
Among these works, long-time validity has been established for algorithms that employ sample splitting or fresh samples at each iteration \cite{benarous2021online,chandrasekher2023sharp,benarous2024high,collins2024hitting,chandrasekher2024alternating,lou2024hyperparameter}, or through case-/distribution-specific analysis \cite{rush2018finite,wu2021randomly,li2022non,li2023approximate,reeves2025dimension}. Universality for fully iterative algorithms has been rigorously confirmed in a few cases \cite{celentano2021high,chen2021universality,dudeja2023universality,wang2024universality,fan2025dynamical,han2025entrywise,han2025precise,lovig2025universality}. Related universality results for empirical risk minimizers of the type (\ref{def:ERM}) can be found, e.g., in \cite{panahi2017universal,elkaroui2018impact,oymak2018universality,abbasi2019universality,montanari2022universality,han2023universality,han2023distribution,montanari2023universality,dudeja2024spectral}. Achieving both long-time validity and universality for many of the above iterative algorithms remains a challenging open problem.

\subsection{Organization}
The rest of the paper is organized as follows. The abstract version of our theory \eqref{eqn:intro_gd_conc}-\eqref{def:incoherence} is presented in Section~\ref{section:main_results}. A long-time version of our general theory, as discussed above, is developed in Section~\ref{section:long_time_dynamics}. The prescribed two applications to the single-index regression model are given in Section~\ref{section:concrete_models}. The relation of our theory to mean-field dynamics is discussed in detail in Section~\ref{section:mean_field_gd}. Proofs are collected in Sections~\ref{section:proof_main_results}-\ref{section:proof_mean_field_dyn_approx} and in the appendices.

\subsection{Notation}
For any two integers $m,n$, let $[m:n]\equiv \{m,m+1,\ldots,n\}$ and $[n]\equiv [1:n]$. When $m>n$, it is understood that $[m:n]=\emptyset$.  

For $a,b \in \R$, $a\vee b\equiv \max\{a,b\}$ and $a\wedge b\equiv\min\{a,b\}$. For $a \in \R$, let $a_\pm \equiv (\pm a)\vee 0$. For a multi-index $a \in \mathbb{Z}_{\geq 0}^n$, let $\abs{a}\equiv \sum_{i \in [n]}a_i$. For $x \in \R^n$, let $\pnorm{x}{p}$ denote its $p$-norm $(0\leq p\leq \infty)$, and $B_{n;p}(R)\equiv \{x \in \R^n: \pnorm{x}{p}\leq R\}$. We simply write $\pnorm{x}{}\equiv\pnorm{x}{2}$ and $B_n(R)\equiv B_{n;2}(R)$. For $x \in \R^n$, let $\mathrm{diag}(x)\equiv (x_i\bm{1}_{i=j})_{i,j \in [n]} \in \R^{n\times n}$. 

For a matrix $M \in \R^{m\times n}$, let $\pnorm{M}{\op},\pnorm{M}{F}$ denote the spectral and Frobenius norms of $M$, respectively. $I_n$ is reserved for the $n\times n$ identity matrix, written simply as $I$ (in the proofs) if no confusion arises. 

We use $C_{x}$ to denote a generic constant that depends only on $x$, whose numerical value may change from line to line unless otherwise specified. $a\lesssim_{x} b$ and $a\gtrsim_x b$ mean $a\leq C_x b$ and $a\geq C_x b$, abbreviated as $a=\bigo_x(b)$ and $a=\Omega_x(b)$, respectively; $a\asymp_x b$ means both $a\lesssim_{x} b$ and $a\gtrsim_x b$. $\bigo$ and $\smallo$ (resp. $\mathcal{O}_{\mathbf{P}}$ and $\mathfrak{o}_{\mathbf{P}}$) denote the usual big-O and small-o notation (resp. in probability). By convention, sums and products over an empty set are understood as $\Sigma_{\emptyset}(\cdots)=0$ and $\Pi_{\emptyset}(\cdots)=1$. 

For a random variable $X$, we use $\Prob_X,\E_X$ (resp. $\Prob^X,\E^X$) to indicate that the probability and expectation are taken with respect to $X$ (resp. conditional on $X$). 

For $\Lambda>0$ and $\mathfrak{p}\in \N$, a measurable map $f:\R^n \to \R$ is called \emph{$\Lambda$-pseudo-Lipschitz of order $\mathfrak{p}$} iff $
\abs{f(x)-f(y)}\leq \Lambda\cdot  (1+\pnorm{x}{}+\pnorm{y}{})^{\mathfrak{p}-1}\cdot\pnorm{x-y}{}$ holds for all  $x,y \in \R^{n}$.
Moreover, $f$ is called \emph{$\Lambda$-Lipschitz} iff $f$ is $\Lambda$-pseudo-Lipschitz of order $1$, and in this case we often write $\pnorm{f}{\mathrm{Lip}}\leq L$, where $\pnorm{f}{\mathrm{Lip}}\equiv \sup_{x\neq y} \abs{f(x)-f(y)}/\pnorm{x-y}{}$.

\section{Main abstract results}\label{section:main_results}

\subsection{Further notation and definitions}

We consider a slightly more general gradient descent scheme that allows for varying step sizes across iterations: with initialization $\mu^{(0)}\in \R^n$ and step sizes $\{\eta_t\}\subset \R_{>0}$, we compute, for $t=1,2,\ldots$,
\begin{align}\label{def:grad_des}
\mu^{(t)}=\mu^{(t-1)}-\frac{\eta_{t-1}}{m}\cdot X^\top \partial_1 \mathsf{L}\big(X\mu^{(t-1)},Y\big) \in \R^n.
\end{align}
Throughout the paper, we use
\begin{align*}
\E^{(0)}[\cdot] \equiv \E\big[\cdot|\,\mu^{(0)},\mu_\ast,\xi\big].
\end{align*}
We also adopt the following notation.
\begin{definition}\label{def:G_M}
\begin{enumerate}
	\item Let $\mathfrak{S}:\R^3 \to \R$ be defined as
	\begin{align*}
	\mathfrak{S}(x_0,z_0,\xi_0)\equiv \partial_1 \mathsf{L}\big(x_0, \mathcal{F}(z_0,\xi_0)\big),\quad \forall (x_0,z_0,\xi_0) \in \R^3.
	\end{align*}
	\item For two generic vectors $\mathrm{u},\mathrm{v} \in \R^n$, with $U\sim \mathrm{Unif}\,[0,1]$ and $\pi_m\sim \mathrm{Unif}\,[1:m]$ independent of all other variables, let
	\begin{align*}
	M_{\mathrm{u},\mathrm{v} }(X)&\equiv \E_{U,\pi_m}   \,\partial_1\mathfrak{S}\big(\bigiprod{X_{\pi_m}}{U\mathrm{u}+(1-U) \mathrm{v} }, \iprod{X_{\pi_m}}{\mu_\ast},\xi_{\pi_m}\big) X_{\pi_m}X_{\pi_m}^\top,\\
	\mathbb{M}_{\mathrm{u},\mathrm{v} }^X&\equiv \E^{(0)}   \,\partial_1\mathfrak{S}\big(\bigiprod{X_{\pi_m}}{U\mathrm{u}+(1-U) \mathrm{v} }, \iprod{X_{\pi_m}}{\mu_\ast},\xi_{\pi_m}\big) X_{\pi_m}X_{\pi_m}^\top.
	\end{align*} 
	Here the expectation $\E^{(0)}$ is taken with respect to $X,U,\pi_m$.
	\item For two sequences of generic vectors $\{\mathrm{u}^{(s)}\},\{\mathrm{v}^{(s)}\} \subset \R^n$, and step sizes $\{\eta_s\} \subset \R$, where $s \in \mathbb{Z}_{\geq 0}$, let
	\begin{align*}
	\mathscr{M}_{ \mathrm{u}^{(\cdot)},\mathrm{v}^{(\cdot)} }^{(t),X}&\equiv 1+ \max_{\tau \in [0:t]}\sum_{s \in [0:\tau]} \bigg(\prod_{r \in [s:\tau]}\bigpnorm{I_n - \eta_r\cdot \mathbb{M}_{\mathrm{u}^{(r)},\mathrm{v}^{(r)}}^X }{\op}\bigg),\\
	\mathfrak{M}_{ \mathrm{u}^{(\cdot)},\mathrm{v}^{(\cdot)} }^{(t)}(X)&\equiv 1+ \max_{\tau \in [0:t]}\sum_{s \in [0:\tau]} \bigg(\prod_{r \in [s:\tau]}\bigpnorm{I_n - \eta_r\cdot M_{\mathrm{u}^{(r)},\mathrm{v}^{(r)}}(X) }{\op}\bigg).
	\end{align*}
	Here we usually omit the notational dependence on the step sizes $\{\eta_s\}$ for simplicity.
\end{enumerate}
\end{definition}
From the definition, we clearly have $\mathbb{M}_{\mathrm{u},\mathrm{v} }^X=\E^{(0)} M_{\mathrm{u},\mathrm{v} }(X)$.

\begin{definition}\label{def:state_evolution}
Let the \emph{state evolution} $\{\mathrm{u}_\ast^{(t)}\}$ be defined as follows: initialized with $\mathrm{u}_\ast^{(0)}\equiv \mu^{(0)}\in \R^n$, for $t=1,2,\ldots$, 
\begin{align}\label{def:u_ast}
\mathrm{u}_\ast^{(t)}\equiv \big(1-\eta_{t-1}  \tau_\ast^{(t-1)}\big)\cdot \mathrm{u}_\ast^{(t-1)} + \eta_{t-1} \delta_\ast^{(t-1)}\cdot \mu_\ast \in \R^n,
\end{align}
where with $\mathsf{Z}_n\sim \mathcal{N}(0,I_n)$,
\begin{align}\label{def:tau_delta}
\begin{cases}
\tau_\ast^{(t-1)}\equiv \E^{(0)}   \partial_1\mathfrak{S}\big(\iprod{\mathsf{Z}_n}{\mathrm{u}_{\ast}^{(t-1)}},\iprod{\mathsf{Z}_n}{ \mu_\ast},\xi_{\pi_m}\big)\in \R,\\
\delta_\ast^{(t-1)}\equiv -\E^{(0)}   \partial_2\mathfrak{S}\big(\iprod{\mathsf{Z}_n}{\mathrm{u}_{\ast}^{(t-1)}},\iprod{\mathsf{Z}_n}{ \mu_\ast},\xi_{\pi_m}\big)\in \R.
\end{cases}
\end{align}
\end{definition}
The state evolution $\{\mathrm{u}_\ast^{(t)}\}$ will play a pivotal role in this paper in understanding the behavior of the actual gradient descent $\{\mu^{(t)}\}$. Clearly, by iterating Eq. (\ref{def:u_ast}), $\mathrm{u}_\ast^{(t)}$ can be represented as a linear combination of the initialization $\mu^{(0)}$ and the true signal $\mu_\ast$.

For technical reasons, we also define a` theoretical version' of the gradient descent $\{\mu^{(t)}\}$ as follows.

\begin{definition}
\begin{enumerate}
	\item Let the \emph{theoretical gradient descent} $\{\mathrm{u}_X^{(t)} \}$ be defined as follows: initialized with $\mathrm{u}_X^{(0)}\equiv \mu^{(0)}\in \R^n$, for $t=1,2,\ldots$, let
	\begin{align}\label{def:theo_grad_des}
	\mathrm{u}_X^{(t)} \equiv \mathrm{u}_X^{(t-1)} - \frac{\eta_{t-1}}{m} \cdot \E^{(0)}  X^\top \mathfrak{S}\big(X\mathrm{u}_X^{(t-1)},X \mu_\ast,\xi\big)\in \R^n.
	\end{align}
	Here $\mathfrak{S}$ is understood as applied componentwise.
	\item With $\mathsf{Z}\equiv \mathsf{Z}_{m\times n}$ denoting an $m\times n$ matrix with i.i.d. $\mathcal{N}(0,1)$ entries, $\{\mathrm{u}_{\mathsf{Z}}^{(t)} \}$ is called the \emph{Gaussian theoretical gradient descent}.
\end{enumerate}
\end{definition}

The Gaussian theoretical gradient descent $\{\mathrm{u}_{\mathsf{Z}}^{(t)} \}$ defined above provides an interpretation of the state evolution $\{\mathrm{u}_\ast^{(t)}\}$ in Definition \ref{def:state_evolution}.

\begin{proposition}\label{prop:u_Z_u_ast}
$\mathrm{u}_{\mathsf{Z}}^{(t)} =\mathrm{u}_\ast^{(t)}$ holds for all $t=0,1,2,\ldots$.
\end{proposition}
\begin{proof}
	For any $k \in [n]$, by Gaussian integration-by-parts,
	\begin{align*}
	\mathrm{u}_{\mathsf{Z},k}^{(t)} &= \mathrm{u}_{\mathsf{Z},k}^{(t-1)} - \eta_{t-1} \cdot \E^{(0)}  \mathsf{Z}_{n,k} \,\mathfrak{S}\big(\iprod{\mathsf{Z}_n}{\mathrm{u}_{\mathsf{Z}}^{(t-1)}},\iprod{\mathsf{Z}_n}{ \mu_\ast},\xi_{\pi_m}\big)\\
	& = \Big(1 - \eta_{t-1} \cdot  \E^{(0)}   \partial_1\mathfrak{S}\big(\iprod{\mathsf{Z}_n}{\mathrm{u}_{\mathsf{Z}}^{(t-1)}},\iprod{\mathsf{Z}_n}{ \mu_\ast},\xi_{\pi_m}\big)\Big)\cdot \mathrm{u}_{\mathsf{Z},k}^{(t-1)}\\
	&\qquad -\eta_{t-1}\cdot \E^{(0)}   \partial_2\mathfrak{S}\big(\iprod{\mathsf{Z}_n}{\mathrm{u}_{\mathsf{Z}}^{(t-1)}},\iprod{\mathsf{Z}_n}{ \mu_\ast},\xi_{\pi_m}\big)\cdot \mu_{\ast,k}.
	\end{align*}
	The claim follows by comparing the above display to the definition of $\{\mathrm{u}_\ast^{(t)}\}$ in Definition \ref{def:state_evolution}.
\end{proof}

In view of the proposition above, we shall use `state evolution' and `Gaussian theoretical gradient descent' interchangeably for $\{\mathrm{u}_{\mathsf{Z}}^{(t)}=\mathrm{u}_{\ast}^{(t)}\}$ in the sequel.

\subsection{Assumptions}

We now state the assumptions for the general results below.

\begin{assumption}\label{assump:abstract}
Suppose the following hold for some $K,\Lambda\geq 2$ and $\mathfrak{p}\geq 1$:
\begin{enumerate}
	\item[(A1)] $\{X_{ij}\}_{i \in [m], j \in [n]}$ are independent random variables with $\E X_{ij}=0, \var(X_{ij})=1$ and  $ \max_{i \in [m], j \in [n]}\pnorm{X_{ij}}{\psi_2}\leq K$.
	\item[(A2)] For all $i \in [m]$ and $\alpha\in \mathbb{Z}_{\geq 0}^2, \abs{\alpha} \in [0:4]$, the mapping $(x,z)\mapsto \partial_\alpha \mathfrak{S}(x,z,\xi_i)$ is $\Lambda$-pseudo-Lipschitz of order $\mathfrak{p}$, and 
	$\abs{\partial_\alpha \mathfrak{S}(0,0,\xi_i)}\leq \Lambda$.
\end{enumerate}
\end{assumption}

\begin{remark}
Some technical remarks for Assumption \ref{assump:abstract} are in order:
\begin{enumerate}
	\item (A1) requires a sub-gaussian tail on the entries of the design matrix $X$, where $\pnorm{\cdot}{\psi_2}$ is the standard Orlicz-2/sub-gaussian norm; see, e.g., \cite[Section 2.1]{van1996weak} for a precise definition. We note that our theory below is non-asymptotic and allows the constant $K$ to grow at a certain rate, so the actual moment condition can be further relaxed by a standard truncation argument. 
	\item (A2) assumes sufficient smoothness of $\mathfrak{S}$ primarily to ensure universality. While we believe further technical work may lead to improved regularity conditions, we have opted to work with this condition to avoid excessive technicalities unrelated to the main new techniques developed in this work.
\end{enumerate}
\end{remark}

\subsection{Abstract results}

\subsubsection{Concentration of gradient descent}
The first main abstract result of this paper provides concentration of the gradient descent $\mu^{(t)}$ around the Gaussian theoretical gradient descent $\mathrm{u}_\ast^{(t)}$; its proof can be found in Section \ref{subsection:proof_grad_des_se}.

\begin{theorem}\label{thm:grad_des_se}
Suppose Assumption \ref{assump:abstract} holds for some $K,\Lambda\geq 2$ and $\mathfrak{p}\geq 1$. Then there exist a universal constant $c_0>1$ and another constant $c_{\mathfrak{p}}>1$ such that for any $x\geq 1$, it holds  with $\Prob^{(0)}$-probability at least $1-c_0 nt e^{-x}$ that
\begin{align*}
\pnorm{\mu^{(t)}-\mathrm{u}_{\ast}^{(t)}}{}
&\leq  \big(K\Lambda  L_{\ast}^{(0)} (1\vee \pnorm{\eta_{[0:t)}}{\infty})\cdot x \big)^{c_{\mathfrak{p}}}\cdot \Big(1+\max_{s \in [1:t-1]} n^{1/2}\pnorm{\mathrm{u}_\ast^{(s)}}{\infty}\vee \pnorm{\mathrm{u}_X^{(s)}}{}\Big)^{c_{\mathfrak{p}}} \\
&\qquad\qquad  \times  \Big\{\phi^{-1/2}\cdot \mathfrak{M}_{ \mu^{(\cdot)},\mathrm{u}_X^{(\cdot)} }^{(t-1)}(X)+ n^{-1/2}\cdot  \mathscr{M}_{ \mathrm{u}_X^{(\cdot)},\mathrm{u}_{\ast}^{(\cdot)} }^{(t-1),X} \Big\}.
\end{align*}
Here $
L_{\ast}^{(0)}\equiv 1+n^{1/2}\pnorm{\mu^{(0)}}{\infty}+n^{1/2}\pnorm{\mu_\ast}{\infty}\in \R$.
\end{theorem}

\begin{remark}
Some technical remarks are in order:
\begin{enumerate}
	\item The constants $c_0>1$ and $c_{\mathfrak{p}}>1$ can be explicitly worked out, but we choose to keep their numeric values implicit for simplicity.
	\item The aspect ratio $\phi$ needs to be large, i.e., $\phi \gg 1$, for Theorem \ref{thm:grad_des_se} to be effective. Indeed, when $\phi\asymp 1$, the gradient descent $\mu^{(t)}$ does not concentrate around any deterministic vector, and a different type of theory is needed. The readers are referred to Section \ref{section:mean_field_gd} for more details along this line.
	\item The terms $\mathfrak{M}_{ \mu^{(\cdot)},\mathrm{u}_X^{(\cdot)} }^{(t-1)}(X)$ and $ \mathscr{M}_{ \mathrm{u}_X^{(\cdot)},\mathrm{u}_{\ast}^{(\cdot)} }^{(t-1),X}$ play different roles in the bound. In particular, the term $\mathfrak{M}_{ \mu^{(\cdot)},\mathrm{u}_X^{(\cdot)} }^{(t-1)}(X)$ quantifies the proximity between the actual gradient descent  $\{\mu^{(\cdot)}\}$ and the theoretical gradient descent $\{\mathrm{u}_X^{(\cdot)} \}$, whereas $ \mathscr{M}_{ \mathrm{u}_X^{(\cdot)},\mathrm{u}_{\ast}^{(\cdot)} }^{(t-1),X}$ quantifies the universality cost between the theoretical gradient descent $\{\mathrm{u}_X^{(\cdot)} \}$ and its Gaussian counterpart $\{\mathrm{u}_{\ast}^{(\cdot)} \}$.
\end{enumerate}
\end{remark}

\subsubsection{Incoherence of gradient descent}

To formally state our result on `incoherence', we need the following definition of leave-one-out iterates. 
\begin{definition}
 For $i \in [m]$, let $X_{[-i]}\in \R^{m\times n}$ be the matrix that sets the $i$-th row of $X$ to be zero. Let  $\{\mu^{(t)}_{[-i]}\equiv \mu^{(t)}(X_{[-i]})\}$ be the leave-one-out gradient descent iterates. 
\end{definition}

Now we may state our second abstract result on the incoherence property (\ref{def:incoherence}) for general gradient descent; its proof can be found in Section \ref{subsection:proof_grad_des_incoh}.

\begin{theorem}\label{thm:grad_des_incoh}
	Suppose Assumption \ref{assump:abstract} holds for some $K,\Lambda\geq 2$ and $\mathfrak{p}\geq 1$. Then there exist a universal constant $c_0>1$ and another constant $c_{\mathfrak{p}}>1$ such that for any $x\geq 1$, it holds  with $\Prob^{(0)}$-probability at least $1-c_0 m e^{-x}$ that
	\begin{align*}
	\max_{i \in [m]}\abs{\iprod{X_i}{\mu^{(t)}}}&\leq\big(K\Lambda (1\vee \pnorm{\eta_{[0:t)}}{\infty})\cdot x \big)^{c_{\mathfrak{p}}} \\
	&\qquad \times \Big(1+\max_{s \in [1:t]} \pnorm{\mu_{[-i]}^{(s)}}{}\Big)\cdot  \Big(1+\phi^{-1} \mathfrak{M}_{\mu^{(\cdot)}, \mu_{[-i]}^{(\cdot)} }^{(t-1)}(X)\Big).
	\end{align*}
\end{theorem}

The incoherence property (\ref{def:incoherence}) was previously established for randomly initialized gradient descent in phase retrieval \cite{chen2019gradient}, and for spectrally initialized gradient descent in several related models \cite{ma2020implicit}, both under Gaussian designs. In these works, the incoherence property plays a pivotal role in theoretically validating the fast global convergence of vanilla gradient descent methods without explicit regularization, and is therefore referred to as the `implicit regularization' phenomenon. Our Theorem~\ref{thm:grad_des_incoh} shows that the incoherence property (\ref{def:incoherence}) is in fact an intrinsic feature of a much broader class of nonconvex gradient descent iterates, holding under a wide family of design matrices beyond Gaussian ensembles.

\subsection{A simple corollary}
It is illustrative at this point to work out a preliminary corollary of the above theorems to demonstrate some quantitative features of Theorems \ref{thm:grad_des_se} and \ref{thm:grad_des_incoh}.

\begin{corollary}\label{cor:general_se}
Suppose Assumption \ref{assump:abstract} holds for some $K,\Lambda\geq 2$ and $\mathfrak{p}= 1$. Then there exists some universal constant $c_0>1$ such that with $\Prob^{(0)}$-probability at least $1-m^{-100}$, uniformly in $t\leq m^{100}$,
\begin{align*}
&(n\wedge \phi)^{1/2}\pnorm{\mu^{(t)}-\mathrm{u}_\ast^{(t)}}{}+\max_{i \in [m]}\abs{\iprod{X_i}{\mu^{(t)}}}\\
&\leq \big(K\Lambda  L_{\ast}^{(0)} (1\vee \phi^{-1})\big)^{c_0}\cdot \big(1+ \pnorm{\eta_{[0:t)}}{\infty}\Lambda\big)^{c_0 t}\cdot  (\log m)^{c_0}.
\end{align*}
\end{corollary}
The proof of the above corollary can be found in Section \ref{subsection:proof_general_se}. We note that the rate $\phi^{-1/2}$ cannot be further improved (up to logarithmic factors), while the rate $n^{-1/2}$ arises from the non-Gaussianity of the design. Moreover, the numeric constant $100$ can be replaced by an arbitrarily large constant at the cost of a possibly larger $c_0>0$.

A major limitation of Corollary~\ref{cor:general_se} lies in its permissible range of iterations $t$. In particular, the effective range is typically of order $\log n$ due to the factor $\big(1+ \pnorm{\eta_{[0:t)}}{\infty}\Lambda\big)^{c_0 t}$. Note that the $\log n$ threshold is both natural and unavoidable, since general nonconvex gradient descent may diverge rapidly for $t \gg \log n$ (a situation allowed under the assumptions of Corollary~\ref{cor:general_se}).

In the next section, we show that the effective range of $t$ can be substantially extended when algorithmic convergence is in force. As expected, this is achieved through a sharp, long-time control of the terms $\mathfrak{M}_{ \mu^{(\cdot)},\mathrm{u}_X^{(\cdot)} }^{(\cdot)}(X)$ and $\mathscr{M}_{ \mathrm{u}_X^{(\cdot)},\mathrm{u}_{\ast}^{(\cdot)} }^{(\cdot),X}$.

\section{Long-time dynamics}\label{section:long_time_dynamics}

\subsection{General results}

We consider a general nonconvex setting in which gradient descent converges and its long-time dynamics can be captured by the master Theorems \ref{thm:grad_des_se} and \ref{thm:grad_des_incoh}. 

To state our results, for any $t\geq 1$, with $\epsilon_n\equiv (\log n)^{-100}$, let
\begin{align}\label{def:Bn}
\mathcal{B}_0(t)&\equiv \sum_{s \in [0:t)} \prod_{r \in [s:t)} \Big(1+\eta_r\cdot \Big[\lambda_{\min}\Big(\mathbb{M}_{ \mathrm{u}_{\ast}^{(r)},\mu_\ast }^{\mathsf{Z}}\Big)\Big]_- \Big),\nonumber\\
\mathcal{B}(t)&\equiv \sum_{s \in [0:t)} \prod_{r \in [s:t)} \Big(1+\eta_r\cdot \Big[\lambda_{\min}\Big(\mathbb{M}_{ \mathrm{u}_{\ast}^{(r)},\mathrm{u}_{\ast}^{(r)} }^{\mathsf{Z}}\Big)\Big]_- +\epsilon_n\Big).
\end{align}
We have the following general result.
\begin{theorem}\label{thm:nonconvex_generic_conv_long}
	Fix $\eta_\ast\in (0,1)$ and step sizes $\{\eta_r\}$ such that
	\begin{align}\label{cond:nonconvex_generic_conv_long_eta}
	\eta_\ast\leq \eta_r\leq 1/\max\Big\{1,\pnorm{\mathbb{M}_{ \mathrm{u}_{\ast}^{(r)},\mu_\ast }^{\mathsf{Z}}}{\op},\pnorm{\mathbb{M}_{ \mathrm{u}_{\ast}^{(r)},\mathrm{u}_{\ast}^{(r)} }^{\mathsf{Z}}}{\op}\Big\}.
	\end{align}
	Suppose Assumption \ref{assump:abstract} holds for some $K,\Lambda\geq 2$ and $\mathfrak{p}\geq 1$, and the following hold:
	\begin{enumerate}
		\item[(N1)] $\mu_\ast$ is a stationary point of the Gaussian theoretical gradient descent (\ref{def:theo_grad_des}) (with $X=\mathsf{Z}$).
		\item[(N2)] For some $0\leq t_n\leq 1/\epsilon_n$ and $\sigma_\ast\in (0,1)$,
		\begin{align}\label{cond:nonconvex_generic_conv_long_lowerbdd}
		\inf_{t \geq t_n}\lambda_{\min}\Big(\mathbb{M}_{ \mathrm{u}_{\ast}^{(t)},\mu_\ast }^{\mathsf{Z}}\Big)\geq \sigma_\ast.
		\end{align}
	\end{enumerate}
	Further assume that for some constant $c_0>0$, it holds that
	\begin{align}\label{cond:nonconvex_generic_conv_long_Bn0}
	K\Lambda L_\ast^{(0)}\vee \log m \vee \Big\{\sup_{t\geq 1} n^{1/2}\pnorm{\mathrm{u}_\ast^{(t)}}{\infty}\Big\} \vee \mathcal{B}_0(t_n)\leq (\log n)^{c_0}.
	\end{align}
	 Then there exist constants $c_1=c_1(\mathfrak{p},c_0)>0$ and  $c_2=c_2(\mathfrak{p},c_0,\eta_\ast,\sigma_\ast)>0$ such that if 
	\begin{align}\label{cond:nonconvex_generic_conv_long_Bn}
	\phi\geq (\log m)^{c_1}\hbox{ and }\mathcal{B}(t_n)\leq (n\wedge \phi)^{1/2}/(\log n)^{c_1},
	\end{align}
	on an event with $\Prob^{(0)}$-probability at least $1-m^{-100}$, uniformly for $t\leq m^{100}$, 
	\begin{align}\label{eqn:nonconvex_generic_conv_long}
	(n\wedge \phi)^{1/2}\pnorm{\mu^{(t)}-\mathrm{u}_\ast^{(t)}}{} + \max_{i \in [m]}\abs{\iprod{X_i}{\mu^{(t)}}}\leq (\log n)^{c_2}\cdot \mathcal{B}(t\wedge t_n).
	\end{align}
\end{theorem}

The proof of the above theorem relies on the long-time estimates for $\mathscr{M}_{ \mathrm{u}_X^{(\cdot)},\mathrm{u}_{\ast}^{(\cdot)} }^{(t),X}$ and $\mathfrak{M}_{ \mu^{(\cdot)},\mathrm{u}_X^{(\cdot)} }^{(t)}(X)$ in Section \ref{subsection:long_time_dynamics_proof_sketch} ahead, with all details presented in Sections \ref{subsection:proof_nonconvex_generic_conv_long_lemma}-\ref{subsection:proof_nonconvex_generic_conv_long_complete}. 

We note that all the conditions in Theorem \ref{thm:nonconvex_generic_conv_long} are imposed on the Gaussian theoretical gradient descent iterates $\{\mathrm{u}_{\ast}^{(\cdot)}\}$. A major feature of Theorem \ref{thm:nonconvex_generic_conv_long} is that it allows the gradient descent to be nonconvergent in the initial search phase $t\leq t_n$, where $t_n$ can be as large as $(\log n)^{100}$. For our long-time dynamics to hold, we require $\mathcal{B}_0(t_n)$ and $\mathcal{B}(t_n)$ to be well controlled in the sense of (\ref{cond:nonconvex_generic_conv_long_Bn0})-(\ref{cond:nonconvex_generic_conv_long_Bn}) in this initial search phase:
\begin{itemize}
	\item The matrices $\big\{\mathbb{M}_{ \mathrm{u}_{\ast}^{(t)},\mu_\ast }^{\mathsf{Z}}\big\}$ appearing in $\mathcal{B}_0(t_n)$ directly control the dynamics of the Gaussian theoretical gradient descent iterates $\{\mathrm{u}_\ast^{(t)}\}$, as under (N1),
	\begin{align}\label{eqn:dynamics_gaussian_N1}
	\mathrm{u}_\ast^{(t)}-\mu_\ast = \Big(I_n-\eta_{t-1}\cdot \mathbb{M}_{ \mathrm{u}_{\ast}^{(t-1)},\mu_\ast }^{\mathsf{Z}}\Big)\,\big(\mathrm{u}_\ast^{(t-1)}-\mu_\ast\big),\quad t=1,2,\ldots.
	\end{align}
	Therefore, $\mathcal{B}_0(t_n)$  requires that at least the Gaussian problem is under control via the smallest eigenvalues of $\big\{\mathbb{M}_{ \mathrm{u}_{\ast}^{(t)},\mu_\ast }^{\mathsf{Z}}\big\}$.
	\item The role of the matrices $\big\{\mathbb{M}_{ \mathrm{u}_{\ast}^{(t)},\mathrm{u}_{\ast}^{(t)} }^{\mathsf{Z}}\}$ in $\mathcal{B}(t_n)$ appears more technical, and can be viewed in a certain sense as quantifying the universality between the theoretical gradient descent iterates $\{\mathrm{u}_X^{(t)}\}$ and their Gaussian counterparts $\{\mathrm{u}_{\ast}^{(t)}\}$. 
\end{itemize} 

\subsection{Checking conditions of Theorem \ref{thm:nonconvex_generic_conv_long}}
To apply Theorem \ref{thm:nonconvex_generic_conv_long} in concrete examples, we need to determine the parameters $t_n$, $\sigma_\ast$ in the condition (N2) and verify the growth bounds (\ref{cond:nonconvex_generic_conv_long_Bn0})-(\ref{cond:nonconvex_generic_conv_long_Bn}) for $\mathcal{B}_0(t_n)$ and $\mathcal{B}(t_n)$.

We first develop some understanding of the condition (N2). 

\begin{proposition}\label{prop:nonconvex_generic_conv_long_contraction}
	Suppose (A2) holds for some $\Lambda\geq 2$ and $\mathfrak{p}\geq 1$. Fix step sizes $\{\eta_r\}$ such that (\ref{cond:nonconvex_generic_conv_long_eta}) holds with  $\eta_\ast \in (0,1)$. 
	\begin{enumerate}
		\item If (N2) holds for some $t_n\geq 0$ and $\sigma_\ast \in (0,1)$, then for $t\geq 0$,
		\begin{align*}
		\pnorm{\mathrm{u}_\ast^{(t)}-\mu_\ast}{}\leq \mathcal{B}_0(t_n)\cdot \big(1-\eta_\ast\sigma_\ast\big)^{(t-t_n)_+}\cdot \pnorm{\mu^{(0)}-\mu_\ast}{}.
		\end{align*}
		\item Suppose there exist some $\epsilon_0 \in (0,1)$ and $\tau_n,\mathcal{D}_n \in [0,n^{100}]$ such that for all $t\geq 0$,
		\begin{align*}
		\pnorm{\mathrm{u}_\ast^{(t)}-\mu_\ast}{}\leq \mathcal{D}_n\cdot \big(1-\epsilon_0\big)^{(t-\tau_n)_+}\cdot \pnorm{\mu^{(0)}-\mu_\ast}{}.
		\end{align*}
		With $\mathsf{Z}_1\sim \mathcal{N}(0,1)$, let
		\begin{align}\label{def:varrho_generic}
		\varrho_\ast &\equiv \min\Big\{\E^{(0)}   \partial_1\mathfrak{S}\big(\pnorm{\mu_\ast}{}\mathsf{Z}_1, \pnorm{\mu_\ast}{}\mathsf{Z}_1,\xi_{\pi_m}\big),\nonumber\\
		&\qquad\qquad  \E^{(0)}  \mathsf{Z}_1^2\cdot \partial_1\mathfrak{S}\big(\pnorm{\mu_\ast}{}\mathsf{Z}_1, \pnorm{\mu_\ast}{}\mathsf{Z}_1,\xi_{\pi_m}\big)\Big\}.
		\end{align}
		If $\varrho_\ast>0$ and $\Lambda L_\ast^{(0)}\leq (\log n)^{c_0}$ for some $c_0>1$, then there exists some $c_1=c_1(\mathfrak{p},c_0,\epsilon_0)>0$ such that (N2) is satisfied with $t_n=\tau_n+c_1 (\log \log n+\log \mathcal{D}_n)$ and $\sigma_\ast =\varrho_\ast/2$ for $n\geq n_0(\mathfrak{p},c_0,\epsilon_0,\varrho_\ast)$.
	\end{enumerate}
\end{proposition}

The proof of the above proposition can be found in Section \ref{subsection:proof_nonconvex_generic_conv_long_contraction}. Moreover, the above proposition shows that (\ref{cond:nonconvex_generic_conv_long_lowerbdd}) is essentially equivalent to requiring approximate linear convergence of the Gaussian theoretical gradient descent $\{\mathrm{u}_\ast^{(t)}\}$, modulo the initial $t_n$ iterations (up to an additive $\log \log n+\log \mathcal{D}_n$ factor).

Verifying the bound $\mathcal{B}(t_n)$ may vary in difficulty depending on the magnitude of the search time $t_n$ of the Gaussian theoretical gradient descent iterates $\{\mathrm{u}_\ast^{(t)}\}$. In particular, if $\{\mathrm{u}_\ast^{(t)}\}$ do not spend a significantly long search time so that $t_n$ can be taken as $t_n \lesssim \log \log n$, then both $\mathcal{B}_0(t_n)$ and $\mathcal{B}(t_n)$ can be controlled with $\mathcal{B}_0(t_n)\vee\mathcal{B}(t_n)\leq \mathrm{polylog}(n)$:

\begin{corollary}\label{cor:nonconvex_generic_conv}
	Fix step sizes $\{\eta_r\}$ such that (\ref{cond:nonconvex_generic_conv_long_eta}) holds with  $\eta_\ast \in (0,1)$. Suppose Assumption \ref{assump:abstract} holds for some $K,\Lambda\geq 2$, and (N1)-(N2) hold. Further assume that there exists some constant $c_0>0$ such that $K\Lambda L_\ast^{(0)}\vee \log m\leq (\log n)^{c_0}$ and $t_n\leq c_0\log \log n$ hold. Then there exist constants $c_1=c_1(\mathfrak{p},c_0)>0$ and  $c_2=c_2(\mathfrak{p},c_0,\eta_\ast,\sigma_\ast)>0$ such that if $\phi \geq (\log m)^{c_1}$, on an event with $\Prob^{(0)}$-probability at least $1-m^{-100}$, uniformly for $t\leq m^{100}$, 
	\begin{align*}
	(n\wedge \phi)^{1/2}\pnorm{\mu^{(t)}-\mathrm{u}_\ast^{(t)}}{} + \max_{i \in [m]}\abs{\iprod{X_i}{\mu^{(t)}}}\leq (\log n)^{c_2}.
	\end{align*}
\end{corollary}
The proof of the above corollary can be found in Section \ref{subsection:proof_nonconvex_generic_conv}.

A more challenging situation arises when the gradient descent has a long search time so that $t_n$ is large. Intuitively, a long search time $t_n$ may cause the gradient descent to blow up in the worst case. In this sense, controlling $\mathcal{B}_0(t_n)$ and $\mathcal{B}(t_n)$ to prevent such blow-up is a fundamental rather than merely a technical issue. In these cases, one typically needs to exploit problem-specific structures in a precise manner.

\subsection{Long-time estimates for $\mathscr{M}_{ \mathrm{u}_X^{(\cdot)},\mathrm{u}_{\ast}^{(\cdot)} }^{(t),X}$ and $\mathfrak{M}_{ \mu^{(\cdot)},\mathrm{u}_X^{(\cdot)} }^{(t)}(X)$}\label{subsection:long_time_dynamics_proof_sketch}

In order to prove Theorem \ref{thm:nonconvex_generic_conv_long} via the master Theorems \ref{thm:grad_des_se} and \ref{thm:grad_des_incoh}, the key is to produce long-time estimates for $\mathscr{M}_{ \mathrm{u}_X^{(\cdot)},\mathrm{u}_{\ast}^{(\cdot)} }^{(t),X}$ and $\mathfrak{M}_{ \mu^{(\cdot)},\mathrm{u}_X^{(\cdot)} }^{(t)}(X)$ under (N1)-(N2).

For the deterministic quantity $\mathscr{M}_{ \mathrm{u}_X^{(\cdot)},\mathrm{u}_{\ast}^{(\cdot)} }^{(t),X}$, we have:

\begin{proposition}\label{prop:nonconvex_generic_conv_long_Mscr}
	Suppose Assumption \ref{assump:abstract} holds for some $K,\Lambda\geq 2$ and $\mathfrak{p}\geq 1$, (N1)-(N2) in Theorem \ref{thm:nonconvex_generic_conv_long} hold, and there exists a constant $c_0>0$ such that (\ref{cond:nonconvex_generic_conv_long_Bn0}) holds. Fix step sizes $\{\eta_r\}$ such that (\ref{cond:nonconvex_generic_conv_long_eta}) holds with  $\eta_\ast \in (0,1)$.  Then there exist constants $c_1=c_1(\mathfrak{p},c_0)>0$ and  $c_2=c_2(\mathfrak{p},c_0,\eta_\ast,\sigma_\ast)>0$ such that if $\phi\geq (\log m)^{c_1}$ and $\mathcal{B}(t\wedge t_n)\leq n^{1/2}/(\log n)^{c_1}$, 
	\begin{align*}
	\mathscr{M}_{ \mathrm{u}_X^{(\cdot)},\mathrm{u}_{\ast}^{(\cdot)} }^{(t),X}\vee \max_{r \in [0:t]} n^{1/2}\pnorm{\mathrm{u}_{\ast}^{(r+1)} - \mathrm{u}_{X}^{(r+1)} }{}\leq (\log n)^{c_2}\cdot \mathcal{B}(t\wedge t_n).
	\end{align*}
\end{proposition}

The estimate in Proposition \ref{prop:nonconvex_generic_conv_long_Mscr} is proved by an induction scheme that exploits the contraction nature of the Gaussian theoretical gradient descent. To illustrate the main idea, by Definition \ref{def:G_M}, we need to control $\big\{\pnorm{I_n - \eta_r\cdot \mathbb{M}_{ \mathrm{u}_{X}^{(r)},\mathrm{u}_{\ast}^{(r)} }^{X} }{\op}: r \in [0:t]\big\}$. To do so, we replace $\mathbb{M}_{ \mathrm{u}_{X}^{(r)},\mathrm{u}_{\ast}^{(r)} }^{X}$ with its Gaussian counterpart $\mathbb{M}_{ \mathrm{u}_{\ast}^{(r)},\mathrm{u}_{\ast}^{(r)} }^{\mathsf{Z}}$ and prove that, for $n$ large enough,
\begin{align}\label{ineq:nonconvex_generic_conv_long_Mscr_sketch_1}
&\bigabs{\,\bigpnorm{I_n - \eta_r\cdot \mathbb{M}_{ \mathrm{u}_{X}^{(r)},\mathrm{u}_{\ast}^{(r)} }^{X} }{\op} - \bigpnorm{I_n - \eta_r\cdot \mathbb{M}_{ \mathrm{u}_{\ast}^{(r)},\mathrm{u}_{\ast}^{(r)} }^{\mathsf{Z}} }{\op }}\nonumber\\
&\leq \bigo\,\big(\mathrm{polylog}(n)\cdot \pnorm{\mathrm{u}_{X}^{(r)}-\mathrm{u}_{\ast}^{(r)}}{} \big) + \epsilon_n/2.
\end{align}
On the other hand, we prove that 
\begin{align}\label{ineq:nonconvex_generic_conv_long_Mscr_sketch_2}
\pnorm{\mathrm{u}_{X}^{(r)}-\mathrm{u}_{\ast}^{(r)}}{}&\leq \mathrm{polylog}(n)\cdot n^{-1/2}\cdot \mathscr{M}_{ \mathrm{u}_X^{(\cdot)},\mathrm{u}_{\ast}^{(\cdot)} }^{(r-1),X}.
\end{align}
From (\ref{ineq:nonconvex_generic_conv_long_Mscr_sketch_1}) and (\ref{ineq:nonconvex_generic_conv_long_Mscr_sketch_2}), it follows that control of $\mathscr{M}_{ \mathrm{u}_X^{(\cdot)},\mathrm{u}_{\ast}^{(\cdot)} }^{(t),X}$ can be reduced to that of $\big\{\mathscr{M}_{ \mathrm{u}_X^{(\cdot)},\mathrm{u}_{\ast}^{(\cdot)} }^{(r),X}: r \in [0:t-1]\big\}$, and its order at iteration $t$ can be controlled, provided that the entire past history is uniformly well controlled. The details of this induction scheme can be found in Section \ref{subsection:proof_nonconvex_generic_conv_long_Mscr_sketch}.

Next, for the random quantity $\mathfrak{M}_{ \mu^{(\cdot)},\mathrm{u}_X^{(\cdot)} }^{(t)}(X)$, we have:

\begin{proposition}\label{prop:nonconvex_generic_conv_long_Mcal}
	Suppose Assumption \ref{assump:abstract} holds for some $K,\Lambda\geq 2$ and $\mathfrak{p}\geq 1$, (N1)-(N2) in Theorem \ref{thm:nonconvex_generic_conv_long} hold, and there exists a universal constant $c_0>0$ such that (\ref{cond:nonconvex_generic_conv_long_Bn0}) holds. Fix step sizes $\{\eta_r\}$ such that (\ref{cond:nonconvex_generic_conv_long_eta}) holds with  $\eta_\ast \in (0,1)$. Then there exist constants $c_1=c_1(\mathfrak{p},c_0)>0$ and  $c_2=c_2(\mathfrak{p},c_0,\eta_\ast,\sigma_\ast)>0$ such that if $\phi\geq (\log m)^{c_1}$ and $\mathcal{B}(t_n)\leq (n\wedge \phi)^{1/2}/(\log n)^{c_1}$, on an event with $\Prob^{(0)}$-probability at least $1-m^{-100}$, uniformly for $t\leq m^{100}$, 
	\begin{align*}
	&\mathfrak{M}_{ \mu^{(\cdot)},\mathrm{u}_X^{(\cdot)} }^{(t)}(X) \vee  \max_{i \in [m]} \mathfrak{M}_{\mu^{(\cdot)}, \mu_{[-i]}^{(\cdot)} }^{(t)}(X)\leq (\log n)^{c_2}\cdot \mathcal{B}(t\wedge t_n).
	\end{align*}
\end{proposition}

The estimate in Proposition \ref{prop:nonconvex_generic_conv_long_Mcal} is much harder to prove due to the random nature of $\mathfrak{M}_{ \mu^{(\cdot)},\mathrm{u}_X^{(\cdot)} }^{(t)}(X)$. In fact, an inductive scheme for $\mathfrak{M}_{ \mu^{(\cdot)},\mathrm{u}_X^{(\cdot)} }^{(t)}(X)$ does not hold solely along its own past iterates, but also relies on the leave-one-out versions $\big\{\mathfrak{M}_{\mu^{(\cdot)}, \mu_{[-i]}^{(\cdot)} }^{(t)}(X)\big\}$. We therefore provide (high-probability) control of the desired term $\mathfrak{M}_{ \mu^{(\cdot)},\mathrm{u}_X^{(\cdot)} }^{(t)}(X)$ by a \emph{joint induction scheme} that simultaneously controls both terms
\begin{align}\label{ineq:nonconvex_generic_conv_long_Mcal_sketch_1}
\mathfrak{M}_{ \mu^{(\cdot)},\mathrm{u}_X^{(\cdot)} }^{(r)}(X),\, \max_{i \in [m]} \mathfrak{M}_{\mu^{(\cdot)}, \mu_{[-i]}^{(\cdot)} }^{(r)}(X),\quad r \in [0:t-1].
\end{align}
A technical subtlety in implementing this joint induction scheme is to control the leave-one-out terms  $\big\{\mathfrak{M}_{\mu^{(\cdot)}, \mu_{[-i]}^{(\cdot)} }^{(\cdot)}(X)\big\}$ again by the same terms in (\ref{ineq:nonconvex_generic_conv_long_Mcal_sketch_1}), without leaving out further samples, to ensure valid long-time control. The details of this technical joint induction scheme can be found in Section \ref{subsection:proof_nonconvex_generic_conv_long_Mcal}.

\section{Applications to regression}\label{section:concrete_models}

\subsection{Data model}

We consider the single-index regression model
\begin{align}\label{def:model_single_index}
Y_i = \varphi (\iprod{\mu_\ast}{X_i})+\xi_i,\quad i \in [m],
\end{align}
where $\varphi:\R \to \R$ is a link function. The data model (\ref{def:model_single_index}) can be recast into the general form (\ref{def:model}) by setting $\mathcal{F}(z_0,\xi_0)\equiv \varphi(z_0)+\xi_0$. Consider the squared loss function $\mathsf{L}:\R^2\to \R$ with
\begin{align*}
\mathsf{L}(x_0,y_0) = \big(\varphi(x_0)-y_0\big)^2/2.
\end{align*}
Note that the loss landscape can be highly nonconvex due to the potentially arbitrarily complicated nature of $\varphi$. To avoid notational complications, we focus on constant step sizes $\eta_t=\eta$.

The Gaussian theoretical gradient descent iterates $\{\mathrm{u}_\ast^{(t)}\}$ for the model (\ref{def:model_single_index}) take the following explicit form.
\begin{definition}\label{def:single_index_theo_gd}
Let 
\begin{align}\label{def:se_single_index}
\mathrm{u}_\ast^{(t)}\equiv \big(1-\eta \tau_\ast^{(t-1)}\big)\cdot \mathrm{u}_\ast^{(t-1)} + \eta \delta_\ast^{(t-1)}\cdot \mu_\ast
\end{align}
be determined with the key parameters $\{(\tau_\ast^{(t-1)},\delta_\ast^{(t-1)})\}$ computed recursively by
\begin{align}\label{def:tau_delta_single_index}
\tau_\ast^{(t-1)}&\equiv \E^{(0)} \big(\varphi'(\iprod{\mathsf{Z}_n}{\mathrm{u}_{\ast}^{(t-1)}})\big)^2 \nonumber\\
&\qquad + \E^{(0)}  \big(\varphi(\iprod{\mathsf{Z}_n}{\mathrm{u}_{\ast}^{(t-1)}})-\varphi(\iprod{\mathsf{Z}_n}{\mu_\ast})-\xi_{\pi_m}\big)\varphi''( \iprod{\mathsf{Z}_n}{\mathrm{u}_{\ast}^{(t-1)}}), \nonumber\\
\delta_\ast^{(t-1)}&\equiv \E^{(0)}  \varphi'\big(\iprod{\mathsf{Z}_n}{\mathrm{u}_{\ast}^{(t-1)}}\big) \varphi'\big(\iprod{\mathsf{Z}_n}{\mu_\ast}\big).
\end{align}
\end{definition}

Due to the presence of the noise $\{\xi_i\}$,  $\mathrm{u}_\ast^{(\cdot)}=\mu_\ast$ is in general only an approximate stationary point of the above system of equations (\ref{def:se_single_index}) and (\ref{def:tau_delta_single_index}).

\subsection{Convergence of nonconvex gradient descent}\label{subsection:conv_noncvx_gd_single_index}

\subsubsection{Convergence with structured $\varphi$'s}
We first examine a general class of structured link functions $\varphi$'s that lead to fast global convergence of gradient descent with moderate search time $t_n$. For notational simplicity, the signal is normalized with $\pnorm{\mu_\ast}{}=1$.

To formally state our result, note that (\ref{def:varrho_generic}) reduces in this setting to 
\begin{align*}
\varrho_\ast\equiv \min\big\{\E^{(0)}  \big(\varphi'(\mathsf{Z}_1)\big)^2,\E^{(0)}  \big(\mathsf{Z}_1\varphi'(\mathsf{Z}_1)\big)^2\big\},\quad \mathsf{Z}_1\sim \mathcal{N}(0,1).
\end{align*}

\begin{theorem}\label{thm:single_index_se}
Suppose the following hold:
\begin{enumerate}
	\item (A1) holds for some $K\geq 2$. 
	\item For some $\Lambda\geq 2$, $\mathfrak{p}\geq 1$, the functions $\{\varphi^{(q)}\}_{q \in [0:5]}$ are $\Lambda$-pseudo-Lipschitz of order $\mathfrak{p}$ and are bounded by $\Lambda$ at $0$. 
	\item For some nonincreasing function $\kappa:\R_{\geq 0} \to (0,1]$, 
	\begin{align}\label{cond:single_index_se_mnt}
	\inf_{\abs{x}\leq z} \abs{\varphi'(x)} \geq \kappa(z)>0,\quad \forall z\geq 0.
	\end{align}
\end{enumerate}
Then under $(K\Lambda L_\ast^{(0)})\vee \log m\leq (\log n)^{100}$, there exists $\eta_0=\eta_0(\varrho_\ast,\kappa,\Lambda,\pnorm{\mu^{(0)}}{})\in (0,1)$ such that for $\eta\leq \eta_0$, the following hold whenever $\abs{\E_{\pi_m}\xi_{\pi_m}}\leq \eta^{1/2}$:
\begin{enumerate}
	\item Proposition \ref{prop:nonconvex_generic_conv_long_contraction}-(2) is verified for $\tau_n=0$, $\mathcal{D}_n=1$ and $\epsilon_0=\eta(\kappa_\ast\wedge \varrho_\ast)/4$ where $\kappa_\ast\equiv \kappa^2\big(6(1+\pnorm{\mu^{(0)}}{})\big)$. 
	\item There exists another constant $c_1=c_1(\mathfrak{p},\eta,\varrho_\ast,\kappa)>0$ such that if $\phi\geq (\log m)^{c_1}$, on an event with $\Prob^{(0)}$-probability at least $1-m^{-100}$, uniformly for $t\leq m^{100}$,  (\ref{eqn:nonconvex_generic_conv_long}) holds with a bound $(\log n)^{c_1}$ and
	\begin{align*}
	\pnorm{\mu^{(t)}-\mu_\ast}{} \leq \bigg(1-\frac{\eta(\kappa_\ast\wedge \varrho_\ast)}{4}\bigg)^t\cdot \pnorm{\mu^{(0)}-\mu_\ast}{}+ \frac{ (\log n)^{c_1} }{(n\wedge \phi)^{1/2}}.
	\end{align*}
\end{enumerate}
\end{theorem}

The proof of the above theorem can be found in Section \ref{subsection:proof_single_index_se}. In this example, gradient descent at most spends a moderate search time with $t_n \lesssim \log \log n$, and therefore enjoys nice linear convergence properties as guaranteed by Corollary \ref{cor:nonconvex_generic_conv}.

\begin{remark}
	Some technical remarks are in order:
	\begin{enumerate}
		\item While we do not present an explicit formula for $\eta_0$ in the statement of Theorem \ref{thm:single_index_se}, such a form can be easily found in the proof. 
		\item The two terms in conclusion (2) are familiar in the statistical literature. In particular, the first term indicates linear convergence of the optimization error, whereas the second term represents the statistical error (up to multiplicative logarithmic factors).
	\end{enumerate}
\end{remark}

\subsubsection{Universality of gradient descent in phase retrieval}

When the qualitative condition (\ref{cond:single_index_se_mnt}) is not satisfied for the link function $\varphi$, the behavior of gradient descent can be much more complicated. Here we consider the example of the real-valued phase retrieval problem with $\varphi(x)=x^2$ \cite{candes2015phase,chen2019gradient}. Again for notational simplicity, the signal is normalized with $\pnorm{\mu_\ast}{}=1$. In this setting, 
\begin{align*}
\mathfrak{S}(x_0,z_0,\xi_0)&\equiv 2(x_0^2-z_0^2-\xi_0)\cdot x_0,\quad\partial_1\mathfrak{S}(x_0,z_0,\xi_0)\equiv 2\cdot \big(3x_0^2-z_0^2-\xi_0\big).
\end{align*}
The key parameters $\{(\tau_\ast^{(t-1)},\delta_\ast^{(t-1)})\}$  in (\ref{def:tau_delta_single_index}) become
\begin{align*}
\tau_\ast^{(t-1)}& \equiv 2\cdot \big(3\pnorm{\mathrm{u}_\ast^{(t-1)}}{}^2-1\big)-2 \E_{\pi_m}\xi_{\pi_m},\quad \delta_\ast^{(t-1)} \equiv 4 \iprod{\mathrm{u}_\ast^{(t-1)}}{\mu_\ast},
\end{align*}
and $\{\mathrm{u}_\ast^{(t)}\}$ in (\ref{def:se_single_index}) takes the form
\begin{align}\label{def:tau_delta_pr}
\mathrm{u}_\ast^{(t)}& = \big[1-2\eta\big(3\pnorm{\mathrm{u}_\ast^{(t-1)}}{}^2-1\big)+2\eta\cdot \E_{\pi_m}\xi_{\pi_m}\big]\cdot \mathrm{u}_\ast^{(t-1)}+ 4\eta \iprod{\mathrm{u}_\ast^{(t-1)}}{\mu_\ast}\cdot \mu_\ast.
\end{align}
The following theorem is proven in Section \ref{subsection:proof_pr_se}.  
\begin{theorem}\label{thm:pr_se}
Consider the phase retrieval problem with $\varphi(x)=x^2$. Suppose (A1) holds for some $K\geq 2$ and $\E X_{ij}^3=0$ for all $i \in [m], j \in [n]$, and the initialization $\mu^{(0)}\in \R^n$ satisfies 
\begin{align}\label{cond:pr_initialization}
\abs{\iprod{\mu^{(0)}}{\mu_\ast}}\geq {1}/{\sqrt{n\log n}},\quad \bigabs{\pnorm{\mu^{(0)}}{}-1}\leq {1}/{\log n}.
\end{align}
Further assume that $KL_\ast^{(0)}\vee \log m\leq (\log n)^{c_0}$, $\eta\leq 1/c_0$ and $\abs{\E_{\pi_m}\xi_{\pi_m}}\leq 1/(\log n)^{c_0}$ hold for a sufficiently large constant $c_0>0$. Then there exists some $c_1=c_1(c_0,\eta)>0$ such that the following hold:
\begin{enumerate}
	\item (N2) is satisfied for some $t_n\leq c_1\log n$ with $\sigma_\ast=1$ and $\mathcal{B}_0(t_n)\vee n^{-1/2}\mathcal{B}(t_n)\leq (\log n)^{c_1}$. 
	\item Suppose further that $\phi/n\geq (\log n)^{c_1}$. Then, on an event with $\Prob^{(0)}$-probability at least $1-m^{-100}$, uniformly for $t\leq m^{100}$, with $\epsilon_\ast^{(0)}\equiv \sign(\iprod{\mu^{(0)}}{\mu_\ast})$,
	\begin{align*}
	&\pnorm{\mu^{(t)}-\mathrm{u}_\ast^{(t)}}{} +\Big(\pnorm{\mu^{(t)}-\epsilon_\ast^{(0)}\mu_\ast}{} - c_1\big(1-\eta\big)^{(t-c_1\log n)_+}\Big)_+ \leq \frac{ (\log n)^{c_1} }{(n\wedge \{\phi/n\})^{1/2}}.
	\end{align*}
\end{enumerate}
\end{theorem}

\begin{remark}
The initialization condition (\ref{cond:pr_initialization}) is the same as in \cite[Eq. (21) in Lemma 1]{chen2019gradient}, and can be easily verified with high probability for the canonical Gaussian initialization $\mu^{(0)}\sim \mathcal{N}(0,I_n/n)$. 
\end{remark}

An important qualitative feature in this example is that  the gradient descent spends a long search time $t_n\asymp \log n$ before entering the convergence stage. This necessitates a nontrivial, case-specific control for $\mathcal{B}_0(t_n)$ and $\mathcal{B}(t_n)$. 

To provide some technical insights, let $\proj_{\mu_\ast}\equiv \mu_\ast \mu_\ast^\top$ be the projection onto $\mu_\ast$, and $\proj_{\mu_\ast}^\perp\equiv I_n-\proj_{\mu_\ast}$. Further let
\begin{align*}
\alpha_t\equiv \iprod{\mathrm{u}_\ast^{(t)}}{\mu_\ast},\quad \beta_t\equiv \pnorm{ \proj_{\mu_\ast}^\perp(\mathrm{u}_\ast^{(t)}) }{}.
\end{align*}
Applying $\proj_{\mu_\ast}$ and $\proj_{\mu_\ast}^\perp$ to both sides of (\ref{def:tau_delta_pr}) yields that
\begin{align}\label{def:alpha_beta_pr}
\alpha_t & = \big[1-6\eta\cdot (\alpha_{t-1}^2+\beta_{t-1}^2-1)+2\eta\cdot \E_{\pi_m}\xi_{\pi_m}\big]\cdot \alpha_{t-1},\nonumber\\
\beta_t& = \big[1-6\eta\cdot \big(\alpha_{t-1}^2+\beta_{t-1}^2-1/3\big)+2\eta\cdot \E_{\pi_m}\xi_{\pi_m}\big]\cdot \beta_{t-1}.
\end{align}
\cite{chen2019gradient} shows that the state evolution (\ref{def:alpha_beta_pr}) proceeds in three stages:
\begin{itemize}
	\item (\emph{Stage 1}). The signal remains small $\abs{\alpha_t}\approx n^{-1/2}$, whereas $\beta_t$ decreases from $1$ to $1/\sqrt{3}$. This stage takes $\abs{\mathcal{T}_1}=\bigo(1)$ iterations.
	\item (\emph{Stage 2}). The signal $\abs{\alpha_t}$ increases to $1$ exponentially, whereas $\beta_t$ remains at $1/\sqrt{3}$. This stage takes $\abs{\mathcal{T}_2}=\bigo(\log n)$ iterations.
	\item (\emph{Stage 3}). The signal $\abs{\alpha_t}$ remains at $1$, whereas $\beta_t$ drops down to a sufficiently small constant exponentially. This stage takes $\abs{\mathcal{T}_3}=\bigo(1)$ iterations.
\end{itemize}
After the combined search phase $\mathcal{T}\equiv \cup_{\ell \in [3]} \mathcal{T}_\ell$, gradient descent enters a locally strongly convex region with linear convergence. The key to controlling $\mathcal{B}_0(t_n)$ and $\mathcal{B}(t_n)$ for our purpose is to identify a substage $\mathcal{T}_2'\subset \mathcal{T}_2$ of the time-dominant Stage 2 such that $\abs{\mathcal{T}_2\setminus \mathcal{T}_2'}\lesssim \log\log n$, and for $t \in \mathcal{T}_2'$,
\begin{align}\label{ineq:pr_control_M}
\lambda_{\min}\Big(\mathbb{M}_{ \mathrm{u}_\ast^{(t)},\mu_\ast }^{\mathsf{Z}}\Big)\geq - c/\log^4 n,\quad \lambda_{\min}\Big(\mathbb{M}_{ \mathrm{u}_{\ast}^{(t)},\mathrm{u}_{\ast}^{(t)} }^{\mathsf{Z}}\Big)\approx -4.
\end{align}
The estimates in~\eqref{ineq:pr_control_M} show that the Gaussian theoretical gradient-descent iterates $\{\mathrm{u}_\ast^{(t)}\}$ remain nearly stable for $t\in\mathcal{T}_2'$, whereas the universality cost incurred when passing from the Gaussian design to other designs can be substantially larger. Indeed, two important and distinct features arise in Theorem~\ref{thm:pr_se}:
\begin{enumerate}
	\item Universality holds in the substantially larger aspect ratio regime $\phi \ge n \cdot \mathrm{polylog}(n)$ (equivalently, $n \le m^{1/2}/\mathrm{polylog}(n)$), whereas for the Gaussian design it already suffices that $\phi \ge \mathrm{polylog}(n)$.
	\item Universality applies to general designs whose entries match the first three moments of the Gaussian design, rather than merely the first two.
\end{enumerate}

\begin{figure}[t]
	\centering
	\begin{minipage}{0.325\textwidth}
		\centering
		\includegraphics[width=\linewidth]{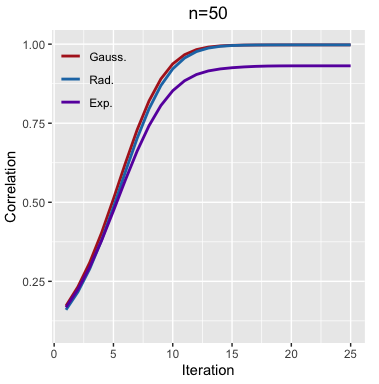}
	\end{minipage}
	\begin{minipage}{0.325\textwidth}
		\centering
		\includegraphics[width=\linewidth,height=0.2\textheight]{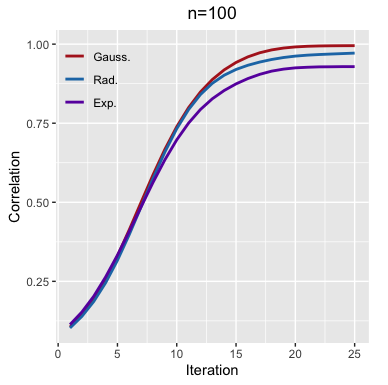}
	\end{minipage}
	\begin{minipage}{0.325\textwidth}
		\centering
		\includegraphics[width=\linewidth,height=0.2\textheight]{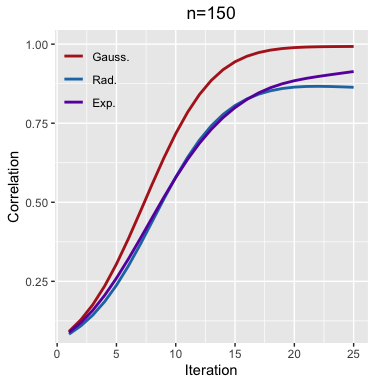}
	\end{minipage}
	\caption{Randomly initialized gradient descent in phase retrieval with step size $\eta=0.1$, sample size $m=3000$, varying dimensions $n$, and true signal $\mu_\ast=\bm{1}_n/n^{1/2}$. \emph{Left}: $n=50$. \emph{Middle}: $n=100$. \emph{Right}: $n=150$. All panels choose random initialization $\mu^{(0)}$ with $\sqrt{n}\iprod{\mu^{(0)}}{\mu_\ast}/\pnorm{\mu^{(0)}}{}\approx 0.7$ and report the mean correlation over $B=50$ Monte Carlo replications.}
	\label{fig:pr}
\end{figure}

Figure~\ref{fig:pr} numerically investigates these issues in the universality theory under the Rademacher design and the exponential design (whose entries are i.i.d.\ standardized $\mathrm{Exp}(1)$ variables), with a sample size $m=3000$ and dimensions $n \in \{50,100,150\}$. Our theory in Theorem~\ref{thm:pr_se} guarantees universality of algorithmic convergence up to dimensions on the order of $n \sim m^{1/2} \approx 55$ for general designs with $\E X_{ij}^3 = 0$. To facilitate a fair comparison, we report numerical results with approximately the same scaled initial correlation $\sqrt{n}\,\langle \mu^{(0)}, \mu_\ast \rangle / \|\mu^{(0)}\| \approx 0.7$. We find that:
\begin{itemize}
	\item In the left panel of Figure~\ref{fig:pr} ($n=50$), universality of randomly initialized gradient descent is confirmed for the Rademacher design, but does not extend to the exponential design with $\E X_{ij}^3 \neq 0$.
	\item In the middle ($n=100$) and right ($n=150$) panels, with moderate aspect ratios, the convergence behavior observed for Gaussian design does not extend to either the Rademacher or the exponential design.
\end{itemize}
A qualitatively similar pattern to that in Figure~\ref{fig:pr} is observed for larger sample sizes (figures omitted for brevity). These numerical experiments suggest the following conjectures: (i) the third-moment condition $\E X_{ij}^3=0$ is necessary in the regime $\phi \gg n$, and (ii) randomly initialized gradient descent for phase retrieval may exhibit distribution-specific non-universality in the regime $1 \ll \phi \ll n$. Proving (or disproving) these conjectures lies beyond the technical scope of this paper.

\subsection{Data-free estimation of state evolution}

\subsubsection{Data-free estimation of state evolution}

We first develop a data-free iterative algorithm to estimate the state evolution parameters $\{(\tau_\ast^{(t-1)}, \delta_\ast^{(t-1)})\}$, whose computational complexity is independent of the problem size and requires no detailed knowledge of the unknown signal $\mu_\ast$.

To motivate the algorithm, in view of (\ref{def:se_single_index}), we have the following recursions for $\{\gamma_\ast^{(t)}\equiv \pnorm{\mathrm{u}_\ast^{(t)}}{}\}$ and $\{\alpha_\ast^{(t)}\equiv \iprod{\mathrm{u}_\ast^{(t)}}{\mu_\ast}\}$ :
\begin{align}\label{eqn:recursion_gamma_alpha}
(\gamma_\ast^{(t)})^2& = \big(1-\eta \tau_\ast^{(t-1)}\big)^2\cdot (\gamma_\ast^{(t-1)})^2+ \big(\eta \delta_\ast^{(t-1)}\big)^2\cdot  \pnorm{\mu_\ast}{}^2\nonumber\\
&\qquad + 2\eta \delta_\ast^{(t-1)} \big(1-\eta \tau_\ast^{(t-1)}\big)\cdot \alpha_\ast^{(t-1)},\nonumber\\
\alpha_\ast^{(t)}&= \big(1-\eta \tau_\ast^{(t-1)}\big)\cdot \alpha_\ast^{(t-1)}+ \eta\delta_\ast^{(t-1)}\cdot \pnorm{\mu_\ast}{}^2.
\end{align}
The recursions in the above display naturally suggest our proposed Algorithm \ref{def:alg_cor_est}, as detailed below.

\begin{algorithm}
	\caption{Data-free estimation of state evolution}\label{def:alg_cor_est}
	\begin{algorithmic}[1]
		\STATE \textbf{Input}: link function $\varphi$, signal strength $\pnorm{\mu_\ast}{}$, and the step size $\eta>0$.
		\STATE \textbf{Initialization}: $\hat{\gamma}_\ast^{(0)}\equiv \pnorm{\mu^{(0)}}{}\in \R$ and $\hat{\alpha}_\ast^{(0)}\in [\pm \hat{\gamma}_\ast^{(0)}\pnorm{\mu_\ast}{}]$.
		\FOR{$ t = 1,2,\ldots$}
		\STATE Compute the covariance matrix:
		\begin{align*}
		\hat{\Sigma}^{(t-1)}\equiv 
		\begin{pmatrix}
		(\hat{\gamma}_\ast^{(t-1)})^2 & \hat{\alpha}_\ast^{(t-1)}\\
		\hat{\alpha}_\ast^{(t-1)} & \pnorm{\mu_\ast}{}^2
		\end{pmatrix}\in \R^{2\times 2}.
		\end{align*}
		\STATE With $\hat{\mathsf{Z}}^{(t-1)}\sim \mathcal{N}(0,\hat{\Sigma}^{(t-1)})\in \R^2$, compute:
		\begin{align*}
		\hat{\tau}_\ast^{(t-1)}&\equiv \hat{\E} \big(\varphi'(\hat{\mathsf{Z}}_1^{(t-1)})\big)^2 + \hat{\E} \big(\varphi(\hat{\mathsf{Z}}_1^{(t-1)})-\varphi(\hat{\mathsf{Z}}_2^{(t-1)})\big)\varphi''(\hat{\mathsf{Z}}_1^{(t-1)}), \nonumber\\
		\hat{\delta}_\ast^{(t-1)}&\equiv \hat{\E} \varphi'\big(\hat{\mathsf{Z}}_1^{(t-1)}\big) \varphi'\big(\hat{\mathsf{Z}}_2^{(t-1)}\big).
		\end{align*}
		Here $\hat{\E}$ is evaluated under the randomness of $\hat{\mathsf{Z}}^{(\cdot)}$ only.
		\item Compute the estimate $\hat{\gamma}_\ast^{(t)}$:
		\begin{align*}
		\hat{\gamma}_\ast^{(t)}&\equiv \big\{\big(1-\eta \hat{\tau}_\ast^{(t-1)}\big)^2\cdot (\hat{\gamma}_\ast^{(t-1)})^2+ \big(\eta \hat{\delta}_\ast^{(t-1)}\big)^2\cdot  \pnorm{\mu_\ast}{}^2\\
		&\qquad\qquad + 2\eta \hat{\delta}_\ast^{(t-1)} \big(1-\eta \hat{\tau}_\ast^{(t-1)}\big)\cdot \hat{\alpha}_\ast^{(t-1)}\big\}_+^{1/2}\wedge e^{100}.
		\end{align*}
		\STATE Compute the estimate $\hat{\alpha}_\ast^{(t)}$:
		\begin{align*}
		\hat{\alpha}_\ast^{(t)}& \equiv  \mathscr{T}_{\hat{\gamma}_\ast^{(t)}\cdot\pnorm{\mu_\ast}{} }\big[ \big(1-\eta \hat{\tau}_\ast^{(t-1)}\big)\cdot \hat{\alpha}_\ast^{(t-1)}+ \eta \hat{\delta}_\ast^{(t-1)}\cdot\pnorm{\mu_\ast}{}^2\big].
		\end{align*}
		Here $\mathscr{T}_M(x)\equiv (x\wedge M)\vee (-M)$ for any $M\geq 0,x \in \R$. 
		\ENDFOR
	\end{algorithmic}
\end{algorithm}

The following theorem provides a formal justification for Algorithm \ref{def:alg_cor_est}; its proof can be found in Section \ref{subsection:proof_single_index_cor_est}. 
\begin{theorem}\label{thm:single_index_cor_est}
Suppose the following hold:
\begin{enumerate}
	\item For some $\Lambda\geq 2$, the functions $\{\varphi^{(q)}\}_{q \in [0:2]}$ are globally bounded by $\Lambda$. 
	\item (N1)-(N2) hold for $t_n\geq 0$ and $\sigma_\ast\in (0,1)$.
\end{enumerate}
Fix $\epsilon \in (0,1)$ with $\pnorm{\mu_\ast}{}\in (\epsilon,1/\epsilon)$, and let $t_{\ast}(\epsilon)\equiv \max\{r \in \mathbb{N}: e^{-100}\leq \pnorm{\mathrm{u}_\ast^{(r)}}{}\leq e^{100}, \abs{\sin(\mathrm{u}_\ast^{(r)},\mu_\ast)}\geq \epsilon\}$. Then there exists some $c_1=c_1(\epsilon,\eta,\sigma_\ast,\Lambda)>0$ and $\eta_0=\eta_0(\sigma_\ast,\Lambda) \in (0,1)$ such that if $\eta\leq \eta_0$, for all $t\leq t_\ast(\epsilon)$, 
\begin{align}\label{ineq:single_index_cor_est}
\abs{\Delta \hat{\tau}_\ast^{(t)} }\vee \abs{\Delta \hat{\delta}_\ast^{(t)} } \vee \hat{\Delta}_\ast^{(t)}&\leq  \min\Big\{ c_1 \mathscr{B}^3(t\wedge t_n) \cdot \Big( \max_{s \in [1:t\wedge t_n]}  \hat{\Delta}_\ast^{(s-1)} +\abs{\E_{\pi_m} \xi_{\pi_m}}\Big),\nonumber\\
&\qquad\qquad   c_1^t\cdot  \big(\abs{\iprod{\mu^{(0)}}{\mu_\ast}-\hat{\alpha}_\ast^{(0)}}+\abs{\E_{\pi_m} \xi_{\pi_m}}\big) \Big\}.
\end{align}
Here $\Delta \hat{\#}_\ast^{(t)}\equiv \hat{\#}_\ast^{(t)}-\#_\ast^{(t)}$, $
\hat{\Delta}_\ast^{(\cdot)}\equiv \abs{\Delta \hat{\alpha}_\ast^{(\cdot)} }+\abs{\Delta (\hat{\gamma}_\ast^{(\cdot)})^2 }$, and for any $t\geq 1$, $\mathscr{B}(t)\equiv 1+ \sum_{s \in [0:t)} \prod_{r \in [s:t)} \abs{1-\eta \tau_\ast^{(r)}}$.

In particular, if $t\wedge t_n\leq 100\log \log n$, then the right hand side of (\ref{ineq:single_index_cor_est}) can be bounded by $(\log n)^{c_1}\cdot \big(\abs{\iprod{\mu^{(0)}}{\mu_\ast}-\hat{\alpha}_\ast^{(0)}}+\abs{\E_{\pi_m} \xi_{\pi_m}}\big)$.
\end{theorem}

We note that the numeric constant $100$ in both Algorithm \ref{def:alg_cor_est} and the above Theorem \ref{thm:single_index_cor_est} is chosen for convenience, and can be replaced by an arbitrarily large constant at the cost of a potentially larger $c_1>0$.

As an immediate application of Theorem \ref{thm:single_index_cor_est}, we may use Algorithm \ref{def:alg_cor_est} to provide consistent predictions of the correlation between the gradient descent $\mu^{(t)}$ and the unknown signal $\mu_\ast$: 
\begin{align}\label{def:oracle_correlation}
\mathsf{corr}_\ast^{(t)} \equiv \frac{\iprod{\mu^{(t)}}{\mu_\ast}}{ \pnorm{\mu^{(t)}}{}\cdot \pnorm{\mu_\ast}{} },\quad t=1,2,\ldots.
\end{align}
Note that since the underlying signal $\mu_\ast$ is unknown, the random quantity $\mathsf{corr}_\ast^{(t)}$ cannot be directly computed from the observed data. Using the output of Algorithm \ref{def:alg_cor_est}, we may naturally construct the estimate
\begin{align}\label{def:corr_est_se}
\widehat{\mathsf{corr}}_\ast^{(t)} \equiv \frac{\hat{\alpha}_\ast^{(t)}}{ \hat{\gamma}_\ast^{(t)}\cdot  \pnorm{\mu_\ast}{} }.
\end{align}
A major feature of $\widehat{\mathsf{corr}}_\ast^{(t)}$ in (\ref{def:corr_est_se}) is that its computational complexity is independent of the problem size, allowing it to be computed very efficiently via Algorithm~\ref{def:alg_cor_est}. From a practical standpoint, such low-cost estimators $\big\{\widehat{\mathsf{corr}}_\ast^{(t)}\big\}$ can be applied to tasks such as hyperparameter tuning of the step size, determining algorithmic running time, and related purposes.

\begin{remark}
	Some technical remarks for Algorithm \ref{def:alg_cor_est} and Theorem \ref{thm:single_index_cor_est} are in order:
	\begin{enumerate}
		\item Algorithm \ref{def:alg_cor_est} requires knowledge of the signal strength $\pnorm{\mu_\ast}{}$. If it is unknown, ad-hoc methods may be used to estimate this quantity. For instance, for sufficiently regular link functions $\varphi$, one may use the method of moments, e.g., the value of $m^{-1}\sum_{i=1}^m Y_i^q \approx \E^{(0)}\big\{\varphi(\pnorm{\mu_\ast}{}\mathsf{Z}_1)+\xi_1\big\}^q$ for $q \in \mathbb{N}$ to obtain an estimate for $\pnorm{\mu_\ast}{}$.
		\item Algorithm~\ref{def:alg_cor_est} is fundamentally different from the `gradient descent inference algorithm' developed in \cite{han2024gradient}. The latter operates in the mean-field regime $\phi \asymp 1$ and uses observed data to estimate only a partial set of the state evolution parameters (in our notation, the parameters $\{\tau_\ast^{(t-1)}\}$) for the purpose of constructing the so-called `debiased gradient descent iterates'. In contrast, the data-free Algorithm~\ref{def:alg_cor_est} computes estimates of the entire set of state evolution parameters $\{(\tau_\ast^{(t-1)}, \delta_\ast^{(t-1)})\}$ in the large-aspect-ratio regime $\phi \gg 1$ by taking advantage of the simple form in Definition~\ref{def:tau_delta_single_index}, and, more importantly, is computationally light with complexity independent of the problem size.
		\item The regularity conditions in Theorem \ref{thm:single_index_cor_est} may not be optimal. For instance, condition (1) in Theorem \ref{thm:single_index_cor_est} can possibly be further relaxed. We adopt the current set of assumptions to avoid non-essential complications in the proofs.
		\item The term $\max_{s \in [1:t\wedge t_n]} \hat{\Delta}_\ast^{(s-1)}$ in the claimed estimate of Theorem \ref{thm:single_index_cor_est} depends on both the gradient descent search time $t_n$ and the quality of the initialization $\hat{\alpha}_\ast^{(0)}$. If the gradient descent spends a moderate search time, i.e., $t_n \leq c_0 \log \log n$, then this term can be trivially bounded by $(\log n)^{c_1}$, but otherwise requires a case-specific treatment. We believe this term cannot be removed for free. 
	\end{enumerate}
\end{remark}

\subsubsection{Numerical illustrations}

\begin{figure}[t]
	\centering
	\begin{minipage}{0.325\textwidth}
	\centering
	\includegraphics[width=\linewidth]{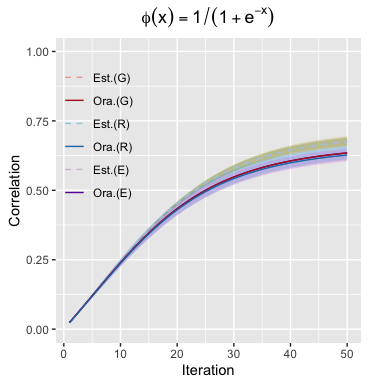}
    \end{minipage}
	\begin{minipage}{0.325\textwidth}
		\centering
		\includegraphics[width=\linewidth]{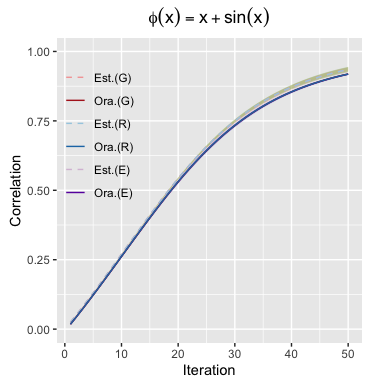}
	\end{minipage}
	\begin{minipage}{0.325\textwidth}
	\centering
	\includegraphics[width=\linewidth]{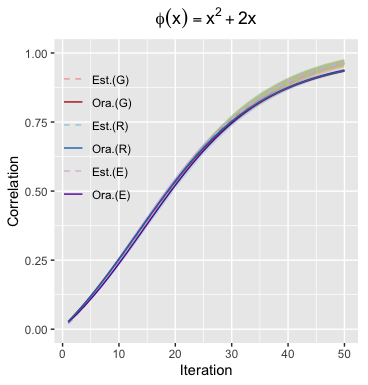}
    \end{minipage}
	\caption{Simulations for $\widehat{\mathsf{corr}}_\ast^{(t)}$ in (\ref{def:corr_est_se}) computed by Algorithm \ref{def:alg_cor_est} with three link functions.  \emph{Left}: $\varphi(x)=1/(1+e^{-x})$. \emph{Middle}: $\varphi(x)=x+\sin(x)$. \emph{Right}: $\varphi(x)=x^2+2x$. }
	\label{fig:mnt}
\end{figure}

We now examine the numerical performance of the proposed estimator $\widehat{\mathsf{corr}}_\ast^{(t)}$ in~\eqref{def:corr_est_se}, computed via Algorithm~\ref{def:alg_cor_est}, and compare it with the oracle correlation $\mathsf{corr}_\ast^{(t)}$ in~\eqref{def:oracle_correlation}. In Figure~\ref{fig:mnt}, we consider three link functions:
\begin{itemize}
	\item left panel: $\varphi(x)=1/(1+e^{-x})$ with step size $\eta=0.5$;
	\item middle panel: $\varphi(x)=x+\sin(x)$ with step size $\eta=0.01$;
	\item right panel: $\varphi(x)=x^2+2x$ with step size $\eta=0.005$.
\end{itemize}
Other simulation parameters are as follows:
\begin{itemize}
	\item signal dimension $n=300$;
	\item true signal $\mu_\ast=\bm{1}_n/\sqrt{n}$;
	\item sample size $m=10^4$;
	\item initialization $\mu^{(0)}\sim \mathcal{N}(0,I_n/n)$ with $\hat{\alpha}_\ast^{(0)}=0$;
	\item the expectation $\hat{\E}$ in Algorithm~\ref{def:alg_cor_est} is approximated by averaging $10^3$ i.i.d.\ draws of the underlying bivariate Gaussian random vector.
\end{itemize}
All simulations in Figure~\ref{fig:mnt} are based on $B=100$ Monte Carlo replications.

We note that the link functions in the left and middle panels of Figure~\ref{fig:mnt} satisfy condition~\eqref{cond:single_index_se_mnt} of Theorem~\ref{thm:single_index_se}; consequently, Theorem~\ref{thm:single_index_cor_est} applies with a moderate search time $t_n \lesssim \log\log n$, yielding a controlled error bound. By contrast, the link function in the right panel of Figure~\ref{fig:mnt} does not satisfy~\eqref{cond:single_index_se_mnt}, so Theorem~\ref{thm:single_index_cor_est} does not directly yield a comparable bound. Even so, our numerical simulations show excellent agreement between the estimated and oracle correlations. Moreover, in Figure~\ref{fig:mnt} we use Gaussian, Rademacher, and exponential designs and confirm the universality of our results under first-two-moment matching (without matching the third moment), as predicted by Theorem~\ref{thm:single_index_cor_est}.

\section{Relation to dynamical mean-field theory}\label{section:mean_field_gd}

\subsection{Mean-field state evolution}
In a recent line of works \cite{celentano2021high,gerbelot2024rigorous,han2024gradient,fan2025dynamical,han2025entrywise,han2025precise}, the precise dynamics of gradient descent have been characterized in a variety of statistical models in the so-called mean-field regime $\phi \asymp 1$, essentially over a constant-time horizon $t=\bigo(1)$. In this regime, due to the intrinsically moderate signal-to-noise ratio, the gradient descent iterates $\mu^{(t)}$ do not concentrate around any vector (for instance, $\mu_\ast$), and the corresponding mean-field theory characterizes their precise distributional behavior.

For convenience in the subsequent discussion, we follow the notation and setup used in \cite{han2024gradient}. In the form of the gradient descent $\mu^{(t)}$ in (\ref{def:grad_des}), to set up the same parameter scaling as in \cite{han2024gradient}, consider the renormalization
\begin{align}\label{def:mean_field_renormalization}
A\equiv X/n^{1/2},\quad \big\{\beta^{(t)}\equiv n^{1/2} \mu^{(t)}:t\geq 0\big\}.
\end{align}
Then the coordinates of the iterates $\{\beta^{(t)}\}$ are expected to have magnitude $\bigo(1)$, and the entries of $A$ have variance $1/n$. With the renormalization (\ref{def:mean_field_renormalization}), for $t=1,2,\ldots$,
\begin{align}\label{eqn:gd_mean_field}
\beta^{(t)}&\equiv \beta^{(t-1)}-\frac{\eta_{t-1}}{\phi}\cdot A^\top \partial_1 \mathsf{L}\big(A\beta^{(t-1)},\mathcal{F}(A\beta^{(0)},\xi)\big).
\end{align}
We now formally describe the state evolution in \cite{han2024gradient}.
\begin{definition}\label{def:gd_mean_field_se}
	Initialize with (i) two formal variables $\Omega_{-1}\equiv n^{1/2}\mu_\ast \in \R^n$ and $\Omega_0 \equiv n^{1/2}\mu^{(0)}\in \R^n$, and (ii) a Gaussian random variable $\mathfrak{Z}^{(0)}\sim \mathcal{N}(0,\pnorm{\mu_\ast}{}^2)$. For $t=1,2,\ldots$, we execute the following steps:
	\begin{enumerate}
		\item[(S1)] Let $\Upsilon_t: \mathbb{R}^{m\times [0:t]}\to \R^m$ be defined as follows: 
		\begin{align*}
		\Upsilon_t(\mathfrak{z}^{([0:t])})\equiv   \partial_1 \mathsf{L}\bigg(\mathfrak{z}^{(t)} - \frac{1}{\phi}\sum_{s \in [1:t-1]}\eta_{s-1}\rho_{t-1,s}\Upsilon_s(\mathfrak{z}^{([0:s])}) ,\mathcal{F}(\mathfrak{z}^{(0)},\xi)\bigg)\in \R^m.
		\end{align*}
		Here the coefficients are defined via
		\begin{align*}
		\rho_{t-1,s}\equiv \E^{(0)} \partial_{ \mathfrak{W}^{(s)}} \Omega_{t-1;\pi_n}(\mathfrak{W}^{([1:t-1])})\in \R,\quad s \in [1:t-1].
		\end{align*}
		\item[(S2)] Let $\mathfrak{Z}^{([0:t])}\in \R^{[0:t]}$ and $\mathfrak{W}^{([1:t])}\in \R^{[1:t]}$ be centered Gaussian random vectors whose laws at iteration $t$ are determined via the correlation specification: 
		\begin{align*}
		\cov(\mathfrak{Z}^{(t)},\mathfrak{Z}^{(s)})
		& \equiv \E^{(0)} \prod\limits_{\ast \in \{s-1,t-1\}} \Omega_{*;\pi_n} (\mathfrak{W}^{([1:*])}),\quad s \in [0:t];\\
		\cov(\mathfrak{W}^{(t)},\mathfrak{W}^{(s)})
		& \equiv \eta_{t-1}^2\phi^{-1}\cdot  \E^{(0)} \prod_{\ast \in \{s,t\}} \Upsilon_{*;\pi_m}(\mathfrak{Z}^{([0:*])}),\quad s \in [1:t].
		\end{align*}
		\item[(S3)] Let $\Omega_t: \R^{n\times [1:t]}\to \R^n$ be defined as follows:
		\begin{align*}
		\Omega_t\big(\mathfrak{w}^{([1:t])}\big)\equiv \mathfrak{w}^{(t)}+\sum_{s \in [1:t]} (\bm{1}_{t=s}-\eta_{t-1}\tau_{t,s})\cdot  \Omega_{s-1}(\mathfrak{w}^{([1:s-1])}) + \eta_{t-1} \delta_t\cdot (n^{1/2}\mu_\ast).
		\end{align*}
		Here the coefficients are defined via
		\begin{align*}
		\tau_{t,s} &\equiv  \E^{(0)}\partial_{\mathfrak{Z}^{(s)}} \Upsilon_{t;\pi_m}(\mathfrak{Z}^{([0:t])})\in \R,\quad s \in [1:t];\\
		\delta_t &\equiv -\E^{(0)}\partial_{\mathfrak{Z}^{(0)}} \Upsilon_{t;\pi_m}(\mathfrak{Z}^{([0:t])}) \in \R.
		\end{align*}
	\end{enumerate}
\end{definition}
The above state evolution is equivalent to \cite[Definition 2.1]{han2024gradient} up to (i) time-invariant loss $\mathsf{L}$, (ii) scale/sign changes in $\{\Upsilon_t\}$, $\{\tau_{t,s}\}$ and $\{\delta_t\}$, and (iii) a special choice of the maps $\{\mathsf{P}_t:\R^n\to \R^n\}$. We summarize below the  correspondence between the notation used in \cite[Definition 2.1]{han2024gradient} and Definition \ref{def:gd_mean_field_se}:
\begin{table}[H]
	\centering
	\begin{tabular}{c||c|c|c|c|c}
		\cite[Def. 2.1]{han2024gradient} & $\mathsf{L}_t$ & $\Upsilon_t$ &$\tau_{t,s}$ & $\delta_{t}$ &$\mathsf{P}_t$  \\ \hline
		Def. \ref{def:gd_mean_field_se} & $\mathsf{L}$ & $-\eta_{t-1}\Upsilon_t/\phi$  & $-\eta_{t-1}\tau_{t,s}$ & $\eta_{t-1}\delta_{t}$ & $\mathrm{id}(\R^n)$
	\end{tabular}
\end{table}

\subsection{Relation to dynamical mean-field theory}
Under regularity conditions, \cite[Theorem 2.2]{han2024gradient} proved that in the mean-field regime $\phi \asymp 1$, for $t=\bigo(1)$ (or growing very slowly),
 \begin{align}\label{eqn:gd_mean_field_dist}
 (\beta_\ell^{(t)}) \stackrel{d}{\approx} \big(\Omega_{t;\ell}(\mathfrak{W}^{([1:t])})\big)\quad\Leftrightarrow\quad (n^{1/2}\mu_\ell^{(t)}) \stackrel{d}{\approx} \big(\Omega_{t;\ell}(\mathfrak{W}^{([1:t])})\big).
 \end{align}
 Comparing (\ref{eqn:gd_mean_field_dist}) with Theorem \ref{thm:grad_des_se}, it is natural to conjecture that, at least for $t=\bigo(1)$,
 \begin{align}\label{eqn:gd_mean_field_dist_1}
 \big(\Omega_{t;\ell}(\mathfrak{W}^{([1:t])})\big) \approx \big(n^{1/2}\mathrm{u}_{\ast;\ell}^{(t)}\big)\,\hbox{ when }\phi\gg 1.
 \end{align}
The following theorem makes the foregoing heuristic precise, and more importantly, unveils the underlying mechanism purely from the state evolution formalism in Definition \ref{def:gd_mean_field_se}.

\begin{theorem}\label{thm:mean_field_dyn_approx}
Suppose $\phi\geq 1$, (A2) holds for some $\Lambda\geq 2$ and $\mathfrak{p}= 1$, and the step sizes $\{\eta_t\}$ are bounded by $\Lambda$.  Then there exists some $c_{t}>1$  such that 
\begin{align*}
&\bigpnorm{\bm{\tau}^{[t]} - \mathrm{diag}\big(\big\{\tau_\ast^{(s-1)}\big\}_{s \in [t]}\big) }{\op}+ \max_{1\leq s\leq t}\, \bigabs{\cov(\mathfrak{W}^{(t)},\mathfrak{W}^{(s)})}\\
&\qquad +\max_{\ell \in [n]} \E^{(0)}\bigabs{\Omega_{t;\ell}(\mathfrak{W}^{([1:t])})-  n^{1/2}\mathrm{u}_{\ast;\ell}^{(t)} }^{100}\leq \big(\Lambda L_\ast^{(0)} \pnorm{\mu_\ast}{}^{-1}\big)^{c_{t}} \cdot \phi^{-1/c_{t}}.
\end{align*}
\end{theorem}
The proof of the above theorem can be found in Section \ref{section:proof_mean_field_dyn_approx}. Theorem \ref{thm:mean_field_dyn_approx} rigorously unveils the mechanism that leads to (\ref{eqn:gd_mean_field_dist_1}):
\begin{enumerate}
	\item the interaction terms $\tau_{t,s}\approx 0$ for $s\neq t$, so the gradient descent $\mu^{(t)}$ depends only approximately on the last iterate,
	\item the Gaussian noise $\mathfrak{W}^{([1:t])}$ becomes degenerate, so the gradient descent $\mu^{(t)}$ concentrates on a deterministic vector, and
	\item  the mean-field characterization $\{\Omega_{t;\ell}(\mathfrak{W}^{([1:t])}):\ell \in [n]\}$ therefore reduces to the rescaled state evolution $\{n^{1/2}\mathrm{u}_{\ast;\ell}^{(t)}: \ell \in [n]\}$. 
\end{enumerate}

Again, the numeric constant $100$ in Theorem \ref{thm:mean_field_dyn_approx} is chosen for convenience and can be replaced by an arbitrarily large constant at the cost of a potentially larger constant $c_t>0$.

\section{Proofs for Section \ref{section:main_results}}\label{section:proof_main_results}

\subsection{Proof of Theorem \ref{thm:grad_des_se}}\label{subsection:proof_grad_des_se}
Let
\begin{align}\label{def:Delta_X}
\Delta_X^{(t)}&\equiv \frac{1}{m}\sum_{i \in [m]} \big(\mathrm{id}-\E^{(0)}\big) X_i \mathfrak{S}\big(\iprod{X_i}{\mathrm{u}_X^{(t)}},\iprod{X_i}{ \mu_\ast},\xi_i\big)\in \R^n.
\end{align}
\begin{proposition}\label{prop:diff_mu_u_X}
With $\{\Delta_X^{(t)}\}$ defined in (\ref{def:Delta_X}), we have
\begin{align*}
\pnorm{\mu^{(t)}-\mathrm{u}_X^{(t)}}{}\leq \pnorm{\eta_{[0:t)}}{\infty} \cdot \mathfrak{M}_{ \mu^{(\cdot)},\mathrm{u}_X^{(\cdot)} }^{(t-1)}(X) \cdot \max_{s \in [1:t]}  \pnorm{\Delta_X^{(s-1)}}{}.
\end{align*}
\end{proposition}
\begin{proof}
Comparing (\ref{def:grad_des}) and (\ref{def:theo_grad_des}), we have 
\begin{align*}
\mu^{(t)}-\mathrm{u}_X^{(t)}&= \big(I_n-\eta_{t-1}\cdot M_{\mu^{(t-1)},\mathrm{u}_X^{(t-1)} }(X) \big)\big(\mu^{(t-1)}-\mathrm{u}_X^{(t-1)}\big)\\
&\qquad - \frac{\eta_{t-1}}{m}\cdot \big(\mathrm{id}-\E^{(0)}\big)X^\top \,\mathfrak{S}\big(X\mathrm{u}_X^{(t-1)},X \mu_\ast,\xi\big).
\end{align*}
Consequently,
\begin{align*}
\pnorm{\mu^{(t)}-\mathrm{u}_X^{(t)}}{}&\leq \Big\{\bigpnorm{I_n-\eta_{t-1}\cdot M_{\mu^{(t-1)},\mathrm{u}_X^{(t-1)} }(X)}{\op}\\
&\qquad\times \pnorm{\mu^{(t-1)}-\mathrm{u}_X^{(t-1)}}{}\Big\}+\eta_{t-1} \pnorm{\Delta_X^{(t-1)}}{}.
\end{align*}
Iterating the above display to conclude the claim.
\end{proof}

\begin{proposition}\label{prop:diff_theo_grad_des}
Suppose Assumption \ref{assump:abstract} holds for some $K,\Lambda\geq 2$ and $\mathfrak{p}\geq 1$. Then there exists some constant $c_{\mathfrak{p}}>1$ such that
\begin{align*}
\pnorm{\mathrm{u}_{\ast}^{(t)} - \mathrm{u}_{X}^{(t)} }{}&\leq \big(K\Lambda L_\ast^{(0)}(1\vee\pnorm{\eta_{[0:t)}}{\infty})\big)^{c_{\mathfrak{p}}} \cdot  \Big(1+\max_{s \in [0:t-1]} n^{1/2}\pnorm{\mathrm{u}_\ast^{(s)}}{\infty}\Big)^{c_{\mathfrak{p}}}\cdot n^{-1/2}\mathscr{M}_{ \mathrm{u}_X^{(\cdot)},\mathrm{u}_{\ast}^{(\cdot)} }^{(t-1),X}.
\end{align*}
\end{proposition}

\begin{proof}
	The proof proceeds in two steps.
	
	\noindent (\textbf{Step 1}). For a generic matrix $Z \in \R^{m\times n}$ and a generic vector $\mathrm{u}\in \R^n$, let
	\begin{align*}
	F_{\mathrm{u}}(Z)\equiv  Z^\top\mathfrak{S}\big(Z\mathrm{u},Z \mu_\ast,\xi\big)\in \R^n.
	\end{align*}
	In this step, we prove that with $\mathsf{Z}\equiv \mathsf{Z}_{m\times n}$ denoting an $m\times n$ matrix with i.i.d. $\mathcal{N}(0,1)$ entries, there exists some constant $c_{\mathfrak{p}}>1$ such that
	\begin{align}\label{ineq:diff_theo_grad_des_step1}
	\frac{1}{m}\cdot \bigpnorm{\E^{(0)} F_{\mathrm{u}}(X)-\E^{(0)} F_{\mathrm{u}}(\mathsf{Z})  }{ }\leq \frac{(K\Lambda)^{c_{\mathfrak{p}}} }{n^{1/2}} \cdot \big(1+n^{1/2}\pnorm{\mathrm{u}}{\infty}+n^{1/2}\pnorm{\mu_\ast}{\infty}\big)^{c_{\mathfrak{p}}}.
	\end{align}
	For $k \in [n]$, note that
	\begin{align*}
	\frac{1}{m}\cdot \E^{(0)} F_{\mathrm{u};k}(X) 
	& = \frac{1}{m} \sum_{i \in [m]} \E^{(0)} X_{ik} \mathfrak{S}\big(\iprod{X_i}{\mathrm{u}},\iprod{X_i}{\mu_\ast},\xi_i\big)\\
	& = \E^{(0)} X_{\pi_m,k} \mathfrak{S}\big(\iprod{X_{\pi_m}}{\mathrm{u}},\iprod{X_{\pi_m}}{\mu_\ast},\xi_{\pi_m}\big).
	\end{align*}
	By using the cumulant expansion in Lemma \ref{lem:cum_exp}, we have
	\begin{align}\label{ineq:diff_theo_grad_des_step1_1}
	&\frac{1}{m}\cdot \bigabs{\E^{(0)} F_{\mathrm{u};k}(X)-\E^{(0)} F_{\mathrm{u};k}(\mathsf{Z}) }\leq \big(\abs{\mathrm{u}_k}+\abs{\mu_{\ast,k}}\big)\nonumber\\
	&\quad\quad  \times \max_{\ell=1,2}\bigabs{\E^{(0)} \big\{\partial_\ell \mathfrak{S}\big(\iprod{X_{\pi_m}}{\mathrm{u}},\iprod{X_{\pi_m}}{\mu_\ast},\xi_{\pi_m}\big)-\partial_\ell \mathfrak{S}\big(\iprod{\mathsf{Z}_{\pi_m}}{\mathrm{u}},\iprod{\mathsf{Z}_{\pi_m}}{\mu_\ast},\xi_{\pi_m}\big)\big\}}\nonumber\\
	&\qquad + \bigg\{c K^3\cdot \big(\mathrm{u}_k^2+\mu_{\ast,k}^2\big)\cdot \frac{1}{m}\sum_{i \in [m]}\nonumber\\
	&\qquad\qquad \sup_{s \in [0,1]^n}\max_{\substack{\alpha\in \mathbb{Z}_{\geq 0}^2, \abs{\alpha}=2 } }\max_{W_i \in \{X_i,\mathsf{Z}_i\}}\E^{(0),1/2} \Big(\partial_{\alpha}\mathfrak{S}\big(\iprod{s\circ W_i}{\mathrm{u}},\iprod{s\circ W_i}{\mu_\ast},\xi_i\big)\Big)^2\bigg\}\nonumber\\
	&\equiv I_{1;k}+I_{2;k}.
	\end{align}
	First we consider $I_{2;k}$. Using the assumption on $\mathfrak{S}$, we may control
	\begin{align}\label{ineq:diff_theo_grad_des_step1_2}
	I_{2;k}&\leq  (K\Lambda)^{c_{\mathfrak{p}}} \cdot \big(\mathrm{u}_k^2+\mu_{\ast,k}^2\big)\cdot \big(1+\pnorm{\mathrm{u}}{}+ \pnorm{\mu_\ast}{}\big)^{c_{\mathfrak{p}}}\nonumber\\
	&\leq \frac{(K\Lambda)^{c_{\mathfrak{p}} } }{n}\cdot \big(1+n^{1/2}\pnorm{\mathrm{u}}{\infty}+n^{1/2} \pnorm{\mu_\ast}{\infty}\big)^{c_{\mathfrak{p}}}.
	\end{align}
	Next we consider $I_{1;k}$. Fixed $i \in [m]$, let $\{X_{ij}(t)\}_{t \in [0,1]}\subset \R^n$ be the Lindeberg path from $X_i$ to $\mathsf{Z}_i$ at position $j$ with $(X_{ij}(t))_{j}=tX_{ij}$, and $\{\mathsf{Z}_{ij}(t)\}_{t \in [0,1]}\subset \R^n$ be the Lindeberg path with $(\mathsf{Z}_{ij}(t))_{j}=t\mathsf{Z}_{ij}$. For notational convenience, we write $\Xi_{\mathrm{u};i}^{(\ell)}(x)\equiv \partial_\ell \mathfrak{S}(\iprod{x}{\mathrm{u}},\iprod{x}{\mu_\ast},\xi_i)$ for $x \in \R^n$, $\ell=1,2$ and $i \in [m]$. Using Chatterjee's Lindeberg principle in Lemma \ref{lem:lindeberg}, 
	\begin{align}\label{ineq:diff_theo_grad_des_step1_3}
	I_{1;k}&\leq \big(\abs{\mathrm{u}_k}+\abs{\mu_{\ast,k}}\big)\max_{\ell=1,2}  \max_{W\in \{X,\mathsf{Z}\}}\biggabs{\sum_{j \in [n]} \E^{(0)} W_{\pi_m j}^3 \int_0^1 \partial_{j}^3  \Xi_{\mathrm{u};\pi_m}^{(\ell)} \big(W_{\pi_m j}(t)\big)(1-t)^2\,\d{t}}\nonumber\\
	&\leq cK^3\cdot \big(\abs{\mathrm{u}_k}+\abs{\mu_{\ast,k}}\big)\cdot\max_{\ell=1,2}  \max_{W\in \{X,\mathsf{Z}\}}\sum_{j \in [n]} \sup_{t \in [0,1]} \E^{(0),1/2} \Big(\partial_{j}^3  \Xi_{\mathrm{u};\pi_m}^{(\ell)}\big(W_{\pi_m j}(t)\big)\Big)^2\nonumber\\
	&\leq \frac{(K\Lambda)^{c_{\mathfrak{p}}}}{n}\cdot \big(1+n^{1/2}\pnorm{\mathrm{u}}{\infty}+n^{1/2} \pnorm{\mu_\ast}{\infty}\big)^{c_{\mathfrak{p}}}.
	\end{align}
	Now combining (\ref{ineq:diff_theo_grad_des_step1_1})-(\ref{ineq:diff_theo_grad_des_step1_3}) yields the claimed estimate (\ref{ineq:diff_theo_grad_des_step1}).
	
	\noindent (\textbf{Step 2}). Using the estimate (\ref{ineq:diff_theo_grad_des_step1}), we have 
	\begin{align}\label{ineq:diff_theo_grad_des_step2_1}
	&\biggpnorm{\mathrm{u}_{\ast}^{(t)} - \bigg[ \mathrm{u}_{\ast}^{(t-1)} - \frac{\eta_{t-1}}{m} \cdot \E^{(0)}  X^\top \,\mathfrak{S}\big(X\mathrm{u}_{\ast}^{(t-1)},X \mu_\ast,\xi\big)\bigg]}{}\nonumber\\
	&\leq \frac{\abs{\eta_{t-1}}}{m}\cdot \bigpnorm{\E^{(0)} F_{\mathrm{u}_{\ast}^{(t-1)}}(X)-\E^{(0)} F_{\mathrm{u}_{\ast}^{(t-1)}}(\mathsf{Z})}{}\nonumber\\
	&\leq (K\Lambda)^{c_{\mathfrak{p}}} \abs{\eta_{t-1}}\cdot n^{-1/2} \cdot \big(1+n^{1/2}\pnorm{\mathrm{u}_{\ast}^{(t-1)}}{\infty}+n^{1/2}\pnorm{\mu_\ast}{\infty}\big)^{c_{\mathfrak{p}}}\equiv \err_{\ast}^{(t-1)}.
	\end{align}
	Comparing the above display (\ref{ineq:diff_theo_grad_des_step2_1}) to the definition of $\mathrm{u}_{X}^{(t)} $ in (\ref{def:theo_grad_des}), 
	we have
	\begin{align*}
	\pnorm{\mathrm{u}_{\ast}^{(t)} - \mathrm{u}_{X}^{(t)} }{}&\leq \bigpnorm{I_n- \eta_{t-1}\cdot \mathbb{M}_{\mathrm{u}_{X}^{(t-1)},\mathrm{u}_{\ast}^{(t-1)}}^X }{\op}\cdot \pnorm{\mathrm{u}_{\ast}^{(t-1)} - \mathrm{u}_{X}^{(t-1)} }{}+ \err_{\ast}^{(t-1)}.
	\end{align*}
	Iterating the above estimate, we have
	\begin{align}\label{ineq:diff_theo_grad_des_step2_2}
	\pnorm{\mathrm{u}_{\ast}^{(t)} - \mathrm{u}_{X}^{(t)} }{}&\leq \mathscr{M}_{ \mathrm{u}_X^{(\cdot)},\mathrm{u}_{\ast}^{(\cdot)} }^{(t-1),X} \cdot \max_{s \in [1:t]} \err_{\ast}^{(s-1)}.
	\end{align}
	The claimed estimate follows by combining (\ref{ineq:diff_theo_grad_des_step2_1}) and (\ref{ineq:diff_theo_grad_des_step2_2}).
\end{proof}

\begin{proof}[Proof of Theorem \ref{thm:grad_des_se}]
Combining Propositions \ref{prop:diff_mu_u_X} and \ref{prop:diff_theo_grad_des}, we have
\begin{align}\label{ineq:grad_des_se_1}
\pnorm{\mu^{(t)}-\mathrm{u}_{\ast}^{(t)}}{}&\leq  \pnorm{\mu^{(t)}-\mathrm{u}_{X}^{(t)}}{}+ \pnorm{\mathrm{u}_{\ast}^{(t)} - \mathrm{u}_{X}^{(t)} }{}\nonumber\\
&\leq  \big(K\Lambda  L_{\ast}^{(0)} (1\vee \pnorm{\eta_{[0:t)}}{\infty}) \big)^{c_{\mathfrak{p}}} \cdot \Big(1+\max_{s \in [0:t-1]} n^{1/2}\pnorm{\mathrm{u}_\ast^{(s)}}{}\Big)^{c_{\mathfrak{p}}} \nonumber\\
&\qquad \times \Big\{\mathfrak{M}_{ \mu^{(\cdot)},\mathrm{u}_X^{(\cdot)} }^{(t-1)}(X)\cdot\max_{s \in [1:t]}  \pnorm{\Delta_X^{(s-1)}}{}+ \mathscr{M}_{ \mathrm{u}_X^{(\cdot)},\mathrm{u}_{\ast}^{(\cdot)} }^{(t-1),X}\cdot n^{-1/2}\Big\}.
\end{align}
Using the assumption on $\mathfrak{S}$, there exists some $\alpha_{\mathfrak{p}} \in (0,1)$ such that the following control of the $\textrm{sub-Weibull}(\alpha_{\mathfrak{p}})$ norms hold:
\begin{align*}
\max_{i \in [m], j \in [n]} \bigpnorm{X_{ij} \mathfrak{S}\big(\iprod{X_i}{\mathrm{u}_X^{(t)}},\iprod{X_i}{ \mu_\ast},\xi_i\big)}{\psi_{\alpha_{\mathfrak{p}} } }\leq \big(K\Lambda\cdot (1+\pnorm{\mathrm{u}_X^{(t)}}{}+\pnorm{\mu_\ast}{}\big)^{c_{\mathfrak{p}}}.
\end{align*}
Using generalized Bernstein's inequality for sub-Weilbull variables in the form of \cite[Theorem 3.1]{kuchibhotla2022moving}, for each $j \in [n]$, with $\Prob^{(0)}$-probability $1-2e^{-x}$, 
\begin{align}\label{ineq:grad_des_se_2}
\abs{m \Delta_{X;j}^{(t)}}\leq \big(K\Lambda L_\ast^{(0)} (1+\pnorm{\mathrm{u}_X^{(t)}}{})\big)^{c_{\mathfrak{p}}}\cdot \big(\sqrt{m x} + x^{1/\alpha_{\mathfrak{p}}} \big).
\end{align}
Consequently, combining (\ref{ineq:grad_des_se_1}) and (\ref{ineq:grad_des_se_2}) and using a union bound, on an event with $\Prob^{(0)}$-probability at least $1-cnt e^{-x}$,
\begin{align*}
\pnorm{\mu^{(t)}-\mathrm{u}_{\ast}^{(t)}}{}
&\leq  \big(K\Lambda  L_{\ast}^{(0)} (1\vee \pnorm{\eta_{[0:t)}}{\infty}) \big)^{c_{\mathfrak{p}}}\cdot \Big(1+\max_{s \in [1:t-1]} n^{1/2}\pnorm{\mathrm{u}_\ast^{(s)}}{\infty}\vee \pnorm{\mathrm{u}_X^{(s)}}{}\Big)^{c_{\mathfrak{p}}} \\
&\qquad \times \bigg\{\bigg(\frac{ x^{1/2} }{ \phi^{1/2}}+\frac{x^{1/\alpha_{\mathfrak{p}}} }{(m\phi)^{1/2}}\bigg)\cdot \mathfrak{M}_{ \mu^{(\cdot)},\mathrm{u}_X^{(\cdot)} }^{(t-1)}(X)+ \frac{1}{n^{1/2}}\cdot  \mathscr{M}_{ \mathrm{u}_X^{(\cdot)},\mathrm{u}_{\ast}^{(\cdot)} }^{(t-1),X} \bigg\}.
\end{align*}
The claim follows.
\end{proof}

\subsection{Proof of Theorem \ref{thm:grad_des_incoh}}\label{subsection:proof_grad_des_incoh}

\begin{lemma}\label{lem:grad_des_loo_diff}
	Suppose (A1) holds for some $K\geq 2$, and for some $\Lambda\geq 2,\mathfrak{p}\geq 1$,
	\begin{align*}
	\max_{i \in [m]}\abs{\mathfrak{S}(x_0,z_0,\xi_i)}\leq \Lambda\cdot (1+\abs{x_0}+\abs{z_0})^{\mathfrak{p}},\quad \forall x_0,z_0 \in \R.
	\end{align*}
	Then there exist some universal constant $c_0$ and another constant $c_{\mathfrak{p}}>1$ such that for any $i \in [m]$ and $x\geq 1$, it holds with $\Prob^{(0)}$-probability at least $1-c_0 e^{-x}$ that
	\begin{align*}
	\pnorm{\mu^{(t)}-\mu_{[-i]}^{(t)}}{}&\leq \big(K\Lambda L_\ast^{(0)} (1\vee \pnorm{\eta_{[0:t)}}{\infty})\cdot x \big)^{c_{\mathfrak{p}}}\cdot \Big(1+\max_{s \in [1:t-1]} \pnorm{\mu_{[-i]}^{(s)}}{}\Big) \\
	&\qquad\qquad \times (m\phi)^{-1/2}\cdot   \mathfrak{M}_{\mu^{(\cdot)}, \mu_{[-i]}^{(\cdot)} }^{(t-1)}(X).
	\end{align*}
\end{lemma}
\begin{proof}
	Note that
	\begin{align}\label{def:grad_des_loo}
	\mu_{[-i]}^{(t)}&=\mu_{[-i]}^{(t-1)}-\frac{\eta_{t-1}}{m}\cdot X_{[-i]}^\top \mathfrak{S}\big(X_{[-i]}\mu_{[-i]}^{(t-1)},X_{[-i]}\mu_\ast,\xi\big),
	\end{align}
	Let $\Delta \mu_{[-i]}^{(t)}\equiv \mu^{(t)}-\mu_{[-i]}^{(t)}$. Comparing (\ref{def:grad_des}) and the above display  (\ref{def:grad_des_loo}), with
	\begin{align}\label{ineq:grad_des_loo_diff_0}
	\Xi_{[-i]}^{(t)}&\equiv X^\top \mathfrak{S}\big(X\mu_{[-i]}^{(t)},X\mu_\ast,\xi\big)- X_{[-i]}^\top \mathfrak{S}\big(X_{[-i]}\mu_{[-i]}^{(t)},X_{[-i]}\mu_\ast,\xi\big)\in \R^n,
	\end{align}
	we have
	\begin{align}\label{ineq:grad_des_loo_diff_1}
	\Delta \mu_{[-i]}^{(t)}& = \big(I_n-\eta_{t-1} M_{\mu^{(t-1)},\mu_{[-i]}^{(t-1)} }(X)\big)\Delta \mu_{[-i]}^{(t-1)}-\frac{\eta_{t-1}}{m}\cdot \Xi_{[-i]}^{(t-1)}.
	\end{align}
	Note that we may control $\Xi_{[-i]}^{(t-1)}$ as 
	\begin{align}\label{ineq:grad_des_loo_diff_1_0}
	\pnorm{\Xi_{[-i]}^{(t-1)}}{}& \leq \pnorm{X_i}{}\cdot \abs{  \mathfrak{S}\big(\iprod{X_i}{\mu_{[-i]}^{(t-1)}},\iprod{X_i}{\mu_\ast},\xi_i\big)}\nonumber\\
	&\leq \Lambda \pnorm{X_i}{}\cdot\big(1+\abs{\bigiprod{X_i}{ \mu_{[-i]}^{(t-1)}}}+ \abs{\iprod{X_i}{\mu_\ast}}\big)^{c_{\mathfrak{p}}},
	\end{align}
	so by (\ref{ineq:grad_des_loo_diff_1}) we have 
	\begin{align*}
	\pnorm{\Delta \mu_{[-i]}^{(t)}}{}&\leq \bigpnorm{I_n-\eta_{t-1} M_{\mu^{(t-1)},\mu_{[-i]}^{(t-1)} }(X)}{\op}\cdot \pnorm{\Delta \mu_{[-i]}^{(t-1)}}{}\\
	&\qquad + \Lambda \abs{\eta_{t-1}}\cdot  \big(m^{-1}\pnorm{X_i}{}\big)\cdot\big(1+\abs{\bigiprod{X_i}{ \mu_{[-i]}^{(t-1)}}}+ \abs{\iprod{X_i}{\mu_\ast}}\big)^{c_{\mathfrak{p}}}.
	\end{align*}
	Iterating the above display until $t=1$ and using the initial condition $\Delta \mu_{[-i]}^{(0)}=0$, 
	\begin{align*}
	\pnorm{\Delta \mu_{[-i]}^{(t)}}{}&\leq \Lambda \pnorm{\eta_{[0:t)}}{\infty}\cdot \big(m^{-1}\pnorm{X_i}{}\big)\cdot  \mathfrak{M}_{\mu^{(\cdot)}, \mu_{[-i]}^{(\cdot)} }^{(t-1)}(X)\\
	&\qquad\qquad \times  \max_{s \in [0:t-1]} \big(1+\abs{\bigiprod{X_i}{ \mu_{[-i]}^{(s)} } }+ \abs{\iprod{X_i}{\mu_\ast}}\big)^{c_{\mathfrak{p}}}.
	\end{align*}
	As $X_i$ is independent of $\mu_{[-i]}^{(\cdot)}$, standard sub-gaussian estimates for both the linear form $\iprod{X_i}{\cdot}$ and $\pnorm{X_i}{}$ yield that, for $x\geq 1$, with $\Prob^{(0)}$-probability at least $1-2e^{-x}$, 
	\begin{align*}
	\pnorm{\Delta \mu_{[-i]}^{(t)}}{}&\leq \frac{(K\Lambda \pnorm{\eta_{[0:t)}}{\infty}\cdot x)^{c_{\mathfrak{p}}} }{(m\phi)^{1/2}}\cdot \max_{s \in [0:t-1]}\big(1+\pnorm{\mu_{[-i]}^{(s)}}{}+\pnorm{\mu_\ast}{}\big)^{c_{\mathfrak{p}}}\cdot \mathfrak{M}_{\mu^{(\cdot)}, \mu_{[-i]}^{(\cdot)} }^{(t-1)}(X).
	\end{align*}
	The claim follows.
\end{proof}

\begin{proof}[Proof of Theorem \ref{thm:grad_des_incoh}]
	Fix $i \in [m]$. Note that
	\begin{align*}
	\abs{\bigiprod{X_i}{\mu^{(t)}}}\leq \abs{\bigiprod{X_i}{\mu_{[-i]}^{(t)} }}+\pnorm{X_i}{}\cdot \pnorm{\mu^{(t)}-\mu_{[-i]}^{(t)}}{}.
	\end{align*}
	As $X_i$ is independent of $\mu_{[-i]}^{(t)}$, we may use a standard sub-gaussian concentration for both the linear form $\iprod{X_i}{\cdot}$ in the first term and $\pnorm{X_i}{}$ in the second term to conclude that, with $\Prob^{(0)}$-probability at least $1-ce^{-x}$,
	\begin{align*}
	\abs{\bigiprod{X_i}{\mu^{(t)}}}\leq cK\sqrt{x}\cdot \big(\pnorm{\mu_{[-i]}^{(t)}}{}+\sqrt{n+x}\cdot \pnorm{\mu^{(t)}-\mu_{[-i]}^{(t)}}{}\big).
	\end{align*}
	Now using Lemma \ref{lem:grad_des_loo_diff}, with $\Prob^{(0)}$-probability at least $1-c e^{-x}$,
	\begin{align*}
	\abs{\bigiprod{X_i}{\mu^{(t)}}} &\leq\big(K\Lambda L_\ast^{(0)} (1\vee \pnorm{\eta_{[0:t)}}{\infty}) \cdot x\big)^{c_{\mathfrak{p}}} \\
	&\qquad\qquad \times  \Big(1+\max_{s \in [1:t]} \pnorm{\mu_{[-i]}^{(s)}}{}\Big)  \cdot \Big(1+\phi^{-1} \mathfrak{M}_{\mu^{(\cdot)}, \mu_{[-i]}^{(\cdot)} }^{(t-1)}(X)\Big).
	\end{align*}
	The claim now follows by a further union bound over $i \in [m]$.
\end{proof}

\subsection{Proof of Corollary \ref{cor:general_se}}\label{subsection:proof_general_se}
	Using a trivial bound, we have $
	\mathscr{M}_{\cdot,\cdot}^{(t),X} \leq \Lambda^{ct}$ and $ \mathfrak{M}_{\cdot,\cdot}^{(t)}(X)\leq (\Lambda \hat{\Sigma})^{ct}$, where $\hat{\Sigma}\equiv m^{-1}\sum_{i \in [m]} X_iX_i^\top$ is the sample covariance, and $c>0$ is a universal constant whose numeric value may vary from line to line. By \cite[Theorem 4.7.1 \& Exercise 4.7.3]{vershynin2018high}, with probability at least $1-e^{-m}$, we have $
	 \mathscr{M}_{\cdot,\cdot}^{(t),X}  \vee \mathfrak{M}_{\cdot,\cdot}^{(t)}(X)\leq \big(1+ \pnorm{\eta_{[0:t)}}{\infty}\Lambda\big)^{ct}$. 
	
	\noindent (1). We shall now provide a bound for $\pnorm{\mathrm{u}_X^{(\cdot)}}{}$. Using (\ref{def:theo_grad_des}),
	\begin{align}\label{ineq:cor_general_se_1_0}
	\mathrm{u}_X^{(t)} 
	& = \Big(I_n - \eta_{t-1}\cdot \mathbb{M}_{\mathrm{u}_X^{(t-1)},\mu_\ast}^X \Big)\, \mathrm{u}_X^{(t-1)} - \frac{\eta_{t-1}}{m}\cdot \E^{(0)}  X^\top \mathfrak{S}\big(X\mu_\ast,X \mu_\ast,\xi\big).
	\end{align}
	Then using the Lipschitz condition on $\mathfrak{S}(\cdot,\cdot,\xi)$, 
	\begin{align*}
	\pnorm{\mathrm{u}_X^{(t)} }{}&\leq  \big(1+ \pnorm{\eta_{[0:t)}}{\infty}\Lambda\big)\cdot \pnorm{\mathrm{u}_X^{(t-1)} }{}+(K \Lambda L_\ast^{(0)} )^{c}.
	\end{align*}
	Iterating the above estimate leads to 
	\begin{align}\label{ineq:cor_general_se_1}
	\pnorm{\mathrm{u}_X^{(t)} }{}&\leq \big(K \Lambda L_\ast^{(0)} (1+ \pnorm{\eta_{[0:t)}}{\infty})\big)^{c t}.
	\end{align}	
	\noindent (2). We shall now provide an estimate for $\pnorm{\mathrm{u}_{\ast}^{(t)}}{\infty}$. Let
	\begin{align*}
	\mathscr{M}_{\tau_\ast^{(\cdot)}}^{(t)}\equiv 1+ \max_{\tau \in [0:t]}\sum_{s \in [0:\tau]} \bigg(\prod_{r \in [s:\tau]}\abs{1-\eta_r\cdot \tau_\ast^{(r)}}\bigg).
	\end{align*}
	It is easy to see $\mathscr{M}_{\tau_\ast^{(\cdot)}}^{(t)}\leq \Lambda^{ct}$. Iterating the recursion defined in (\ref{def:u_ast}),
	\begin{align}\label{ineq:cor_general_se_2_0}
	\pnorm{\mathrm{u}_{\ast}^{(t)}}{\infty}&\leq c \cdot  L_\ast^{(0)} (1+ \pnorm{\eta_{[0:t)}}{\infty}) \cdot \Big(1+\max_{s \in [0:t-1]} \abs{\delta_\ast^{(s)} }\Big)\cdot n^{-1/2} \mathscr{M}_{\tau_\ast^{(\cdot)}}^{(t-1)}.
	\end{align}
	Using the trivial estimate for $\abs{\delta_\ast^{(\cdot)} }\leq \Lambda$, we then arrive at 
	\begin{align}\label{ineq:cor_general_se_2}
	n^{1/2}\pnorm{\mathrm{u}_\ast^{(t)} }{\infty}&\leq \big(K \Lambda L_\ast^{(0)} (1+ \pnorm{\eta_{[0:t)}}{\infty})\big)^{c t}.
	\end{align}
	
	\noindent (3). We shall now provide a bound for $\pnorm{\mu_{[-i]}^{(\cdot)}}{}$. Using (\ref{def:grad_des_loo}),
	\begin{align*}
	\mu_{[-i]}^{(t)}&=\big(I_n-\eta_{t-1} M_{\mu_{[-i]}^{(t-1)},\mu_\ast}(X_{[-i]})\big)\mu_{[-i]}^{(t-1)} -\frac{\eta_{t-1}}{m}\cdot X_{[-i]}^\top \mathfrak{S}\big(X_{[-i]}\mu_\ast,X_{[-i]}\mu_\ast,\xi\big).
	\end{align*}
	Then using the Lipschitz condition on $\mathfrak{S}(\cdot,\cdot,\xi)$,
	\begin{align*}
	\pnorm{\mu_{[-i]}^{(t)}}{}&\leq \big(1+ \pnorm{\eta_{[0:t)}}{\infty}\Lambda\big)\cdot \pnorm{\mu_{[-i]}^{(t-1)}}{}\\
	&\qquad\qquad  +\frac{ c \Lambda \abs{\eta_{t-1}}}{m}\cdot \pnorm{X_{[-i]}}{\op}\cdot \bigg\{\sum_{k\neq i} \big(1+\abs{\iprod{X_k}{\mu_\ast}}\big)^{2}\bigg\}^{1/2}.
	\end{align*}
	Iterating the above display and using standard sub-gaussian estimates for $\pnorm{X_{[-i]}}{\op}$ and $\{\iprod{X_k}{\mu_\ast}:k\neq i\}$, for $x\geq 1$, with $\Prob^{(0)}$-probability at least $1-2me^{-x}$, uniform in $t\geq 1$, 
	\begin{align}\label{ineq:cor_general_se_3}
	\pnorm{\mu_{[-i]}^{(t)}}{}&\leq \big(K \Lambda L_\ast^{(0)} (1+ \pnorm{\eta_{[0:t)}}{\infty})\cdot x\big)^{c t}.
	\end{align}
	Now we may apply Theorems \ref{thm:grad_des_se} and \ref{thm:grad_des_incoh} along with the estimates in (\ref{ineq:cor_general_se_1}), (\ref{ineq:cor_general_se_2}) and (\ref{ineq:cor_general_se_3}) to conclude. \qed

\section{Proofs for Section \ref{section:long_time_dynamics}}

\subsection{Proof of Proposition \ref{prop:nonconvex_generic_conv_long_contraction}}\label{subsection:proof_nonconvex_generic_conv_long_contraction}

	\noindent (1). Under (N1), the proof follows immediately by noting (\ref{eqn:dynamics_gaussian_N1}) and invoking the condition (\ref{cond:nonconvex_generic_conv_long_lowerbdd}) in (N2).
	
	\noindent (2). Note that by Gaussian integration-by-parts applied twice, with $\mathsf{Z}_n\sim \mathcal{N}(0,I_n)$ and $s_{ab}\equiv  \E^{(0)}   \partial_{1ab}\mathfrak{S}\big(\bigiprod{\mathsf{Z}_n}{\mu_\ast }, \iprod{\mathsf{Z}_n}{\mu_\ast},\xi_{\pi_m}\big)$, for $a,b \in [2]$,
	\begin{align*}
	\mathbb{M}_{\mu_\ast,\mu_\ast }^{\mathsf{Z}}&= \E^{(0)}   \partial_1\mathfrak{S}\big(\bigiprod{\mathsf{Z}_n}{\mu_\ast }, \iprod{\mathsf{Z}_n}{\mu_\ast},\xi_{\pi_m}\big) \mathsf{Z}_n\mathsf{Z}_n^\top\\
	&=\E^{(0)}   \partial_1\mathfrak{S}\big(\bigiprod{\mathsf{Z}_n}{\mu_\ast }, \iprod{\mathsf{Z}_n}{\mu_\ast},\xi_{\pi_m}\big)\cdot I_n+ \bigg(\sum_{a,b\in [2]} s_{ab}\bigg) \,\mu_\ast \mu_\ast^\top.
	\end{align*}
	This means that, with  $\mathsf{Z}_1\sim \mathcal{N}(0,1)$, $\mathbb{M}_{\mu_\ast,\mu_\ast }^{\mathsf{Z}}$ has two distinct eigenvalues given by $\E^{(0)}   \partial_1\mathfrak{S}\big(\pnorm{\mu_\ast}{}\mathsf{Z}_1, \pnorm{\mu_\ast}{}\mathsf{Z}_1,\xi_{\pi_m}\big)$ and $\E^{(0)}   \mathsf{Z}_1^2 \partial_1\mathfrak{S}\big(\pnorm{\mu_\ast}{}\mathsf{Z}_1, \pnorm{\mu_\ast}{}\mathsf{Z}_1,\xi_{\pi_m}\big)$, i.e., $\lambda_{\min}\big(\mathbb{M}_{\mu_\ast,\mu_\ast }^{\mathsf{Z}}\big)=\varrho_\ast$.
	
	On the other hand, under the assumption, we have $\pnorm{\mathrm{u}_\ast^{(t)}}{}\leq c \mathcal{D}_n L_\ast^{(0)}$, so for some $c_2=c_2(\mathfrak{p},c_0)>0$,
	\begin{align}\label{ineq:nonconvex_generic_conv_long_contraction_1}
	&\bigpnorm{\mathbb{M}_{\mathrm{u}_\ast^{(t)},\mu_\ast }^{\mathsf{Z}}- \mathbb{M}_{\mu_\ast,\mu_\ast }^{\mathsf{Z}}}{\op}\nonumber\\
	&\leq  \Lambda^{c_{\mathfrak{p}}}\cdot \bigpnorm{\E^{(0)} \big(1+\iprod{\mathsf{Z}_n}{\mathrm{u}_\ast^{(t)} }+\iprod{\mathsf{Z}_n}{\mu_\ast }\big)^{\mathfrak{p}-1 }\abs{ \iprod{\mathsf{Z}_n}{\mathrm{u}_\ast^{(t)}-\mu_\ast } }\mathsf{Z}_n\mathsf{Z}_n^\top}{\op}\nonumber\\
	&\leq \big(\Lambda (1+\pnorm{\mathrm{u}_\ast^{(t)}}{}+\pnorm{\mu_\ast}{})\cdot \log n\big)^{c_{\mathfrak{p}} }\cdot\big( \pnorm{  \mathrm{u}_\ast^{(t)}-\mu_\ast }{}+n^{-200 c_{\mathfrak{p}}}\big)\nonumber\\
	&\leq  \big(c\Lambda L_\ast^{(0)} \mathcal{D}_n \log n \big)^{c_{\mathfrak{p}}}\cdot  \big( \pnorm{  \mathrm{u}_\ast^{(t)}-\mu_\ast }{}+n^{-200 c_{\mathfrak{p}}}\big)\nonumber\\
	&\leq  \big(c\Lambda L_\ast^{(0)} \mathcal{D}_n \log n \big)^{c_{\mathfrak{p}}}\cdot \big[ (1-\epsilon_0)^{(t-\tau_n)_+}+n^{-200 c_{\mathfrak{p}}}\big]\nonumber\\
	&\leq (\log n)^{c_2}\cdot (\mathcal{D}_n)^{c_{\mathfrak{p}}}\cdot  (1-\epsilon_0)^{(t-\tau_n)_+}+(\log n)^{-1}.
	\end{align}
    Now for some large enough $c_1=c_1(c_2,\epsilon_0)=c_1(\mathfrak{p},c_0,\epsilon_0)>0$, if $t\geq \tau_n+c_1(\log \log n+\log \mathcal{D}_n)$ and $n\geq n_0(\mathfrak{p},c_0,\epsilon_0,\varrho_\ast)$,
	\begin{align*}
	\bigpnorm{\mathbb{M}_{\mathrm{u}_\ast^{(t)},\mu_\ast }^{\mathsf{Z}}- \mathbb{M}_{\mu_\ast,\mu_\ast }^{\mathsf{Z}}}{\op}\leq 2(\log n)^{-1}\leq \varrho_\ast/2\,\Rightarrow\,  \lambda_{\min}\Big(\mathbb{M}_{\mathrm{u}_\ast^{(t)},\mu_\ast }^{\mathsf{Z}}\Big)\geq \varrho_\ast/2,
	\end{align*}
	proving the claim. \qed

\subsection{Preparatory lemmas}\label{subsection:proof_nonconvex_generic_conv_long_lemma}

We need a few auxiliary results tailored to the assumptions in Theorem \ref{thm:nonconvex_generic_conv_long} before proving Propositions \ref{prop:nonconvex_generic_conv_long_Mscr} and \ref{prop:nonconvex_generic_conv_long_Mcal}.

\begin{lemma}\label{lem:nonconvex_generic_conv_long_tau}
	Suppose (N1)-(N2) in Theorem \ref{thm:nonconvex_generic_conv_long} hold and $n\geq 3$. Then $\mathbb{M}_{ \mathrm{u}_{\ast}^{(t)},\mu_\ast }^{\mathsf{Z}}$ must have $n-2$ eigenvalues equal to $\tau_\ast^{(t)}$ and therefore $
	\inf_{t \geq t_n} \tau_\ast^{(t)}\geq \sigma_\ast$.
\end{lemma}
\begin{proof}
	By Gaussian integration-by-parts applied twice, 
	\begin{align*}
	\mathbb{M}_{ \mathrm{u}_{\ast}^{(t)},\mu_\ast }^{\mathsf{Z}}&= \E^{(0)}   \partial_1\mathfrak{S}\big(\bigiprod{\mathsf{Z}_n}{U \mathrm{u}_{\ast}^{(t)}+(1-U) \mu_\ast }, \iprod{\mathsf{Z}_n}{\mu_\ast},\xi_{\pi_m}\big) \mathsf{Z}_n\mathsf{Z}_n^\top\\
	&=\tau_\ast^{(t)} I_n+s_{11}\mathrm{u}_{\ast}^{(t)}\mathrm{u}_{\ast}^{(t),\top}+ s_{12}\mathrm{u}_{\ast}^{(t)}\mu_\ast^\top+ s_{21} \mu_\ast \mathrm{u}_{\ast}^{(t),\top}+ s_{22}\mu_\ast\mu_\ast^\top\\
	&\equiv \tau_\ast^{(t)} I_n+ \mathbb{L}_{ \mathrm{u}_{\ast}^{(t)},\mu_\ast }.
	\end{align*}
	where $\{s_{ab}\}_{a,b \in [2]}\subset \R$ and $\mathbb{L}_{ \mathrm{u}_{\ast}^{(t)},\mu_\ast }\in \R^{n\times n}$ is a symmetric matrix of at most rank $2$. Consequently, for $n\geq 3$, $\mathbb{M}_{ \mathrm{u}_{\ast}^{(t)},\mu_\ast }^{\mathsf{Z}}$ must have $n-2$ eigenvalues equal to $\tau_\ast^{(t)}$, which by (N2), must be bounded from below by $\sigma_\ast$ for $t \geq t_n$.
\end{proof}

\begin{lemma}\label{lem:nonconvex_generic_conv_long_min_eig}
Suppose (A2) holds for some $\Lambda\geq 2$ and $\mathfrak{p}\geq 1$. Suppose (N1)-(N2) in Theorem \ref{thm:nonconvex_generic_conv_long} hold, and there exists a constant $c_0>0$ such that 
\begin{align*}
\Lambda L_\ast^{(0)}\vee \log m \vee \Big\{ \sup_{t\geq 1} \pnorm{\mathrm{u}_\ast^{(t)}}{} \Big\} \vee \mathcal{B}_0(t_n)\leq (\log n)^{c_0}.
\end{align*}
Fix step sizes $\{\eta_r\}$ such that (\ref{cond:nonconvex_generic_conv_long_eta}) holds with  $\eta_\ast \in (0,1)$.  Then there exists some $c_1=c_1(\mathfrak{p},c_0,\eta_\ast,\sigma_\ast)>0$ such that for $t_n'\equiv t_n+c_1 \log \log n$, we have $\inf_{t \geq t_n'} \lambda_{\min}\big(\mathbb{M}_{ \mathrm{u}_{\ast}^{(t)},\mathrm{u}_{\ast}^{(t)} }^{\mathsf{Z}}\big)\geq 3\sigma_\ast/4$ whenever $n\geq n_0(\mathfrak{p},c_0)$.
\end{lemma}
\begin{proof}
By a similar calculation to (\ref{ineq:nonconvex_generic_conv_long_contraction_1}) with $\sup_{t\geq 1} \pnorm{\mathrm{u}_\ast^{(t)}}{}\leq (\log n)^{c_0}$, for some $c_2=c_2(\mathfrak{p},c_0)>0$,
\begin{align*}
\bigpnorm{\mathbb{M}_{ \mathrm{u}_{\ast}^{(t)},\mathrm{u}_{\ast}^{(t)} }^{\mathsf{Z}}-\mathbb{M}_{ \mathrm{u}_{\ast}^{(t)},\mu_\ast }^{\mathsf{Z}}}{\op}\leq (\log n)^{c_2}\cdot \big( \pnorm{  \mathrm{u}_\ast^{(t)}-\mu_\ast }{}+n^{-200}\big).
\end{align*}
By Proposition \ref{prop:nonconvex_generic_conv_long_contraction}-(1), we may further control for $n\geq n_0(\mathfrak{p},c_0)$ that
\begin{align*}
\bigpnorm{\mathbb{M}_{ \mathrm{u}_{\ast}^{(t)},\mathrm{u}_{\ast}^{(t)} }^{\mathsf{Z}}-\mathbb{M}_{ \mathrm{u}_{\ast}^{(t)},\mu_\ast }^{\mathsf{Z}}}{\op} & \leq (\log n)^{c_2}\cdot \big(\mathcal{B}_0(t_n)\cdot (1-\eta_\ast\sigma_\ast)^{(t-t_n)_+}+n^{-200}\big)\\
&\leq (\log n)^{c_2}\cdot (1-\eta_\ast\sigma_\ast)^{(t-t_n)_+}+n^{-1}.
\end{align*}
The claim now follows by the assumption (N2).
\end{proof}

\begin{lemma}\label{lem:Bn_prime}
Suppose (N1)-(N2) in Theorem \ref{thm:nonconvex_generic_conv_long} hold. Fix step sizes $\{\eta_r\}$ such that (\ref{cond:nonconvex_generic_conv_long_eta}) holds with  $\eta_\ast \in (0,1)$. For $c_1>0$, let $t_n'\equiv t_n+c_1\log \log n$. Then for $n\geq 3$, $\#(t_n')\leq (\log n)^{c_1}\cdot \big(\#(t_n)+c_1\log\log n\big)$ for $\# \in \{\mathcal{B},\mathcal{B}_0\}$. 
\end{lemma}
\begin{proof}
We only prove the case for $\#=\mathcal{B}$. For notational simplicity, let
\begin{align*}
S_r&\equiv 1+\eta_r\cdot \Big[\lambda_{\min}\Big(\mathbb{M}_{ \mathrm{u}_{\ast}^{(r)},\mathrm{u}_{\ast}^{(r)} }^{\mathsf{Z}}\Big)\Big]_- \Big)+\epsilon_n\leq e.
\end{align*}
Note that
\begin{align*}
\mathcal{B}(t_n')&\leq \sum_{s \in [0:t_n')} \bigg\{\prod_{r \in [s:t_n)}\cdot \prod_{r \in [t_n:t_n')}\bigg\} S_r\leq \sum_{s \in [0:t_n')} \bigg(\prod_{r \in [s:t_n)} S_r \cdot e^{t_n'-t_n}\bigg)\\
&\leq e^{t_n'-t_n}\cdot \bigg(\sum_{s \in [0:t_n)} \prod_{r \in [s:t_n)} S_r +(t_n'-t_n)\bigg)\leq (\log n)^{c_1}\cdot \big(\mathcal{B}(t_n)+c_1\log\log n\big),
\end{align*}
as desired.
\end{proof}

\subsection{Concentration and universality for weighted sample covariance}

The following generic concentration and universality estimate for a weighted sample covariance will prove essential in the proof of Proposition \ref{prop:nonconvex_generic_conv_long_Mcal}.

\begin{proposition}\label{prop:weighted_sample_cov}
	Suppose $\phi\geq 1$ and the following hold for some $K,\Lambda\geq 2$ and $\mathfrak{p}\geq 1$:
	\begin{enumerate}
		\item $\{X_{ij}\}_{i \in [m], j \in [n]}$ are independent random variables with $\E X_{ij}=0, \var(X_{ij})=1$ and  $ \max_{i \in [m], j \in [n]}\pnorm{X_{ij}}{\psi_2}\leq K$. 
		\item For $q \in [0:4]$, the mapping $x\mapsto \mathsf{H}^{(q)}(x)$ is $\Lambda$-pseudo-Lipschitz of order $\mathfrak{p}$ and $\abs{\mathsf{H}^{(q)}(0)}\leq \Lambda$.  
	\end{enumerate}
	Then there exists some constant $c_{\mathfrak{p}}>0$ such that 
	\begin{align*}
	\biggpnorm{\frac{1}{m}\sum_{k=1}^m  \E\mathsf{H}(\iprod{X_k}{\mu_\ast}) X_kX_k^\top - \mathsf{M}_{\mathsf{H},\mu_\ast}}{\op}\leq (K\Lambda L_\ast)^{c_{\mathfrak{p}}}\cdot n^{-1/2},
	\end{align*}
	and with probability at least $1-m^{-100}$,
	\begin{align*}
	\biggpnorm{\frac{1}{m}\sum_{k=1}^m  \mathsf{H}(\iprod{X_k}{\mu_\ast}) X_kX_k^\top - \mathsf{M}_{\mathsf{H},\mu_\ast}}{\op}\leq \big(K\Lambda L_\ast \log m\big)^{c_{\mathfrak{p}}}\cdot (n\wedge \phi)^{-1/2}.
	\end{align*}
	where $L_\ast\equiv n^{1/2}\pnorm{\mu_\ast}{\infty}$ and
	\begin{align*}
	\mathsf{M}_{\mathsf{H},\mu_\ast} &= \E \mathsf{H}(\iprod{\mathsf{Z}_n}{\mu_\ast})\mathsf{Z}_n\mathsf{Z}_n^\top\\
	&=\E \mathsf{H}(\iprod{\mathsf{Z}_n}{\mu_\ast})\cdot I_n+\E \mathsf{H}''(\iprod{\mathsf{Z}_n}{\mu_\ast})\cdot \mu_\ast \mu_\ast^\top \in \R^{n\times n}.
	\end{align*}
	
\end{proposition}
\begin{proof}
	\noindent (\textbf{Step 1}). In this step, we prove that there exists some constant $c_{\mathfrak{p}}>0$ such that with probability at least $1-m^{-100}$,
	\begin{align}\label{ineq:weighted_sample_cov_step1}
	\biggpnorm{\frac{1}{m}\sum_{k=1}^m (\mathrm{id}-\E) \mathsf{H}(\iprod{X_k}{\mu_\ast}) X_kX_k^\top }{\op}\leq (K\Lambda L_\ast \log m)^{c_{\mathfrak{p}}}\cdot \phi^{-1}.
	\end{align}
	For $\delta>0$, let $\mathcal{U}_\delta$ be a minimal $\delta$-cover of the unit ball $B_n(1)$ under the Euclidean metric $\pnorm{\cdot}{}$. Now using a standard discretization argument, 
	\begin{align*}
	&\biggpnorm{\frac{1}{m}\sum_{k=1}^m (\mathrm{id}-\E) \mathsf{H}(\iprod{X_k}{\mu_\ast}) X_kX_k^\top }{\op}\leq \max_{u \in \mathcal{U}_\delta}\biggabs{ \frac{1}{m}\sum_{k=1}^m (\mathrm{id}-\E) \mathsf{H}(\iprod{X_k}{\mu_\ast}) \iprod{X_k}{u}^2 }\\
	&\qquad\qquad +c\Lambda\cdot (\mathrm{id}+\E) \cdot \E_{\pi_m}\pnorm{X_{\pi_m}}{}^3\cdot (1+\pnorm{\mu_\ast}{})\delta \equiv I_1+I_2.
	\end{align*}
	For $M>0$, let $\mathsf{H}_M(x)\equiv (\mathsf{H}(x)\wedge M)\vee (-M)$. 
	By choosing $M\equiv (K\Lambda L_\ast\log m)^{c_{\mathfrak{p}}}$ for a large enough $c_{\mathfrak{p}}>0$, on an event $E_{0;M}$ with probability at least $1-m^{-200}$, we have $\max_{k \in [m]} \abs{\mathsf{H}(\iprod{X_k}{\mu_\ast})}\leq M$. This means on the event  $E_{0;M}$, 
	\begin{align*}
	I_1&= \max_{u \in \mathcal{U}_\delta}\biggabs{ \frac{1}{m}\sum_{k=1}^m (\mathrm{id}-\E) \mathsf{H}_M(\iprod{X_k}{\mu_\ast}) \iprod{X_k}{u}^2 },
	\end{align*}
	so by using that $\sup_{u \in B_n(1)}\pnorm{\mathsf{H}_M(\iprod{X_1}{\mu_\ast}) \iprod{X_1}{u}^2}{\psi_1}\lesssim MK^2$, and the fact that $\abs{\mathcal{U}_\delta}\lesssim n\log(1/\delta)$, by choosing $\delta=m^{-100}$, an application of Bernstein's inequality yields that $I_1$ satisfies the concentration estimate (\ref{ineq:weighted_sample_cov_step1}). The same estimate trivially holds for $I_2$.

	\noindent (\textbf{Step 2}). In this step, we prove that 
	\begin{align}\label{ineq:weighted_sample_cov_step2}
	&\biggpnorm{\frac{1}{m}\sum_{k=1}^m \Big(\E \mathsf{H}(\iprod{X_k}{\mu_\ast}) X_kX_k^\top-\E \mathsf{H}(\iprod{\mathsf{Z}_{n;k}}{\mu_\ast}) \mathsf{Z}_{n;k}\mathsf{Z}_{n;k}^\top\Big)}{\op}\nonumber\\
	&\leq (K\Lambda L_\ast )^{c_{\mathfrak{p}}}\cdot n^{-1/2}.
	\end{align}
	Clearly, it suffices to prove that with $\mathsf{X}_n\sim X_1$,
	\begin{align}\label{ineq:weighted_sample_cov_step2_alt}
	\bigpnorm{\E \mathsf{H}(\iprod{\mathsf{X}_n}{\mu_\ast}) \mathsf{X}_n\mathsf{X}_n^\top-\E \mathsf{H}(\iprod{\mathsf{Z}_n}{\mu_\ast}) \mathsf{Z}_n\mathsf{Z}_n^\top}{\op}\leq (K\Lambda L_\ast)^{c_{\mathfrak{p}}}\cdot n^{-1/2}.
	\end{align}
	Fix $i_1,i_2 \in [n]$. Then using Lemma \ref{lem:cum_exp} with $f(x_{i_1})\equiv x_{i_2} \mathsf{H}(\iprod{x}{\mu_\ast})$, as 
	\begin{align*}
	f'(x_{i_1})& = \delta_{i_1,i_2}\mathsf{H}(\iprod{x}{\mu_\ast})+ \mu_{i_1} \cdot x_{i_2} \mathsf{H}'(\iprod{x}{\mu_\ast}),\\
	f''(x_{i_1})& = 2\delta_{i_1,i_2}\cdot \mu_{i_1}\mathsf{H}(\iprod{x}{\mu_\ast})+ \mu_{i_1}^2 \cdot x_{i_2} \mathsf{H}''(\iprod{x}{\mu_\ast}),
	\end{align*}
	we have 
	\begin{align}\label{ineq:weighted_sample_cov_step2_1}
	&\bigabs{\E \mathsf{H}(\iprod{\mathsf{X}_n}{\mu_\ast}) \mathsf{X}_{n,i_1} \mathsf{X}_{n,i_2}-\E \mathsf{H}(\iprod{\mathsf{Z}_n}{\mu_\ast}) \mathsf{Z}_{n,i_1}\mathsf{Z}_{n,i_2}}\nonumber\\
	&\leq \delta_{i_1,i_2}\cdot \bigabs{\E \mathsf{H}(\iprod{\mathsf{X}_n}{\mu_\ast}) - \E \mathsf{H}(\iprod{\mathsf{Z}_n}{\mu_\ast}) }+2\Lambda\cdot \delta_{i_1,i_2} \pnorm{\mu_\ast}{\infty}\nonumber\\
	&\qquad + \pnorm{\mu_\ast}{\infty}\cdot \bigabs{\E \mathsf{X}_{n,i_2}\mathsf{H}'(\iprod{\mathsf{X}_n}{\mu_\ast}) -\E  \mathsf{Z}_{n,i_2} \mathsf{H}'(\iprod{\mathsf{Z}_n}{\mu_\ast}) }\nonumber\\
	&\qquad +c\pnorm{\mu_\ast}{\infty}^2\cdot \sup_{s \in [0,1]^n}\max_{q=1,3}\max_{W_n \in \{\mathsf{X}_n,\mathsf{Z}_n\}}\bigabs{\E W_{n,i_1}^qW_{n,i_2}\mathsf{H}''(\iprod{W_n}{s\circ\mu_\ast}) }\nonumber\\
	&\equiv \delta_{i_1,i_2}\cdot J_1+2\Lambda\cdot \delta_{i_1,i_2} \pnorm{\mu_\ast}{\infty}+J_2+J_3.
	\end{align}
	For $J_1$, using the Lindeberg principle in Lemma \ref{lem:lindeberg}, we have
	\begin{align*}
	J_1&\leq (K\Lambda L_\ast)^{c_{\mathfrak{p}}}\cdot n^{-1/2}.
	\end{align*} 
	For $i_1=i_2$, we may therefore use trivial bounds 
	\begin{align*}
	J_2\leq (K\Lambda L_\ast)^{c_{\mathfrak{p}}}\cdot n^{-1/2},\quad J_3\leq (K\Lambda L_\ast)^{c_{\mathfrak{p}}}\cdot n^{-1}.
	\end{align*}
	So for $i_1=i_2$, in view of (\ref{ineq:weighted_sample_cov_step2_1}),
	\begin{align}\label{ineq:weighted_sample_cov_step2_2}
	\bigabs{\E \mathsf{H}(\iprod{\mathsf{X}_n}{\mu_\ast}) \mathsf{X}_{n,i_1} \mathsf{X}_{n,i_2}-\E \mathsf{H}(\iprod{\mathsf{Z}_n}{\mu_\ast}) \mathsf{Z}_{n,i_1}\mathsf{Z}_{n,i_2}}\leq (K\Lambda L_\ast)^{c_{\mathfrak{p}}}\cdot n^{-1/2}.
	\end{align}
	For $i_1\neq i_2$, we need refined the control for $J_2,J_3$:
	\begin{itemize}
		\item For $J_2$, we apply Lemma \ref{lem:cum_exp} (with the obvious choice of the function therein) with trivial residual estimates to obtain that
		\begin{align*}
		J_2&\leq \pnorm{\mu_\ast}{\infty}^2\cdot \bigabs{\E \mathsf{H}''(\iprod{\mathsf{X}_n}{\mu_\ast}) -\E   \mathsf{H}''(\iprod{\mathsf{Z}_n}{\mu_\ast}) }+(K\Lambda L_\ast)^{c_{\mathfrak{p}}}\cdot  \pnorm{\mu_\ast}{\infty}^3\\
		&\leq (K\Lambda L_\ast)^{c_{\mathfrak{p}}}\cdot n^{-3/2}.
		\end{align*}
		Here in the last line we used again the Lindeberg principle in Lemma \ref{lem:lindeberg}.
		
		\item For $J_3$, we apply Lemma \ref{lem:cum_exp} with $g(x_{i_2})\equiv x_{i_1}^q \mathsf{H}''(\iprod{x}{s\circ \mu_\ast})$. Then as $i_1\neq i_2$, 
		$g'(x_{i_2}) = \mu_{i_2}\cdot x_{i_1}^q \mathsf{H}^{(3)}(\iprod{x}{s\circ \mu_\ast})$ and $g''(x_{i_2}) = \mu_{i_2}^2\cdot x_{i_1}^q \mathsf{H}^{(4)}(\iprod{x}{s\circ \mu_\ast})$. Consequently, using trivial residual estimates,
		\begin{align*}
		J_3&\leq (K\Lambda L_\ast)^{c_{\mathfrak{p}}}\cdot  \pnorm{\mu_\ast}{\infty}^2\cdot \big(\pnorm{\mu_\ast}{\infty}+\pnorm{\mu_\ast}{\infty}^2\big)\leq (K\Lambda L_\ast)^{c_{\mathfrak{p}}}\cdot n^{-3/2}. 
		\end{align*}
	\end{itemize}
	Combining the above estimates, for $i_1\neq i_2$, in view of (\ref{ineq:weighted_sample_cov_step2_1}) we have 
	\begin{align}\label{ineq:weighted_sample_cov_step2_3}
	\bigabs{\E \mathsf{H}(\iprod{\mathsf{X}_n}{\mu_\ast}) \mathsf{X}_{n,i_1} \mathsf{X}_{n,i_2}-\E \mathsf{H}(\iprod{\mathsf{Z}_n}{\mu_\ast}) \mathsf{Z}_{n,i_1}\mathsf{Z}_{n,i_2}}\leq(K\Lambda L_\ast)^{c_{\mathfrak{p}}}\cdot n^{-3/2}.
	\end{align}
	The claimed estimate in (\ref{ineq:weighted_sample_cov_step2_alt}) follows by using (\ref{ineq:weighted_sample_cov_step2_2}) and (\ref{ineq:weighted_sample_cov_step2_3}) upon noting the generic estimate $\pnorm{M}{\op}\leq \max_{i \in [n]} \abs{M_{ii}}+n\cdot \max_{i\neq j} \abs{M_{ij}}$ for $M \in \R^{n\times n}$.
	
	\noindent (\textbf{Step 3}). Finally, by applying Gaussian integration-by-parts twice, 
	\begin{align*}
	\E \mathsf{H}(\iprod{\mathsf{Z}_n}{\mu_\ast}) \mathsf{Z}_n\mathsf{Z}_n^\top = \E \mathsf{H}(\iprod{\mathsf{Z}_n}{\mu_\ast})\cdot I_n+\E \mathsf{H}''(\iprod{\mathsf{Z}_n}{\mu_\ast})\cdot \mu_\ast \mu_\ast^\top = \mathsf{M}_{\mathsf{H},\mu_\ast}.
	\end{align*}
	The claim now follows by combining the above identity with (\ref{ineq:weighted_sample_cov_step1}) and (\ref{ineq:weighted_sample_cov_step2}).
\end{proof}

\subsection{Proof of Proposition \ref{prop:nonconvex_generic_conv_long_Mscr}}\label{subsection:proof_nonconvex_generic_conv_long_Mscr_sketch}

	In the proof we shall assume $n\geq 3$ and $t\geq t_n$ without loss of generality. 
	
	We prove the claimed estimate by induction. For $t=0$, note that
	\begin{align*}
	\mathscr{M}_{ \mathrm{u}_X^{(\cdot)},\mathrm{u}_{\ast}^{(\cdot)} }^{(0),X}\leq 1+\bigpnorm{I_n - \eta_0\cdot \mathbb{M}_{ \mu^{(0)},\mu^{(0)} }^{X} }{\op}\leq 2+\eta_0\cdot \bigpnorm{ \mathbb{M}_{ \mu^{(0)},\mu^{(0)} }^{X} }{\op}\leq 3
	\end{align*}
	holds trivially. Fix $c_\ast>0$ to be determined later on, consider the following induction hypothesis at iteration $t-1$:
	\begin{align}\label{ineq:nonconvex_generic_conv_long_ind}
	\mathscr{M}_{ \mathrm{u}_{X}^{(\cdot)},\mathrm{u}_\ast^{(\cdot)} }^{(t-1),X}\leq n^{1/2}/(\log n)^{c_\ast}.
	\end{align}
	Let us now consider iteration $t$. Note that by Proposition \ref{prop:weighted_sample_cov}, for some $c_1=c_1(\mathfrak{p},c_0)>0$, with $\epsilon_n=(\log n)^{-100}$, for $n\geq n_0(\mathfrak{p},c_0)$,
	\begin{align*}
	&\bigpnorm{\mathbb{M}_{ \mathrm{u}_{X}^{(t)},\mathrm{u}_{\ast}^{(t)} }^{X} - \mathbb{M}_{ \mathrm{u}_\ast^{(t)},\mathrm{u}_{\ast}^{(t)} }^{\mathsf{Z}}  }{\op} \nonumber\\
	& \leq \bigpnorm{\mathbb{M}_{ \mathrm{u}_{X}^{(t)},\mathrm{u}_{\ast}^{(t)} }^{X} - \mathbb{M}_{ \mathrm{u}_\ast^{(t)},\mathrm{u}_{\ast}^{(t)} }^{X}  }{\op}+ \bigpnorm{ \mathbb{M}_{ \mathrm{u}_\ast^{(t)},\mathrm{u}_{\ast}^{(t)} }^{X} - \mathbb{M}_{ \mathrm{u}_\ast^{(t)},\mathrm{u}_{\ast}^{(t)} }^{\mathsf{Z}}  }{\op}\nonumber\\
	&\leq c_{\mathfrak{p}}\Lambda\cdot \bigpnorm{ \E^{(0)}\Big\{ \max_{i \in [m]}\Big(1+\abs{ \iprod{X_{i}}{ \mathrm{u}_{X}^{(t)}}  }+\abs{ \iprod{X_{i}}{ \mathrm{u}_{\ast}^{(t)}}  } \Big)^{c_{\mathfrak{p}}}\nonumber\\
	&\qquad\qquad \times \max_{i \in [m]}\abs{ \iprod{X_{i}}{ \mathrm{u}_{X}^{(t)}- \mathrm{u}_{\ast}^{(t)} }  }\cdot X_{\pi_m}X_{\pi_m}^\top \big\} }{\op}+\epsilon_n/4\nonumber\\
	&\leq \big(K\Lambda(1+\pnorm{\mathrm{u}_{X}^{(t)} }{}+\pnorm{ \mathrm{u}_\ast^{(t)} }{}) \log m\big)^{ c_{\mathfrak{p}}} \cdot \big(\pnorm{\mathrm{u}_{X}^{(t)}- \mathrm{u}_\ast^{(t)} }{}+ m^{-101}\big)+\epsilon_n/4\\
	&\leq (\log n)^{c_1}\cdot \big(1+\pnorm{\mathrm{u}_{X}^{(t)} }{}\big)^{c_{\mathfrak{p}} }\cdot \big(\pnorm{\mathrm{u}_{X}^{(t)}- \mathrm{u}_\ast^{(t)} }{}+ m^{-101}\big)+\epsilon_n/4.
	\end{align*}
	Using Proposition \ref{prop:diff_theo_grad_des}, for $n\geq n_0(\mathfrak{p},c_0)$,
	\begin{align*}
	\bigpnorm{\mathbb{M}_{ \mathrm{u}_{X}^{(t)},\mathrm{u}_{\ast}^{(t)} }^{X} - \mathbb{M}_{ \mathrm{u}_\ast^{(t)},\mathrm{u}_{\ast}^{(t)} }^{\mathsf{Z}}  }{\op}&\leq (\log n)^{c_1'}\cdot \Big(1+ n^{-1/2} \mathscr{M}_{ \mathrm{u}_{X}^{(\cdot)},\mathrm{u}_\ast^{(\cdot)} }^{(t-1),X} \Big)^{c_1'}\\
	&\qquad \times \Big( n^{-1/2}\mathscr{M}_{ \mathrm{u}_{X}^{(\cdot)},\mathrm{u}_\ast^{(\cdot)} }^{(t-1),X} + n^{-100}\Big) + \epsilon_n/4.
	\end{align*}
	By the induction hypothesis (\ref{ineq:nonconvex_generic_conv_long_ind}), by choosing $c_\ast\equiv c_1'+1$, for $n\geq n_0(\mathfrak{p},c_0)$,
	\begin{align*}
	\bigpnorm{\mathbb{M}_{ \mathrm{u}_{X}^{(t)},\mathrm{u}_{\ast}^{(t)} }^{X} - \mathbb{M}_{ \mathrm{u}_\ast^{(t)},\mathrm{u}_{\ast}^{(t)} }^{\mathsf{Z}}  }{\op}\leq \epsilon_n.
	\end{align*}	
	Consequently, in view of Lemma \ref{lem:nonconvex_generic_conv_long_min_eig}, with $t_n'\equiv t_n+c_1\log \log n$, for $n\geq n_0(\mathfrak{p},c_0,\eta_\ast,\sigma_\ast)$ and all $t$,
	\begin{align*}
	&\mathscr{M}_{ \mathrm{u}_X^{(\cdot)},\mathrm{u}_{\ast}^{(\cdot)} }^{(t),X}= 1+ \max_{\tau \in [0:t]}\sum_{s \in [0:\tau]} \bigg(\bigg\{\prod_{r \in [s:t_n')}\cdot\prod_{r \in [t_n':\tau]}\bigg\}\bigpnorm{I_n - \eta_r\cdot \mathbb{M}_{ \mathrm{u}_{X}^{(r)},\mathrm{u}_{\ast}^{(r)} }^{X} }{\op }\bigg)\nonumber\\
	&\leq 1+ \max_{\tau \in [0:t]} \bigg(\sum_{s \in [0:\tau]} \prod_{r \in [s:t_n')}  \Big(\bigpnorm{I_n - \eta_t\cdot \mathbb{M}_{ \mathrm{u}_{\ast}^{(t)},\mathrm{u}_{\ast}^{(t)} }^{\mathsf{Z}} }{\op }+\epsilon_n\Big)   \bigg)\cdot \bigg(1-\frac{\eta_\ast \sigma_\ast}{2}\bigg)^{(\tau-t_n'+1)_+}\nonumber\\
	&\leq 1+ \max_{\tau \in [0:t]} \bigg(\sum_{s \in [0:t_n')} \prod_{r \in [s:t_n')}  \Big(\bigpnorm{I_n - \eta_t\cdot \mathbb{M}_{ \mathrm{u}_{\ast}^{(t)},\mathrm{u}_{\ast}^{(t)} }^{\mathsf{Z}} }{\op}+\epsilon_n\Big)\nonumber\\
	&\qquad \qquad\qquad\qquad +(\tau-t_n'+1)_+  \bigg)\cdot \bigg(1-\frac{\eta_\ast \sigma_\ast}{2}\bigg)^{(\tau-t_n'+1)_+}\nonumber\\
	&\leq \mathcal{B}(t_n')+c_2\stackrel{(\ast)}{\leq} (\log n)^{c_2}\cdot \big(\mathcal{B}(t_n)+\log \log n\big)\\
	&\leq n^{1/2}/(\log n)^{c_1'+1}=n^{1/2}/(\log n)^{c_\ast}.
	\end{align*}
	Here $(\ast)$ follows from Lemma \ref{lem:Bn_prime}. This proves the induction (\ref{ineq:nonconvex_generic_conv_long_ind}) for iteration $t$. The claim $\pnorm{\mathrm{u}_{X}^{(\cdot)}- \mathrm{u}_\ast^{(\cdot)} }{}$ now follows by an application of Proposition \ref{prop:diff_theo_grad_des}.\qed

\subsection{Proof of Proposition \ref{prop:nonconvex_generic_conv_long_Mcal}}\label{subsection:proof_nonconvex_generic_conv_long_Mcal}

	Again in the proof we shall assume $n\geq 3$ and $t\geq t_n$ without loss of generality. 	
	
	By Proposition \ref{prop:nonconvex_generic_conv_long_Mscr} and Theorem \ref{thm:grad_des_se}, for some constants $c_1=c_1(\mathfrak{p},c_0)>0$ and  $c_2=c_2(\mathfrak{p},c_0,\eta_\ast,\sigma_\ast)>0$  whose numeric values may change from line to line, on an event $E_{0,1}$ with probability at least $1-m^{-200}$, uniformly in $t\leq m^{100}$, 
	\begin{align}\label{ineq:nonconvex_generic_conv_long_Mcal_step1_0}
	&\pnorm{\mu^{(t)}-\mathrm{u}_\ast^{(t)}}{}\leq (\log n)^{\bar{c}_2} \cdot \Big\{\phi^{-1/2}\cdot \mathfrak{M}_{ \mu^{(\cdot)},\mathrm{u}_X^{(\cdot)} }^{(t-1)}(X)+n^{-1/2}\cdot \mathcal{B}(t_n)\Big\},\nonumber\\
	& \mathscr{M}_{ \mathrm{u}_X^{(\cdot)},\mathrm{u}_{\ast}^{(\cdot)} }^{(t),X}\vee n^{1/2}\pnorm{\mathrm{u}_{X}^{(t)}-\mathrm{u}_{\ast}^{(t)}}{}\leq (\log n)^{\bar{c}_2}\cdot \mathcal{B}(t_n).
	\end{align}
	Fix $c_\ast>\bar{c}_2+100$ to be determined later on, and consider the following induction hypothesis at iteration $t-1$ that holds on $E_\ast^{(t-1)}$:
	\begin{align}\label{ineq:nonconvex_generic_conv_long_Mcal_step1_ind_hyp}
	\mathfrak{M}_{ \mu^{(\cdot)},\mathrm{u}_X^{(\cdot)} }^{(t-1)}(X)\vee  \max_{i \in [m]} \mathfrak{M}_{\mu^{(\cdot)}, \mu_{[-i]}^{(\cdot)} }^{(t-1)}(X)\leq \phi^{1/2}/(\log n)^{c_\ast}.
	\end{align}	
	We note that under the induction hypothesis (\ref{ineq:nonconvex_generic_conv_long_Mcal_step1_ind_hyp}), by Lemma \ref{lem:grad_des_loo_diff}, on an event $E_{0,2}$ with probability at least $1-m^{-200}$, uniformly in $i \in [m]$ and $t\leq m^{100}$, 
	\begin{align*}
	\pnorm{\mu_{[-i]}^{(t)}}{}\leq \pnorm{\mu^{(t)}}{}+ m^{-1/2}(\log n)^{c_1}\cdot \Big(1+\max_{s \in [1:t-1]} \pnorm{\mu_{[-i]}^{(s)}}{}\Big).
	\end{align*}
	In view of (\ref{ineq:nonconvex_generic_conv_long_Mcal_step1_0}) and the trivial estimates for $\pnorm{\mathrm{u}_{\ast}^{(t)} }{}$ by Proposition \ref{prop:nonconvex_generic_conv_long_contraction}-(1), on $E_{0,1}\cap E_{0,2}$, uniformly in $i \in [m]$ and $t\leq m^{100}$, 
	\begin{align*}
	\pnorm{\mu_{[-i]}^{(t)}}{}\leq (\log n)^{c_1}\cdot \Big(1+m^{-1/2}\max_{s \in [1:t-1]} \pnorm{\mu_{[-i]}^{(s)}}{}\Big).
	\end{align*}
	This means that under the induction hypothesis (\ref{ineq:nonconvex_generic_conv_long_Mcal_step1_ind_hyp}), on the event $E_0$, uniformly in $t\leq m^{100}$, for $n\geq n_0(\mathfrak{p},c_0)$,
	\begin{align}\label{ineq:nonconvex_generic_conv_long_Mcal_step1_00}
	\max_{i \in [m]}\pnorm{\mu_{[-i]}^{(t)}}{}\leq (\log n)^{\bar{c}_1}.
	\end{align}
	Below we shall always work on the event $E^{(t-1)}\equiv E_\ast^{(t-1)}\cap E_{0,1}\cap E_{0,2}$ and use an coupled, inductive argument to produce a bound for the quantities on the left hand side of (\ref{ineq:nonconvex_generic_conv_long_Mcal_step1_ind_hyp}) at iteration $t$.
	
	\noindent (\textbf{Step 1}). In this step, we provide inductive control of $\mathfrak{M}_{ \mu^{(\cdot)},\mathrm{u}_X^{(\cdot)} }^{(t)}(X)$. As
	\begin{align*}
	M_{ \mu^{(t)},\mathrm{u}_X^{(t)} }(X)& = \E_{U,\pi_m}  \partial_1\mathfrak{S}\big(\bigiprod{X_{\pi_m}}{U \mu^{(t)}+(1-U) \mathrm{u}_X^{(t)}}, \iprod{X_{\pi_m}}{\mu_\ast},\xi_{\pi_m}\big) X_{\pi_m}X_{\pi_m}^\top,
	\end{align*}
	we have
	\begin{align*}
	&\bigpnorm{M_{\mu^{(t)},\mathrm{u}_X^{(t)} }(X)- \mathbb{M}_{ \mathrm{u}_{\ast}^{(t)},\mathrm{u}_{\ast}^{(t)} }^{X}  }{\op}\\
	&\leq \Lambda^{c_{\mathfrak{p}}} \cdot \sup_{ \pnorm{u}{}=1} \E_{\pi_m} \Big\{\Big(1+\abs{\iprod{X_{\pi_m}}{\mu^{(t)} } }+\abs{\iprod{X_{\pi_m}}{\mathrm{u}_\ast^{(t)}}}+\abs{ \iprod{X_{\pi_m}}{ \mathrm{u}_X^{(t)}  } }+\abs{ \iprod{X_{\pi_m}}{ \mu_\ast  } } \Big)^{c_{\mathfrak{p}}}\\
	&\qquad\qquad\qquad\qquad \times  \big(\abs{\iprod{X_{\pi_m}}{\mu^{(t)}-\mathrm{u}_\ast^{(t)}}}+\abs{ \iprod{X_{\pi_m}}{ \mathrm{u}_X^{(t)}-\mathrm{u}_\ast^{(t)}  } }\big) \iprod{X_{\pi_m}}{u}^2\Big\}\\
	&\leq \Lambda^{c_{\mathfrak{p}}} \cdot \pnorm{\E_{\pi_m} X_{\pi_m} X_{\pi_m}^\top}{\op}\cdot \max_{i \in [m]} \Big(1+\abs{\iprod{X_{i}}{\mu^{(t)} } }+\max_{\mathrm{u}\in \{\mathrm{u}_\ast^{(t)},\mathrm{u}_X^{(t)},\mu_\ast\} }\abs{\iprod{X_{i}}{\mathrm{u} }}\Big)^{c_{\mathfrak{p}}}\\
	&\qquad\qquad \times  \max_{i \in [m]} \Big(\abs{\iprod{X_{i}}{\mu^{(t)}-\mu_{[-i]}^{(t)}}}+\abs{\iprod{X_{i}}{\mu_{[-i]}^{(t)}-\mathrm{u}_X^{(t)}  }} + \abs{ \iprod{X_{i}}{ \mathrm{u}_\ast^{(t)}-\mathrm{u}_X^{(t)}  } }\Big).
	\end{align*}
	Now under the induction hypothesis (\ref{ineq:nonconvex_generic_conv_long_Mcal_step1_ind_hyp}) and using:
	\begin{itemize}
		\item a standard sub-gaussian estimate for $\pnorm{\E_{\pi_m} X_{\pi_m} X_{\pi_m}^\top}{\op}$ (cf. \cite[Theorem 4.7.1 and Exercise 4.7.3]{vershynin2018high});
		\item Proposition \ref{prop:weighted_sample_cov} for the left hand side of the above display;
		\item Theorem \ref{thm:grad_des_incoh} coupled with (\ref{ineq:nonconvex_generic_conv_long_Mcal_step1_00}), the estimate (\ref{ineq:nonconvex_generic_conv_long_Mcal_step1_0}) and  trivial estimates for $\pnorm{\mathrm{u}_{\ast}^{(t)} }{}$ by Proposition \ref{prop:nonconvex_generic_conv_long_contraction}-(1) to deduce that on an event with $\Prob^{(0)}$-probability at least $1-m^{-200}$, uniformly in $i \in [m]$,
		\begin{align*}
		\abs{\iprod{X_{i}}{\mu^{(t)} } }+\max_{\mathrm{u}\in \{\mathrm{u}_\ast^{(t)},\mathrm{u}_X^{(t)},\mu_\ast\} }\abs{\iprod{X_{i}}{\mathrm{u} }}\leq (\log n)^{c_2},
		\end{align*}
	\end{itemize}
	we have on an event with $\Prob^{(0)}$-probability at least $1-cm^{-200}$ (we slightly change the absolute constant), for $n\geq n_0(\mathfrak{p},c_0)$, 
	\begin{align}\label{ineq:nonconvex_generic_conv_long_Mcal_step1_1}
	&\bigpnorm{M_{\mu^{(t)},\mathrm{u}_X^{(t)} }(X)- \mathbb{M}_{ \mathrm{u}_{\ast}^{(t)},\mathrm{u}_{\ast}^{(t)} }^{\mathsf{Z}} }{\op}\leq \frac{(\log n)^{c_1}}{ (n\wedge \phi)^{1/2}}+(\log n)^{c_{2;0}}\nonumber\\
	&\quad  \times \max_{i \in [m]} \Big(\abs{\iprod{X_{i}}{\mu^{(t)}-\mu_{[-i]}^{(t)}}}+\abs{\iprod{X_{i}}{\mu_{[-i]}^{(t)}-\mathrm{u}_X^{(t)}  }} + \abs{ \iprod{X_{i}}{ \mathrm{u}_\ast^{(t)}-\mathrm{u}_X^{(t)}  } }\Big).
	\end{align}	
	For the first term in the parenthesis of (\ref{ineq:nonconvex_generic_conv_long_Mcal_step1_1}), in view of (\ref{ineq:grad_des_loo_diff_1}), recall the definition of $\Xi_{[-i]}^{(\cdot)}$ in (\ref{ineq:grad_des_loo_diff_0}),
	\begin{align*}
	\abs{\iprod{X_{i}}{\mu^{(t)}-\mu_{[-i]}^{(t)}}}& = \frac{1}{m}\cdot  \biggabs{\sum_{s \in [1:t]} \eta_{s-1} \biggiprod{X_i}{\bigg(\prod_{r: t-1\to s} \big(I_n-\eta_r\cdot M_{\mu^{(r)},\mu_{[-i]}^{(r)} }(X)\big)\bigg)\,\Xi_{[-i]}^{(s-1)} } }\\
	&\leq \frac{1}{m}\cdot \pnorm{\eta_{[0:t)}}{\infty} \cdot \pnorm{X_i}{}\cdot \mathfrak{M}_{\mu^{(\cdot)}, \mu_{[-i]}^{(\cdot)} }^{(t-1)}(X)\cdot \max_{s \in [1:t]} \pnorm{\Xi_{[-i]}^{(s-1)}}{}\\
	&\leq \frac{\Lambda^{c_{\mathfrak{p}}}}{m}\cdot \pnorm{X_i}{}^2\cdot \mathfrak{M}_{\mu^{(\cdot)}, \mu_{[-i]}^{(\cdot)} }^{(t-1)}(X)\cdot\big(1+\abs{\iprod{X_i}{ \mu_{[-i]}^{(t-1)}}}+ \abs{\iprod{X_i}{\mu_\ast}}\big)^{c_{\mathfrak{p}}}.
	\end{align*}
	Here the last inequality follows from the estimate (\ref{ineq:grad_des_loo_diff_1_0}). Now by standard concentration and the estimate in (\ref{ineq:nonconvex_generic_conv_long_Mcal_step1_00}), on an event with $\Prob^{(0)}$-probability at least $1-cm^{-200}$, uniformly in $i \in [m]$ and $ t\leq m^{100}$, 
	\begin{align}\label{ineq:nonconvex_generic_conv_long_Mcal_step1_est1}
	\abs{\bigiprod{X_{i}}{\mu^{(t)}-\mu_{[-i]}^{(t)}}}
	&\leq \frac{(\log n)^{c_2}}{\phi^{1/2}}.
	\end{align}
	For the second term in the parenthesis of (\ref{ineq:nonconvex_generic_conv_long_Mcal_step1_1}), in view of the estimates in (\ref{ineq:nonconvex_generic_conv_long_Mcal_step1_0}) and (\ref{ineq:nonconvex_generic_conv_long_Mcal_step1_00}), along with Proposition \ref{prop:diff_theo_grad_des} and Lemma \ref{lem:grad_des_loo_diff}, on an event with $\Prob^{(0)}$-probability at least $1-cm^{-200}$, uniformly in $i \in [m]$ and $ t\leq m^{100}$, 
	\begin{align*}
	\abs{\bigiprod{X_{i}}{\mu_{[-i]}^{(t)}-\mathrm{u}_X^{(t)} }} &\leq c K \sqrt{\log m}\cdot \big(\pnorm{\mu^{(t)}-\mathrm{u}_\ast^{(t)}}{}+\pnorm{\mathrm{u}_\ast^{(t)}-\mathrm{u}_X^{(t)} }{}+ \pnorm{\mu^{(t)}-\mu_{[-i]}^{(t)}}{}\big)\nonumber\\
	&\leq (\log n)^{c_{2;I}}\cdot \Big\{ \phi^{-1/2}\cdot  \mathfrak{M}_{ \mu^{(\cdot)},\mathrm{u}_X^{(\cdot)} }^{(t-1)}(X)+ n^{-1/2}\cdot \mathcal{B}(t_n)+m^{-1/2}\Big\}.
	\end{align*}
	By choosing $c_\ast\geq c_{2;I}+c_{2;0}+101$, on an event with $\Prob^{(0)}$-probability at least $1-cm^{-200}$, uniformly in $i \in [m]$ and $ t\leq m^{100}$, for $n\geq n_0(\mathfrak{p},c_0)$, 
	\begin{align}\label{ineq:nonconvex_generic_conv_long_Mcal_step1_est2}
	\abs{\bigiprod{X_{i}}{\mu_{[-i]}^{(t)}-\mathrm{u}_X^{(t)} }}\leq 3(\log n)^{-c_{2;0}-101}.
	\end{align}
	For the last term in the parenthesis of (\ref{ineq:nonconvex_generic_conv_long_Mcal_step1_1}), by (\ref{ineq:nonconvex_generic_conv_long_Mcal_step1_0}), on an event with $\Prob^{(0)}$-probability at least $1-cm^{-200}$, uniformly in $i \in [m]$ and $t\leq m^{100}$,
	\begin{align}\label{ineq:nonconvex_generic_conv_long_Mcal_step1_est3}
	\abs{ \iprod{X_{i}}{ \mathrm{u}_\ast^{(t)}-\mathrm{u}_X^{(t)}  } }
	&\leq (\log n)^{c_{2;I}'}\cdot n^{-1/2}\cdot \mathcal{B}(t_n)\leq (\log n)^{-c_{2;0}-101},
	\end{align}
	by choosing $c_\ast\geq c_{2;I}'+c_{2;0}+101$.
	
	Now with $c_\ast\geq c_{2;I}\vee c_{2;I}'+c_{2;0}+101$, by combining the estimates (\ref{ineq:nonconvex_generic_conv_long_Mcal_step1_est1})-(\ref{ineq:nonconvex_generic_conv_long_Mcal_step1_est3}) into (\ref{ineq:nonconvex_generic_conv_long_Mcal_step1_1}), if $\phi\geq (\log m)^{c_1}$, on an event $E_{1,1}$ with $\Prob^{(0)}$-probability at least $1-c m^{-200}$, uniformly for $t\leq m^{100}$, for $n\geq n_0(\mathfrak{p},c_0)$,
	\begin{align*}
	& \bigpnorm{M_{\mu^{(t)},\mathrm{u}_X^{(t)} }(X)- \mathbb{M}_{ \mathrm{u}_{\ast}^{(t)},\mathrm{u}_{\ast}^{(t)} }^{\mathsf{Z}} }{\op}\leq \epsilon_n.
	\end{align*}
	By Lemma \ref{lem:nonconvex_generic_conv_long_min_eig}, there exists some $c_1'=c_1'(\mathfrak{p},c_0,\eta_\ast,\sigma_\ast)>0$, such that with $t_n'\equiv t_n+c_1'\log \log n$, for all $t\geq t_n'$, $\lambda_{\min}\big(\mathbb{M}_{ \mathrm{u}_{\ast}^{(t)},\mathrm{u}_{\ast}^{(t)} }^{\mathsf{Z}}\big)\geq 3\sigma_\ast/4$. Consequently, uniformly in $t\leq m^{100}$, for $n\geq n_0(\mathfrak{p},c_0,\eta_\ast,\sigma_\ast)$, 
	\begin{align}\label{ineq:nonconvex_generic_conv_long_Mcal_step1_local_op}
	&\mathfrak{M}_{ \mu^{(\cdot)},\mathrm{u}_X^{(\cdot)} }^{(t)}(X)= 1+ \max_{\tau \in [0:t]}\sum_{s \in [0:\tau]} \bigg(\bigg\{\prod_{r \in [s:t_n')}\cdot\prod_{r \in [t_n':\tau]}\bigg\}\bigpnorm{I_n-\eta_r \cdot M_{\mu^{(r)},\mathrm{u}_X^{(r)} }(X)}{\op}\bigg)\nonumber\\
	&\leq 1+ \max_{\tau \in [0:t]}\sum_{s \in [0:\tau]} \bigg(\prod_{r \in [s:t_n')}\Big(\bigpnorm{I_n-\eta_r\cdot \mathbb{M}_{ \mathrm{u}_{\ast}^{(r)},\mathrm{u}_{\ast}^{(r)} }^{\mathsf{Z}}}{\op}+\epsilon_n\Big)\bigg)\cdot \bigg(1-\frac{\eta_\ast\sigma_\ast}{2}\bigg)^{(\tau-t_n'+1)_+}\nonumber\\
	&\leq \mathcal{B}(t_n')+ c_{2}\leq  (\log n)^{c_2}\cdot \big(\mathcal{B}(t_n)+\log \log n)\leq \phi^{1/2}/(\log n)^{c_\ast}.
	\end{align}
	\noindent (\textbf{Step 2}). In this step, we provide inductive control of $\big\{\mathfrak{M}_{\mu^{(\cdot)}, \mu_{[-i]}^{(\cdot)} }^{(t)}(X):i \in [m]\}$. Note that by definition,
	\begin{align*}
	M_{\mu^{(t)}, \mu_{[-i]}^{(t)} }(X)& = \E_{U,\pi_m} \partial_1\mathfrak{S}\big(\bigiprod{X_{\pi_m}}{U \mu^{(t)}+(1-U) \mu_{[-i]}^{(t)}  }, \iprod{X_{\pi_m}}{\mu_\ast},\xi_{\pi_m}\big) X_{\pi_m}X_{\pi_m}^\top.
	\end{align*}
	Using Proposition \ref{prop:weighted_sample_cov} and the same arguments as in the beginning of Steps 1 and 2, on an event with $\Prob^{(0)}$-probability at least $1-cm^{-200}$, 
	\begin{align*}
	&\bigpnorm{M_{\mu^{(t)}, \mu_{[-i]}^{(t)} }(X)-\mathbb{M}_{ \mathrm{u}_\ast^{(t)},\mathrm{u}_{\ast}^{(t)} }^{\mathsf{Z}} }{\op}\\
	&\leq \frac{(\log n)^{c_1}}{(n\wedge \phi)^{1/2}}+ \Lambda^{c_{\mathfrak{p}}}\cdot \max_{k \in [m]}\Big(1+\abs{\iprod{X_k}{\mu^{(t)} }}+\abs{\iprod{X_k}{\mu_{[-i]}^{(t)} }}+\abs{\iprod{X_k}{\mathrm{u}_\ast^{(t)} }  }+\abs{\iprod{X_k}{\mu_\ast}  }\Big)^{c_{\mathfrak{p}}}\\
	&\qquad\qquad \times \Big(\max_{k\in [m]} \abs{\iprod{X_k}{\mu_{[-i]}^{(t)}-\mathrm{u}_\ast^{(t)} }  }+\max_{k\in [m]} \abs{\iprod{X_k}{\mu^{(t)}-\mathrm{u}_\ast^{(t)} }  } \Big)\cdot \bigpnorm{\E_{\pi_m} X_{\pi_m} X_{\pi_m}^\top}{\op}\\
	&\leq \frac{(\log n)^{c_1}}{(n\wedge \phi)^{1/2}}+(\log n)^{c_2}\cdot  \Big(\max_{k\in [m]} \abs{\iprod{X_k}{\mu^{(t)}-\mathrm{u}_\ast^{(t)} }  }  +\max_{k\in [m]}\pnorm{X_k}{}\cdot \pnorm{\mu^{(t)}-\mu_{[-i]}^{(t)} }{} \Big)\\
	&\leq \frac{(\log n)^{c_1}}{(n\wedge \phi)^{1/2}}+ (\log n)^{c_2}\cdot \bigg\{\max_{k \in [m]}\pnorm{X_k}{}\cdot \pnorm{\mu^{(t)}-\mu_{[-i]}^{(t)}}{}  +\max_{i \in [m]} \Big(\abs{\iprod{X_{i}}{\mu^{(t)}-\mu_{[-i]}^{(t)}}}\\
	&\qquad\qquad\qquad\qquad\qquad\qquad\qquad\qquad +\abs{\iprod{X_{i}}{\mu_{[-i]}^{(t)}-\mathrm{u}_X^{(t)}  }} + \abs{ \iprod{X_{i}}{ \mathrm{u}_\ast^{(t)}-\mathrm{u}_X^{(t)}  } }\Big)\bigg\}.
	\end{align*}
	From here the control of the second term in the bracket on right hand side of the above display parallels exactly as in the Step 1; the first term in the bracket can be handled via Lemma \ref{lem:grad_des_loo_diff} and the induction hypothesis (\ref{ineq:nonconvex_generic_conv_long_Mcal_step1_ind_hyp}). In other words, with $c_\ast=c_\ast(\mathfrak{p},c_0,\eta_\ast,\sigma_\ast)>0$ chosen large enough, if $\phi\geq (\log m)^{c_1}$, on an event $E_{1,2}$ with $\Prob^{(0)}$-probability at least $1-cm^{-200}$, uniformly for $t\leq m^{100}$, for $n\geq n_0(\mathfrak{p},c_0,\eta_\ast,\sigma_\ast)$, 
	\begin{align*}
	&\bigpnorm{M_{\mu^{(t)}, \mu_{[-i]}^{(t)} }(X)-\mathbb{M}_{ \mathrm{u}_\ast^{(t)},\mathrm{u}_{\ast}^{(t)} }^{\mathsf{Z}} }{\op}\leq \epsilon_n.
	\end{align*}
	Now using the same calculations as in (\ref{ineq:nonconvex_generic_conv_long_Mcal_step1_local_op}), there exists some $c_1''=c_1''(\mathfrak{p},c_0,\eta_\ast,\sigma_\ast)>0$, such that with $t_n''\equiv t_n+c_1''\log \log n$, on $E_{1,2}$, uniformly for $t\leq m^{100}$, for $n\geq n_0(\mathfrak{p},c_0,\eta_\ast,\sigma_\ast)$, 
	\begin{align*}
	 \max_{i \in [m]} \mathfrak{M}_{\mu^{(\cdot)}, \mu_{[-i]}^{(\cdot)} }^{(t-1)}(X) \leq  \mathcal{B}(t_n'') + c_2\leq \phi^{1/2}/(\log n)^{c_\ast}.
	\end{align*}
		
	\noindent (\textbf{Step 3}). On the event $E_\ast^{(t-1)}\cap E_1$, where $E_1\equiv \cap_{\ell \in [2]}E_{1,\ell}$, uniformly for $t\leq m^{100}$, for $n\geq n_0(\mathfrak{p},c_0,\eta_\ast,\sigma_\ast)$, if $\phi\geq (\log m)^{c_1}$,
	\begin{align*}
	\hbox{(\ref{ineq:nonconvex_generic_conv_long_Mcal_step1_ind_hyp}) holds for $t-1$} \Rightarrow\, \hbox{(\ref{ineq:nonconvex_generic_conv_long_Mcal_step1_ind_hyp}) holds for $t$}.
	\end{align*}
	It is easy to verify the initial condition,
	so the above induction remains valid for each $t\leq m^{100}$. \qed
	
\subsection{Proof of Theorem \ref{thm:nonconvex_generic_conv_long}}\label{subsection:proof_nonconvex_generic_conv_long_complete}

	The claim follows by combining the estimates obtained Proposition \ref{prop:nonconvex_generic_conv_long_Mscr}, Proposition \ref{prop:nonconvex_generic_conv_long_Mcal} and the master Theorems \ref{thm:grad_des_se} and \ref{thm:grad_des_incoh}.\qed

\subsection{Proof of Corollary \ref{cor:nonconvex_generic_conv}}\label{subsection:proof_nonconvex_generic_conv}

	By the step size choice (\ref{cond:nonconvex_generic_conv_long_eta}), we may control $\mathcal{B}_0(t_n)$ and $\mathcal{B}(t_n)$ via a trivial bound under $t_n\leq c_0\log \log n$: for $n\geq n_0(\mathfrak{p},c_0)$, for some $c_1=c_1(\mathfrak{p},c_0)>0$,
	\begin{align*}
	\mathcal{B}_0(t_n)\vee \mathcal{B}(t_n)\leq \sum_{s \in [0:t_n)} (2+\epsilon_n)^{t_n-s}\leq e^{t_n+1}\leq e\cdot (\log n)^{c_1}.
	\end{align*}
	On the other hand, using the estimate (\ref{ineq:cor_general_se_2_0}), we have
	\begin{align}\label{ineq:nonconvex_generic_conv_1}
	\pnorm{\mathrm{u}_{\ast}^{(t)}}{\infty}
	&\leq (\log n)^{c_1} \cdot \Big(1+\max_{s \in [0:t-1]} \pnorm{\mathrm{u}_\ast^{(s)}}{}\Big)^{c_{\mathfrak{p}}}\cdot n^{-1/2} \mathscr{M}_{\tau_\ast^{(\cdot)}}^{(t-1)},
	\end{align}
	where by definition $\mathscr{M}_{\tau_\ast^{(\cdot)}}^{(\cdot)}$ in (\ref{ineq:cor_general_se_1}), we have for $t\geq t_n$ and some $c_2=c_2(\eta_\ast,\sigma_\ast)>0$,
	\begin{align}\label{ineq:nonconvex_generic_conv_2}
	\mathscr{M}_{\tau_\ast^{(\cdot)}}^{(t)}&= 1+ \max_{\tau \in [0:t]}\sum_{s \in [0:\tau]} \bigg(\prod_{r \in [s:\tau]}\abs{1-\eta_r\cdot \tau_\ast^{(r)}}\bigg)\nonumber\\
	&\leq 1+ \max_{\tau \in [0:t]}\sum_{s \in [0:\tau]} \bigg(2^{(t_n-s)_+}\prod_{r \in [t_n\vee s:\tau]}\abs{1-\eta_r\cdot \tau_\ast^{(r)}}\bigg)\nonumber\\
	&\stackrel{(\ast)}{\leq} 1+ 2^{t_n}\cdot \max_{\tau \in [0:t]} \bigg( t_n\cdot (1-\eta_\ast\sigma_\ast)^{(\tau-t_n+1)_+}+\sum_{s \in [t_n:\tau]}(1-\eta_\ast\sigma_\ast)^{(\tau-s+1)_+}\bigg)\nonumber\\
	&\leq c_2 \cdot (\log n)^{c_1}. 
	\end{align}
	Here in $(\ast)$ we used Lemma \ref{lem:nonconvex_generic_conv_long_tau}.
	
	Furthermore, by (\ref{ineq:cor_general_se_1_0}),
	\begin{align*}
	\pnorm{\mathrm{u}_\ast^{(t)}}{}&\leq \mathscr{M}_{ \mathrm{u}_\ast^{(\cdot)},\mu_\ast }^{(t-1),\mathsf{Z}} \cdot m^{-1} \bigpnorm{\E^{(0)}  \mathsf{Z}^\top \mathfrak{S}\big(\mathsf{Z}\mu_\ast,\mathsf{Z} \mu_\ast,\xi\big)}{}\leq (\log n)^{c_1}\cdot \mathscr{M}_{ \mathrm{u}_\ast^{(\cdot)},\mu_\ast }^{(t-1),\mathsf{Z}}.
	\end{align*}
	Using a similar calculation as in (\ref{ineq:nonconvex_generic_conv_2}), $\mathscr{M}_{ \mathrm{u}_\ast^{(\cdot)},\mu_\ast }^{(t),\mathsf{Z}}\leq c_2(\log n)^{c_1}$, so 
	\begin{align}\label{ineq:nonconvex_generic_conv_3}
	\pnorm{\mathrm{u}_\ast^{(t)}}{}&\leq c_2 \cdot (\log n)^{c_1}. 
	\end{align}
	Combining (\ref{ineq:nonconvex_generic_conv_1})-(\ref{ineq:nonconvex_generic_conv_3}), we then arrive at
	\begin{align*}
	n^{1/2}\pnorm{\mathrm{u}_\ast^{(t)}}{\infty}\leq c_2 \cdot (\log n)^{c_1}. 
	\end{align*}
	Now we may apply Theorem \ref{thm:nonconvex_generic_conv_long} by adjusting the constant $c_0>0$ in the hypothesis that may further depends on $\mathfrak{p}$.\qed

\section{Proofs for Section \ref{section:concrete_models}}

\subsection{Proof of Theorem \ref{thm:single_index_se}}\label{subsection:proof_single_index_se}

\begin{lemma}\label{lem:single_indx_S_fcn}
	In the single-index regression setting,
	\begin{align*}
	\mathfrak{S}(x_0,z_0,\xi_0) & = (\varphi(x_0)-\varphi(z_0)-\xi_0)\cdot \varphi'(x_0),\\
	\partial_1 \mathfrak{S}(x_0,z_0,\xi_0) & =  \big(\varphi'(x_0)\big)^2 +(\varphi(x_0)-\varphi(z_0)-\xi_0)\cdot \varphi''(x_0),\\
	\partial_2 \mathfrak{S}(x_0,z_0,\xi_0) & = - \varphi'(x_0) \varphi'(z_0).
	\end{align*}
\end{lemma}
\begin{proof}
	As $\partial_1\mathsf{L}(x_0,y_0)=(\varphi(x_0)-y_0)\varphi'(x_0)$, by definition of $\mathfrak{S}$,
	\begin{align*}
	\mathfrak{S}(x_0,z_0,\xi_0) & = (\varphi(x_0)-\varphi(z_0)-\xi_0)\cdot \varphi'(x_0).
	\end{align*}
	Taking derivatives with respect to $x_0,z_0$ for both sides of the above display concludes the claimed identities.
\end{proof}

\begin{lemma}\label{lem:single_index_mnt_theo_gd}
Suppose the conditions in Theorem \ref{thm:single_index_se} holds. Then there exists some $\eta_0=\eta_0(\mathfrak{p},\varrho_\ast,\kappa,\Lambda,\pnorm{\mu^{(0)}}{}) \in (0,1)$ such that for $\eta \leq \eta_0$, on the event $\big\{\abs{\E_{\pi_m}\xi_{\pi_m}}\leq \eta^{1/2}\}$, we have for $t=1,2,\ldots$,
\begin{align*}
\pnorm{\mathrm{u}_\ast^{(t)}-\mu_\ast}{ }\leq \bigg(1-\frac{\eta }{4}(\kappa_\ast \wedge \varrho_\ast)\bigg)^t\cdot \pnorm{\mu^{(0)}-\mu_\ast}{}.
\end{align*}
\end{lemma}
\begin{proof}
	 Let the population loss be 
	 \begin{align}\label{ineq:single_index_mnt_theo_pop_loss}
	 \mathsf{R}(u)&\equiv \E^{(0)}\big(\varphi(\iprod{\mathsf{Z}_n}{u})-\varphi(\iprod{\mathsf{Z}_n}{\mu_\ast})-\xi_{\pi_m}\big)^2/2,\quad \forall u \in \R^n.
	 \end{align}
	 It is easy to compute that
	 \begin{align}\label{ineq:single_index_mnt_theo_pop_loss_1}
	 \nabla \mathsf{R}(u)&\equiv \E^{(0)}\big(\varphi(\iprod{\mathsf{Z}_n}{u})-\varphi(\iprod{\mathsf{Z}_n}{\mu_\ast})-\xi_{\pi_m}\big) \varphi'(\iprod{\mathsf{Z}_n}{u})\cdot \mathsf{Z}_n,\\
	 \nabla^2 \mathsf{R}(u)&\equiv \E^{(0)}\Big[\big(\varphi'(\iprod{\mathsf{Z}_n}{u})\big)^2 -\big(\varphi(\iprod{\mathsf{Z}_n}{u})-\varphi(\iprod{\mathsf{Z}_n}{\mu_\ast})-\xi_{\pi_m}\big) \varphi''(\iprod{\mathsf{Z}_n}{u})\Big]\cdot \mathsf{Z}_n\mathsf{Z}_n^\top.\nonumber
	 \end{align}
	 Using $\mathsf{R}$ defined in (\ref{ineq:single_index_mnt_theo_pop_loss}), we may rewrite (\ref{def:se_single_index}) as
	 \begin{align}\label{ineq:single_index_mnt_theo_u_ast}
	 \mathrm{u}_\ast^{(t)} = \mathrm{u}_\ast^{(t-1)}-\eta\cdot \nabla \mathsf{R}(\mathrm{u}_\ast^{(t-1)}).
	 \end{align}
	
	\noindent (\textbf{Step 1}). In this step, we first prove that for all $u \in \R^n$,
	\begin{align}\label{ineq:single_index_mnt_theo_step1}
    \bigiprod{u-\mu_\ast}{\nabla\mathsf{R}(u)}&\geq \frac{1}{2}\kappa^2\big(3(\pnorm{u}{}+\pnorm{\mu_\ast}{})\big)\cdot \pnorm{u-\mu_\ast}{}^2\nonumber\\
    &\qquad - \Lambda^{c_{\mathfrak{p}}}\cdot \abs{\E_{\pi_m}\xi_{\pi_m}}\cdot  (1+\pnorm{u}{})^{\mathfrak{p}}\cdot \pnorm{u-\mu_\ast}{}.
	\end{align} 
	First note that with $U\sim\mathrm{Unif}[0,1]$,
	\begin{align*}
	&\bigiprod{u-\mu_\ast}{\nabla\mathsf{R}(u)} \\
	&= \E^{(0)} \varphi'(\iprod{\mathsf{Z}_n}{u})\cdot  \big(\varphi(\iprod{\mathsf{Z}_n}{u})-\varphi(\iprod{\mathsf{Z}_n}{\mu_\ast})-\xi_{\pi_m}\big)\iprod{\mathsf{Z}_n}{u-\mu_\ast}\\
	& = \E^{(0)} \Big[\varphi'(\iprod{\mathsf{Z}_n}{u})\cdot \varphi'\big(\iprod{\mathsf{Z}_n}{U u+(1-U)\mu_\ast}\big) \cdot \iprod{\mathsf{Z}_n}{u-\mu_\ast}^2\Big]\\
	&\qquad - \E_{\pi_m}\xi_{\pi_m}\cdot \E^{(0)} \varphi'(\iprod{\mathsf{Z}_n}{u})\cdot  \iprod{\mathsf{Z}_n}{u-\mu_\ast},
	\end{align*}
	Consider the event $E_u(z)\equiv \big\{\abs{\iprod{\mathsf{Z}_n}{u}}\leq \pnorm{u}{}z,\, \abs{\iprod{\mathsf{Z}_n}{\mu_\ast}}\leq \pnorm{\mu_\ast}{}z\big\}$. Then $\Prob^{(0)}\big(E_u^c(z)\big)\leq 4e^{-z^2/2}$ and the above display yields that 
	\begin{align}\label{ineq:single_index_mnt_theo_step1_1}
	\bigiprod{u-\mu_\ast}{\nabla\mathsf{R}(u)} &\geq \kappa^2\big((\pnorm{u}{}+\pnorm{\mu_\ast}{})z\big)\cdot \E^{(0)} \iprod{\mathsf{Z}_n}{u-\mu_\ast}^2\bm{1}_{E_u(z)}\nonumber\\
	&\qquad - \Lambda^{c_{\mathfrak{p}}}\cdot \abs{\E_{\pi_m}\xi_{\pi_m}}\cdot  (1+\pnorm{u}{})^{\mathfrak{p}}\cdot \pnorm{u-\mu_\ast}{}.
	\end{align}
	On the other hand, 
	\begin{align}\label{ineq:single_index_mnt_theo_step1_2}
	\E^{(0)} \iprod{\mathsf{Z}_n}{u-\mu_\ast}^2\bm{1}_{E_u(z)}&=\E^{(0)} \iprod{\mathsf{Z}_n}{u-\mu_\ast}^2-\E^{(0)} \iprod{\mathsf{Z}_n}{u-\mu_\ast}^2\bm{1}_{E_u^c(z)}\nonumber\\
	& \geq  \pnorm{u-\mu_\ast}{}^2- \Prob^{(0),1/2}\big(E_u^c(z)\big)\cdot \E^{(0),1/2} \iprod{\mathsf{Z}_n}{u-\mu_\ast}^4\nonumber\\
	&\geq \pnorm{u-\mu_\ast}{}^2\cdot \big(1-2\sqrt{3} e^{-z^2/4}\big).
	\end{align}
	The claimed estimate (\ref{ineq:single_index_mnt_theo_step1}) now follows by combining (\ref{ineq:single_index_mnt_theo_step1_1}) and (\ref{ineq:single_index_mnt_theo_step1_2}) with the choice $z=3$. 

	\noindent (\textbf{Step 2}). In this step, we prove that there exists $c_1=c_1(\kappa,\Lambda,\pnorm{\mu^{(0)}}{})>0$ such that 
	\begin{align}\label{ineq:single_index_mnt_theo_step2}
	\pnorm{\Delta_{\ast}^{(t-1)}}{ }\geq c_1\cdot \max\big\{ \abs{\E_{\pi_m}\xi_{\pi_m}}, \eta^{1/2} \big\}\,\Rightarrow\, \pnorm{\Delta_{\ast}^{(t)} }{ }\leq \bigg(1-\frac{\eta \kappa_\ast}{4}\bigg)\cdot \pnorm{\Delta_{\ast}^{(t-1)}}{}
	\end{align}
	holds for $t=1,2,\ldots$. Here $\Delta_{\ast}^{(t)}\equiv \mathrm{u}_{\ast}^{(t)}-\mu_\ast $.

	To this end, suppose the right hand  side of (\ref{ineq:single_index_mnt_theo_step2}) is valid up to $t-1$. By the identity in (\ref{ineq:single_index_mnt_theo_u_ast}) and the estimate in (\ref{ineq:single_index_mnt_theo_step1}),
	\begin{align}\label{ineq:single_index_mnt_theo_step2_1}
	&\pnorm{\Delta_\ast^{(t)} }{ }^2- \pnorm{\Delta_\ast^{(t-1)} }{ }^2 \nonumber\\
	&= \pnorm{\Delta_\ast^{(t-1)}-\eta\cdot \nabla \mathsf{R}(\mathrm{u}_\ast^{(t-1)}) }{ }^2- \pnorm{\Delta_\ast^{(t-1)} }{ }^2\nonumber\\
	& = -2\eta\cdot \bigiprod{\Delta_\ast^{(t-1)}}{ \nabla \mathsf{R}(\mathrm{u}_\ast^{(t-1)}) }+ \eta^2\cdot \pnorm{\nabla \mathsf{R}(\mathrm{u}_\ast^{(t-1)})}{}^2\nonumber\\
    &\leq -\eta\cdot \kappa^2\big(3(\pnorm{\mathrm{u}_{\ast}^{(t-1)}}{}+\pnorm{\mu_\ast}{})\big)\cdot \pnorm{\Delta_{\ast}^{(t-1)} }{}^2  \nonumber\\
    &\qquad + \eta \Lambda^{c_{\mathfrak{p}}}\cdot \abs{\E_{\pi_m}\xi_{\pi_m}}\cdot  (1+\pnorm{\mathrm{u}_{\ast}^{(t-1)}}{})^{\mathfrak{p}}\cdot \pnorm{\Delta_{\ast}^{(t-1)}}{}+\eta^2\cdot \pnorm{\nabla \mathsf{R}(\mathrm{u}_{\ast}^{(t-1)})}{}^2.
	\end{align}
	As the right hand side of (\ref{ineq:single_index_mnt_theo_step2}) is valid up to $t-1$, it necessarily holds that $\pnorm{\Delta_\ast^{(t-1)}}{}=\pnorm{\mathrm{u}_\ast^{(t-1)}-\mu_\ast }{ }\leq \pnorm{\mu^{(0)}-\mu_\ast}{}$,  so $\pnorm{\mathrm{u}_\ast^{(t-1)}}{}\leq \pnorm{\mu^{(0)}}{}+2\pnorm{\mu_\ast}{}$. Consequently, 
	\begin{align*}
	\pnorm{\nabla \mathsf{R}(\mathrm{u}_\ast^{(t-1)})}{}&=\pnorm{\tau_\ast^{(t-1)}\cdot \mathrm{u}_\ast^{(t-1)} - \delta_\ast^{(t-1)}\cdot \mu_\ast}{} \nonumber\\
	&\leq \big[\Lambda (1+ \pnorm{\mathrm{u}_\ast^{(t-1)}}{}+\pnorm{\mu_\ast}{} )\big]^{ c_{\mathfrak{p}} }\leq \big[\Lambda (1+\pnorm{\mu^{(0)}}{})\big]^{c_{\mathfrak{p}}}.
	\end{align*}
	Combined with (\ref{ineq:single_index_mnt_theo_step2_1}), we have
	\begin{align*}
	\pnorm{\Delta_{\ast}^{(t)} }{ }^2 &\leq (1-\eta \kappa_\ast)\cdot  \pnorm{\Delta_{\ast}^{(t-1)} }{ }^2\\
	&\qquad  + \eta\cdot \big[\Lambda (1+\pnorm{\mu^{(0)}}{})\big]^{c_{\mathfrak{p}}'}\cdot   \abs{\E_{\pi_m}\xi_{\pi_m}}\cdot \pnorm{\Delta_{\ast}^{(t-1)}}{} +\eta^2\cdot \big[\Lambda (1+\pnorm{\mu^{(0)}}{})\big]^{c_{\mathfrak{p}}'}.
	\end{align*}
	This means if 
	\begin{itemize}
		\item $\pnorm{\Delta_{\ast}^{(t-1)}}{}\geq 4\big[\Lambda (1+\pnorm{\mu^{(0)}}{})\big]^{c_{\mathfrak{p}}'}\cdot   \abs{\E_{\pi_m}\xi_{\pi_m}}/\kappa_\ast$,
		\item $\pnorm{\Delta_{\ast}^{(t-1)}}{}\geq  2\eta^{1/2} \big[\Lambda (1+\pnorm{\mu^{(0)}}{})\big]^{c_{\mathfrak{p}}'/2}/\kappa_\ast^{1/2}$,
	\end{itemize}
	we have $
	\pnorm{\Delta_{\ast}^{(t)} }{ }^2 \leq \big(1-{\eta \kappa_\ast}/{2}\big)\cdot  \pnorm{\Delta_{\ast}^{(t-1)} }{ }^2$, proving (\ref{ineq:single_index_mnt_theo_step2}).
	
	\noindent (\textbf{Step 3}). In view of (\ref{ineq:single_index_mnt_theo_pop_loss_1}), we have 
	\begin{align}\label{ineq:single_index_se_pop_loss_hessian_est_1}
	\nabla^2 \mathsf{R}(u)&\geq \E^{(0)} \big(\varphi'(\iprod{\mathsf{Z}_n}{\mu_\ast})\big)^2 \mathsf{Z}_n\mathsf{Z}_n^\top -\Lambda^{c_{\mathfrak{p}}}(1+\abs{ \E_{\pi_m}\xi_{\pi_m}})\nonumber\\
	&\qquad\qquad \times \E^{(0)}(1+\abs{\iprod{\mathsf{Z}_n}{u}}+\abs{\iprod{\mathsf{Z}_n}{\mu_\ast}})^{c_{\mathfrak{p}}}\abs{\iprod{\mathsf{Z}_n}{u-\mu_\ast}}\mathsf{Z}_n\mathsf{Z}_n^\top.
	\end{align}
	Note that by Gaussian integration-by-parts, with $\mathsf{H}_\varphi\equiv (\varphi')^2$,
	\begin{align*}
	\E^{(0)} \big(\varphi'(\iprod{\mathsf{Z}_n}{\mu_\ast})\big)^2 \mathsf{Z}_n\mathsf{Z}_n^\top= \E \mathsf{H}_\varphi(\iprod{\mathsf{Z}_n}{\mu_\ast})\cdot I_n +\E \mathsf{H}_\varphi''(\iprod{\mathsf{Z}_n}{\mu_\ast})\mu_\ast\mu_\ast^\top.
	\end{align*}
	This means that 
	\begin{align}\label{ineq:single_index_se_pop_loss_hessian_est_2}
	&\lambda_{\min}\Big(\E^{(0)} \big(\varphi'(\iprod{\mathsf{Z}_n}{\mu_\ast})\big)^2 \mathsf{Z}_n\mathsf{Z}_n^\top\Big)\nonumber\\
	&=\min\big\{\E^{(0)}\big( \varphi'(\mathsf{Z}_1)\big)^2, \E^{(0)}\big(\mathsf{Z}_1 \varphi'(\mathsf{Z}_1)\big)^2 \big\}=\varrho_\ast.
	\end{align}
	On the other hand,
	\begin{align}\label{ineq:single_index_se_pop_loss_hessian_est_3}
	&\bigpnorm{\E^{(0)}(1+\abs{\iprod{\mathsf{Z}_n}{u}}+\abs{\iprod{\mathsf{Z}_n}{\mu_\ast}})^{c_{\mathfrak{p}}}\abs{\iprod{\mathsf{Z}_n}{u-\mu_\ast}}\mathsf{Z}_n\mathsf{Z}_n^\top}{\op}\nonumber\\
	& = \sup_{v \in B_n(1)} \E^{(0)}(1+\abs{\iprod{\mathsf{Z}_n}{u}}+\abs{\iprod{\mathsf{Z}_n}{\mu_\ast}})^{c_{\mathfrak{p}}}\abs{\iprod{\mathsf{Z}_n}{u-\mu_\ast}}\iprod{\mathsf{Z}_n}{v}^2\nonumber\\
	&\leq c\cdot (1+\pnorm{u}{})^{c_{\mathfrak{p}}}\cdot \pnorm{u-\mu_\ast}{}.
	\end{align}	
	Combining (\ref{ineq:single_index_se_pop_loss_hessian_est_1})-(\ref{ineq:single_index_se_pop_loss_hessian_est_3}), there exists $c_2=c_2(\mathfrak{p},\varrho_\ast,\Lambda)>0$ such that if $\pnorm{u-\mu_\ast}{}\leq 1/c_2$ and $\abs{\E_{\pi_m}\xi_{\pi_m}}\leq 1$, we have
	\begin{align*}
	\frac{\varrho_\ast}{2}\cdot I_n\leq \nabla^2 \mathsf{R}(u)\leq c_2 \cdot I_n.
	\end{align*}
	Now using the standard convex optimization theory (cf. Lemma \ref{lem:cvx_gd_generic}), for $\eta<1/c_2$,
	\begin{align}\label{ineq:single_index_mnt_theo_step3_1}
	\pnorm{\Delta_\ast^{(t-1)}}{}<1/c_2\,\Rightarrow\,\pnorm{\Delta_\ast^{(t)}}{}^2 &\leq \bigg(1-\frac{\eta \varrho_\ast}{2}\bigg)\cdot  \pnorm{\Delta_\ast^{(t-1)}}{}^2.
	\end{align}
	Let us now choose $\eta\leq 1\wedge (1/c_2)\wedge \{1/(c_1c_2)^2\}$. Then if $\abs{\E_{\pi_m}\xi_{\pi_m}}\leq \eta^{1/2}$, it must also satisfy $\abs{\E_{\pi_m}\xi_{\pi_m}}\leq 1$. Now combining (\ref{ineq:single_index_mnt_theo_step2}) and (\ref{ineq:single_index_mnt_theo_step3_1}), we have for $t=1,2,\ldots$,
	\begin{align*}
	\pnorm{\Delta_\ast^{(t)} }{ }\leq \bigg(1-\frac{\eta }{4}(\kappa_\ast\wedge \varrho_\ast)\bigg)\cdot \pnorm{\Delta_\ast^{(t-1)}}{}.
	\end{align*}
	The claim follows by iterating the above estimate.
\end{proof}

\begin{proof}[Proof of Theorem \ref{thm:single_index_se}]
By Lemma \ref{lem:single_index_mnt_theo_gd} and Proposition \ref{prop:nonconvex_generic_conv_long_contraction}-(2), the condition (N2) is satisfied with $t_n\asymp\log \log n$ and $\sigma_\ast=\varrho_\ast/2$, so Corollary \ref{cor:nonconvex_generic_conv} applies and the conclusion of Theorem \ref{thm:nonconvex_generic_conv_long} is valid on the event $\{\abs{\E_{\pi_m}\xi_{\pi_m}}\leq \eta^{1/2}\}$. This concludes the first inequality. The second inequality follows by a further application of Lemma \ref{lem:single_index_mnt_theo_gd}.
\end{proof}

\subsection{Proof of Theorem \ref{thm:pr_se}}\label{subsection:proof_pr_se}

We need a few auxiliary results.

\begin{lemma}\label{lem:pr_min_eig}
	Suppose the initialization condition (\ref{cond:pr_initialization}) holds and $\abs{\E_{\pi_m} \xi_{\pi_m}}\leq 1/2$. There exists a universal small constant $\epsilon_0 \in (0,1/2)$ such that the following hold:
	\begin{enumerate}
		\item $\pnorm{\mathbb{M}_{ \mathrm{u}_{\ast}^{(t)},\mu_\ast }^{\mathsf{Z}}}{\op}\vee \pnorm{\mathbb{M}_{ \mu_\ast,\mu_\ast }^{\mathsf{Z}}}{\op}\leq 1/\epsilon_0$, and 
		\item $
		\pnorm{\mathrm{u}_{\ast}^{(t)}-\mu_\ast}{}\leq \epsilon_0\,\Rightarrow\, \lambda_{\min}\big(\mathbb{M}_{ \mathrm{u}_{\ast}^{(t)},\mu_\ast }^{\mathsf{Z}}\big)\geq 2$.
	\end{enumerate}
\end{lemma}
\begin{proof}
	Using that for $\mathrm{u}\in \{\mathrm{u}_{\ast}^{(t)},\mu_\ast\}$, 
	\begin{align*}
	\mathbb{M}_{ \mathrm{u},\mu_\ast }^{\mathsf{Z}}& = 2\E^{(0)} \big(3 \bigiprod{\mathsf{Z}_{1}}{U \mathrm{u}+(1-U) \mu_\ast }^2-\iprod{\mathsf{Z}_{1}}{\mu_\ast}^2\big) \mathsf{Z}_1\mathsf{Z}_1^\top - 2 \E_{\pi_m} \xi_{\pi_m}\cdot I_n,
	\end{align*}
	we have 
	\begin{align*}
	&\bigpnorm{\mathbb{M}_{ \mathrm{u}_{\ast}^{(t)},\mu_\ast }^{\mathsf{Z}}-\mathbb{M}_{ \mu_\ast,\mu_\ast }^{\mathsf{Z}} }{\op}\\
	&\leq c\cdot \sup_{u \in B_n(1)} \E^{(0)} \Big\{ \abs{ \iprod{\mathsf{Z}_1}{ \mathrm{u}_{\ast}^{(t)}-\mu_\ast} }\cdot \big(\abs{ \iprod{\mathsf{Z}_1}{ \mathrm{u}_{\ast}^{(t)}}}+\abs{ \iprod{\mathsf{Z}_1}{\mu_\ast}}\big)\cdot \iprod{\mathsf{Z}_1}{u}^2\Big\}\\
	&\leq c\cdot \pnorm{\mathrm{u}_{\ast}^{(t)}-\mu_\ast}{}\cdot \big(1+\pnorm{\mathrm{u}_{\ast}^{(t)}}{}\big).
	\end{align*}
	Under the initialization condition (\ref{cond:pr_initialization}), \cite[Eq. (24) in Lemma 1]{chen2019gradient} proves that $\pnorm{\mathrm{u}_{\ast}^{(t)}}{}^2=\alpha_t^2+\beta_t^2\leq 2^2+(1.5)^2<8$, so under $\abs{\E_{\pi_m} \xi_{\pi_m}}\leq 1/2$,
	\begin{align*}
	\lambda_{\min}\Big(\mathbb{M}_{ \mathrm{u}_{\ast}^{(t)},\mu_\ast }^{\mathsf{Z}}\Big)&\geq \lambda_{\min}\Big(\mathbb{M}_{ \mu_\ast,\mu_\ast }^{\mathsf{Z}}\Big)-c\cdot \pnorm{\mathrm{u}_{\ast}^{(t)}-\mu_\ast}{}\geq 3-c\cdot \pnorm{\mathrm{u}_{\ast}^{(t)}-\mu_\ast}{},
	\end{align*}
	where the last inequality follows as
	\begin{align*}
	\mathbb{M}_{ \mu_\ast,\mu_\ast }^{\mathsf{Z}}
	& = 4\E^{(0)}  \bigiprod{\mathsf{Z}_{1}}{\mu_\ast }^2 \mathsf{Z}_1\mathsf{Z}_1^\top - 2 \E_{\pi_m} \xi_{\pi_m}\cdot I_n= (4-2 \E_{\pi_m} \xi_{\pi_m}) \cdot I_n+8 \mu_\ast\mu_\ast^\top.
	\end{align*}
	The claims follow. 
\end{proof}

\begin{lemma}\label{lem:M_generic_eigenvalue}
	Fix $n\geq 3$. Let $u,v \in \R^n$ with $\pnorm{v}{}=1$. Further let
	\begin{align*}
	M&\equiv \pnorm{u}{}^2\cdot I_n+ 2 uu^\top+uv^\top + vu^\top,\\
	N&\equiv (3\pnorm{u}{}^2-1)\cdot I_n + 6 uu^\top - 2 vv^\top.
	\end{align*}
	Then the following hold:
	\begin{enumerate}
		\item $M$ has $n-2$ eigenvalues equal to $\pnorm{u}{}^2$, and the remaining two eigenvalues are given by 
		\begin{align*}
		2\pnorm{u}{}^2+\iprod{u}{v}\pm \sqrt{\big(\pnorm{u}{}^2+\iprod{u}{v}\big)^2+\pnorm{u}{}^2-\iprod{u}{v}^2}.
		\end{align*}
		\item $N$ has $n-2$ eigenvalues equal to $3\pnorm{u}{}^2-1$, and the remaining two eigenvalues are given by
		\begin{align*}
		2(3 \pnorm{u}{}^2-1) \pm \sqrt{ (3 \pnorm{u}{}^2-1)^2+12(\pnorm{u}{}^2-\iprod{u}{v}^2) }.
		\end{align*}
	\end{enumerate}
\end{lemma}
\begin{proof}
	First we consider $M$. As $M_0\equiv 2 uu^\top+uv^\top + vu^\top$ is a symmetric matrix with rank at most 2, $M$ has $n-2$ eigenvalues equal to $\pnorm{u}{}^2$, and remaining two real-valued eigenvalues equal to $\pnorm{u}{}^2$ plus the two nontrivial eigenvalues of $M_0$. As $\mathrm{range}(M_0)=\mathrm{span}(u,v)$, we shall compute the nontrivial eigenvalues of $M_0$ by representing $M_0|_{\mathrm{span}(u,v)}$ as a $2\times 2$ matrix in the space $\mathrm{span}(u,v)$. 
	
	To this end, let $e_1,e_2$ be an orthonormal basis of $\mathrm{span}(u,v)$ with  $e_1\equiv u/\pnorm{u}{}$. Let $\gamma\equiv \iprod{v}{e_1}$. Then we may write $v=\gamma e_1+\{1-\gamma^2\}^{1/2}e_2$, and
	\begin{align*}
	M_0&\equiv 2\pnorm{u}{}^2 e_1e_1^\top+\pnorm{u}{}e_1 \big(\gamma e_1+\{1-\gamma^2\}^{1/2}e_2\big)^\top+\big(\gamma e_1+\{1-\gamma^2\}^{1/2}e_2\big) \pnorm{u}{} e_1^\top\\
	& =2(\pnorm{u}{}^2+\pnorm{u}{}\gamma)e_1e_1^\top + \pnorm{u}{}\{1-\gamma^2\}^{1/2}\cdot (e_1e_2^\top+e_2e_1^\top).
	\end{align*}
	In other words, under $\{e_1,e_2\}$, $M_0$ has the matrix representation
	\begin{align*}
	\overline{M}_0\equiv 
	\begin{pmatrix}
	2(\pnorm{u}{}^2+\pnorm{u}{}\gamma) & \pnorm{u}{}\{1-\gamma^2\}^{1/2}\\
	\pnorm{u}{}\{1-\gamma^2\}^{1/2} & 0
	\end{pmatrix} \in \R^{2\times 2},
	\end{align*}
	and the two nontrivial eigenvalues of $M_0$ equal to those of $\overline{M}_0$. The eigenvalues $\lambda_\pm$ of $\overline{M}_0$ can be explicitly solved by the method of characteristic equation via $\det(\overline{M}_0-\lambda_{M;\pm} I_2)=0$, given as
	\begin{align*}
	\lambda_{M;\pm} = \pnorm{u}{}^2+\pnorm{u}{}\gamma\pm \sqrt{\big(\pnorm{u}{}^2+\pnorm{u}{}\gamma\big)^2+\pnorm{u}{}^2(1-\gamma^2)}.
	\end{align*}
	The claim for $M$ follows by recalling $\pnorm{u}{}\gamma=\iprod{u}{v}$.
	
	For $N$, with $N_0\equiv 6 uu^\top - 2 vv^\top$, under $\{e_1,e_2\}$, it has matrix representation
	\begin{align*}
	\overline{N}_0\equiv 2 
	\begin{pmatrix}
	3\pnorm{u}{}^2 -\gamma^2 & - \gamma \{1-\gamma^2\}^{1/2}\\
	- \gamma \{1-\gamma^2\}^{1/2} & -(1-\gamma^2)
	\end{pmatrix} \in \R^{2\times 2},
	\end{align*}
	which has two eigenvalues 
	\begin{align*}
	\lambda_{N;\pm}& = (3 \pnorm{u}{}^2-1) \pm \sqrt{ (3 \pnorm{u}{}^2-1)^2+12\pnorm{u}{}^2(1-\gamma^2) }.
	\end{align*}
	The claim for $N$ follows. 
\end{proof}

We shall define several quantities as in \cite{chen2019gradient}: For a sufficiently large constant $c_0>1$, let
\begin{align}\label{def:pr_T0_Teps}
T_0&\equiv \min\{t:\alpha_{t+1}\sign(\alpha_0)\geq 1/(c_0\log^5 m)\},\nonumber\\
T_{\epsilon_0} &\equiv \min\{t: \abs{\alpha_t-1}\vee \beta_t\leq \epsilon_0/4\}.
\end{align}

\begin{proposition}\label{prop:pr_Bn}
	Suppose the initialization condition (\ref{cond:pr_initialization}) holds, $\phi\geq 1$ and $KL_\ast^{(0)}\vee \log m\leq (\log n)^{c_0}$, $\eta\leq 1/c_0$ and $\abs{\E_{\pi_m}\xi_{\pi_m}}\leq 1/(\log n)^{c_0}$ holds for sufficiently large $c_0>1$. Then there exist $c_1=c_1(c_0,\eta)>0$ and $c_2=c_2(c_0,\epsilon_0,\eta)>0$ such that $T_0\leq c_1 \log n$ and $T_{\epsilon_0}-T_0\leq c_2 \log \log n$. Moreover, with $t_n\equiv T_{\epsilon_0}$, we have $\mathcal{B}_0(t_n)\vee n^{-1/2}\mathcal{B}(t_n)\leq (\log n)^{c_2}$.
\end{proposition}
\begin{proof}
	As in the proof of \cite[Lemma 1]{chen2019gradient} in Appendix B therein, we shall focus on the difficult case that $1/\sqrt{n\log n}\leq \abs{\alpha_0}\leq \log n/\sqrt{n}$.
	
	\noindent (\textbf{Step 1}). First, by \cite[Lemma 1-(1)]{chen2019gradient} and the assumption $\log m\leq (\log n)^{c_0}$ and $\abs{\E_{\pi_m}\xi_{\pi_m}}\leq 1/(\log n)^{c_0}$, we have $T_0\leq c_1 \log n$, and $T_{\epsilon_0}-T_0\leq c_2\log\log n$. 
	
	Next, by (\ref{def:alpha_beta_pr}), for any $t\leq T_0$, if $\abs{\beta_{t-1}-1/\sqrt{3}}\leq 1/\log^4 n$, as $\beta_\cdot\lesssim 1$,
	\begin{align*}
	\abs{\beta_{t}-\beta_{t-1}}\leq 6\eta\cdot \big(\abs{\beta_{t-1}^2-1/3}+c/\log^5 m\big)\leq {c}/\log^4 n.
	\end{align*}
	On the other hand, if $\abs{\beta_{t-1}-1/\sqrt{3}}\geq 100/(\log^5 n)$, then for $\eta$ sufficiently small,
	\begin{align*}
	\abs{\beta_{t}-{1}/{\sqrt{3}}}& = \abs{\beta_{t-1}-{1}/{\sqrt{3}}}\cdot \bigabs{1-6\eta\cdot \big(\beta_{t-1}+{1}/{\sqrt{3}}+{\alpha_{t-1}^2}/(\beta_{t-1}-1/\sqrt{3})\big) \\
	&\qquad\qquad - 2\eta\cdot \E_{\pi_m}\xi_{\pi_m}/(\beta_{t-1}-1/\sqrt{3}) }\\
	&\leq \abs{1-\eta/100}\cdot \abs{\beta_{t-1}-{1}/{\sqrt{3}}}.
	\end{align*}
	Combining the above two displays, we conclude that 
	\begin{align*}
	 \abs{\beta_{t}-1/\sqrt{3}}\leq 100/\log^5 n+c/\log ^4 n\hbox{ for all } t \in \big[c_3\log \log n \wedge T_0, T_0\big].
	\end{align*}
	Here $c_3>0$ is universal. Consequently, by letting
	\begin{align}\label{ineq:pr_Bn_step1_t_def}
	t_n^{(1)}\equiv c_3\log\log n\wedge T_0,\quad t_n^{(2)}\equiv T_0,\quad t_n\equiv T_{\epsilon_0},
	\end{align}
	we have the following:
	\begin{itemize}
		\item For $t\leq t_n^{(1)}$ and $t \in (t_n^{(2)},t_n]$, $\beta_t\lesssim 1$.
		\item For $t \in (t_n^{(1)},t_n^{(2)}]$, $\alpha_t\leq 1/(c_0 \log^5 n)$ and $\abs{\beta_t-1/\sqrt{3}}\leq c/\log^4 n$.
	\end{itemize} 
	\noindent (\textbf{Step 2}). In this step, we prove that for small enough $\eta$, $t \in (t_n^{(1)},t_n^{(2)}]$ and $m\geq m_0$,
	\begin{align}\label{ineq:pr_Bn_step2}
	\lambda_{\min}\Big(\mathbb{M}_{ \mathrm{u}_\ast^{(t)},\mu_\ast }^{\mathsf{Z}}\Big)\geq  - c_1/\log^4 n.
	\end{align}
	First by using that $\E^{(0)}\iprod{\mathsf{Z}_1}{u}^2\mathsf{Z}_1\mathsf{Z}_1^\top= \pnorm{u}{}^2 I_n+ 2uu^\top$, we have 
	\begin{align*}
	&\mathbb{M}_{ \mathrm{u}_\ast^{(t)},\mu_\ast }^{\mathsf{Z}}+2 \E_{\pi_m} \xi_{\pi_m}\cdot I_n\\
	& = 2\E^{(0)} \big(3 \bigiprod{\mathsf{Z}_{1}}{U \mathrm{u}_\ast^{(t)}+(1-U) \mu_\ast }^2-\iprod{\mathsf{Z}_{1}}{\mu_\ast}^2\big) \mathsf{Z}_1\mathsf{Z}_1^\top\\
	& = \big(6 \E^{(0)}\pnorm{U \mathrm{u}_\ast^{(t)}+(1-U) \mu_\ast}{}^2-2\pnorm{\mu_\ast}{}^2\big)\cdot I_n\\
	&\qquad + 12\E^{(0)} \big(U \mathrm{u}_\ast^{(t)}+(1-U) \mu_\ast\big)\big(U \mathrm{u}_\ast^{(t)}+(1-U) \mu_\ast\big)^\top -4 \mu_\ast\mu_\ast^\top\\
	&=\big(2 \pnorm{\mathrm{u}_\ast^{(t)}}{}^2+2\iprod{\mathrm{u}_\ast^{(t)}}{\mu_\ast}\big)\cdot I_n +4 \mathrm{u}_\ast^{(t)}\mathrm{u}_\ast^{(t),\top}+ 2 \mathrm{u}_\ast^{(t)} \mu_\ast^\top+2 \mu_\ast \mathrm{u}_\ast^{(t),\top}.
	\end{align*}
	Let 
	\begin{align*}
	M_t\equiv 2 \pnorm{\mathrm{u}_\ast^{(t)}}{}^2\cdot I_n +4 \mathrm{u}_\ast^{(t)}\mathrm{u}_\ast^{(t),\top}+ 2 \mathrm{u}_\ast^{(t)} \mu_\ast^\top+2 \mu_\ast \mathrm{u}_\ast^{(t),\top}.
	\end{align*}
	Then for $t\leq t_n^{(2)}=T_0$,
	\begin{align}\label{ineq:pr_Bn_step2_1}
	\lambda_{\min}\Big(\mathbb{M}_{ \mathrm{u}_\ast^{(t)},\mu_\ast }^{\mathsf{Z}}\Big)\geq \lambda_{\min}(M_t) - c_1/\log^5 n.
	\end{align}
	As $\pnorm{\mathrm{u}_\ast^{(t)}}{}^2\geq 0.3$ for small enough $\eta$ (cf. \cite[Eq. (70)]{chen2019gradient}), an application of Lemma \ref{lem:M_generic_eigenvalue} entails that with
	\begin{align*}
	\lambda^{(t)}\equiv 4\pnorm{\mathrm{u}_\ast^{(t)}}{}^2-2\sqrt{\pnorm{\mathrm{u}_\ast^{(t)}}{}^4+\pnorm{\mathrm{u}_\ast^{(t)}}{}^2},
	\end{align*}
	we have 
	\begin{align}\label{ineq:pr_Bn_step2_2}
	\abs{\lambda_{\min}(M_t) -\lambda^{(t)}}\leq c_1/\log^5 n.
	\end{align}
	On the other hand, using the estimates obtained in Step 1, for $t \in (t_n^{(1)},t_n^{(2)}]$,
	\begin{align}\label{ineq:pr_Bn_step2_3}
	\abs{\lambda^{(t)}}\leq c \pnorm{\mathrm{u}_\ast^{(t)}}{}^2\cdot \abs{3\pnorm{\mathrm{u}_\ast^{(t)}}{}^2-1}\leq c\big(\abs{\beta_t^2-1/3}+\alpha_t^2\big)\leq c/(\log^4 n).
	\end{align}
	Combining (\ref{ineq:pr_Bn_step2_1}), (\ref{ineq:pr_Bn_step2_2}) and (\ref{ineq:pr_Bn_step2_3}) proves the claimed estimate (\ref{ineq:pr_Bn_step2}).

	\noindent (\textbf{Step 3}). In this step, we prove that for small enough $\eta$ and $t \in (t_n^{(1)},t_n^{(2)}]$,
	\begin{align}\label{ineq:pr_Bn_step3}
    \bigabs{\lambda_{\min}\Big(\mathbb{M}_{ \mathrm{u}_{\ast}^{(t)},\mathrm{u}_{\ast}^{(t)} }^{\mathsf{Z}}\Big)+4}\leq c/\log^4 n.
	\end{align}
	Note that
	\begin{align*}
	\mathbb{M}_{ \mathrm{u}_{\ast}^{(t)},\mathrm{u}_{\ast}^{(t)} }^{\mathsf{Z}}+2 \E_{\pi_m} \xi_{\pi_m}\cdot I_n & = 2\E^{(0)} \big(3 \bigiprod{\mathsf{Z}_{1}}{\mathrm{u}_{\ast}^{(t)} }^2-\iprod{\mathsf{Z}_{1}}{\mu_\ast}^2\big) \mathsf{Z}_1\mathsf{Z}_1^\top\\
	&= 2(3 \pnorm{\mathrm{u}_{\ast}^{(t)}}{}^2-1)\cdot I_n+12 \mathrm{u}_{\ast}^{(t)}\mathrm{u}_{\ast}^{(t),\top}-4\mu_\ast\mu_\ast^\top.
	\end{align*}
	The claim in (\ref{ineq:pr_Bn_step3}) follows by Lemma \ref{lem:M_generic_eigenvalue}.
	
	\noindent (\textbf{Step 4}). Now we shall provide an estimate for both $\mathcal{B}_0(t_n)$ and $\mathcal{B}(t_n)$, where $t_n$ is defined in (\ref{ineq:pr_Bn_step1_t_def}). In particular, we shall only use nontrivial estimates on the sub-interval $(t_n^{(1)},t_n^{(2)}]$ obtained in Steps 2 and 3, and the trivial estimates in Lemma \ref{lem:pr_min_eig} for the remaining parts. More specifically, recall $\epsilon_n=(\log n)^{-100}$, for $n\geq n_0(c_0,\epsilon_0,\eta)$,
	\begin{align*}
	\mathcal{B}_0(t_n)&=\sum_{s \in [0:t_n)} \bigg\{\prod_{r \in [s:t_n)\setminus (t_n^{(1)},t_n^{(2)}]}\prod_{r \in (t_n^{(1)},t_n^{(2)}]}\bigg\} \Big(1+\eta\cdot \Big[\lambda_{\min}\Big(\mathbb{M}_{ \mathrm{u}_{\ast}^{(r)},\mu_\ast }^{\mathsf{Z}}\Big)\Big]_- \Big)\\
	&\leq  2^{\abs{[0:t_n)\setminus (t_n^{(1)},t_n^{(2)}]}}\cdot t_n\cdot \big(1+c_1/\log^4 n\big)^{\abs{ (t_n^{(1)},t_n^{(2)}] }}\\
	&\leq c_2^{\log\log n}\cdot \log n\cdot \big(1+c_1/\log^4 n\big)^{c_1 \log n}\leq (\log n)^{c_2}.
	\end{align*}
	Moreover,
	\begin{align*}
	\mathcal{B}(t_n)&= \sum_{s \in [0:t_n)} \bigg\{\prod_{r \in [s:t_n)\setminus (t_n^{(1)},t_n^{(2)}]}\prod_{r \in (t_n^{(1)},t_n^{(2)}]}\bigg\} \Big(1+\eta\cdot \Big[ \lambda_{\min}\Big(\mathbb{M}_{ \mathrm{u}_{\ast}^{(t)},\mathrm{u}_{\ast}^{(t)} }^{\mathsf{Z}}\Big)\Big]_- +\epsilon_n\Big)\nonumber\\
	&\leq 2^{\abs{[0:t_n)\setminus (t_n^{(1)},t_n^{(2)}]}}\cdot t_n\cdot \big(1+4\eta+c/(\log n)^{4}\big)^{\abs{ (t_n^{(1)},t_n^{(2)}] }}\\
	&\leq c_2^{\log \log n}\cdot \log n\cdot (1+4\eta)^{\abs{ (t_n^{(1)},t_n^{(2)}] } }\leq n^{1/2}(\log n)^{c_2}.
	\end{align*}
	Here the last estimate follows as for all $t \in (t_n^{(1)},t_n^{(2)}]$, the first equation of (\ref{def:alpha_beta_pr}) entails that $\alpha_{t}\geq (1+4\eta-c/\log^4 n)\cdot  \alpha_{t-1}$, and therefore 
	\begin{align*}
	1/(c_0\log^5 n)\geq \alpha_{ t_n^{(2)}} &\geq (1+4\eta-c/\log^4 n)^{ \abs{ (t_n^{(1)},t_n^{(2)}] } } \cdot \alpha_{t_n^{(1)}+1}\\
	&\stackrel{(\ast)}{\geq} c (1+4\eta)^{ \abs{ (t_n^{(1)},t_n^{(2)}] } } \big/ (n \log n)^{1/2},
	\end{align*}
	where $(\ast)$ follows from \cite[Eq. (64a)]{chen2019gradient}.
\end{proof}

\begin{proof}[Proof of Theorem \ref{thm:pr_se}]
	The first desired claim follows from Lemma \ref{lem:pr_min_eig} and Proposition \ref{prop:pr_Bn}.

	For the second claim, note that by Proposition \ref{prop:nonconvex_generic_conv_long_contraction}-(1) and the trivial apriori estimate $\sup_{t} \pnorm{\mathrm{u}_\ast^{(t)}}{}\lesssim 1$, we have
	\begin{align*}
	\pnorm{\mathrm{u}_\ast^{(t)}-\mu_\ast}{}\leq (\log n)^{c_1}\cdot \big(1-\eta_\ast\sigma_\ast\big)^{(t-c_1'\log n)_+}\wedge c_1.
	\end{align*}
	By adjusting $c_1'$ in the exponent, we may replace the multiplicative logarithmic factor $(\log n)^{c_1}$ by $c_1$. The second claim then follows by (i) a modified version of Theorem \ref{thm:nonconvex_generic_conv_long} under the extra third moment condition $\E X_{ij}^3=0$, and (ii) an estimate for $\pnorm{\mathrm{u}_\ast^{(t)}}{\infty}$, as detailed below.
	
	For (i), under the extra third moment condition $\E X_{ij}^3=0$, a simple modification of the proof of (\ref{ineq:diff_theo_grad_des_step1}) by cumulant expansion (cf. Lemma \ref{lem:cum_exp}) up to the third order shows that the order of the right hand side of (\ref{ineq:diff_theo_grad_des_step1}) can be improved from $\bigo(n^{-1/2})$ to $\bigo(n^{-1})$ (with extra smoothness assumptions satisfied by the phase retrieval example we consider here). Therefore, the same rate improvement holds for Proposition \ref{prop:diff_theo_grad_des} and Theorem \ref{thm:grad_des_se}, where the term $n^{-1/2}\cdot  \mathscr{M}_{ \mathrm{u}_X^{(\cdot)},\mathrm{u}_{\ast}^{(\cdot)} }^{(t-1),X}$ can be replaced by $n^{-1}\cdot  \mathscr{M}_{ \mathrm{u}_X^{(\cdot)},\mathrm{u}_{\ast}^{(\cdot)} }^{(t-1),X}$. Consequently, (\ref{eqn:nonconvex_generic_conv_long}) in Theorem \ref{thm:nonconvex_generic_conv_long} now reads as follows:
	on an event with $\Prob^{(0)}$-probability at least $1-m^{-100}$, uniformly for $t\leq m^{100}$, 
	\begin{align*}
	(n^2\wedge \phi)^{1/2}\pnorm{\mu^{(t)}-\mathrm{u}_\ast^{(t)}}{} + \max_{i \in [m]}\abs{\iprod{X_i}{\mu^{(t)}}}\leq (\log n)^{c_2}\cdot \mathcal{B}(t\wedge t_n).
	\end{align*}
     For (ii), recall $T_0,T_{\epsilon_0}$ defined in (\ref{def:pr_T0_Teps}) and $t_n^{(1)},t_n^{(2)}=T_0$ defined in the proof of Proposition \ref{prop:pr_Bn} above. By (\ref{def:tau_delta_pr}), with
	\begin{align*}
	\gamma_t&\equiv 
	\begin{cases}
	1+c/\log^4 n, & t \in (t_n^{(1)},t_n^{(2)}];\\
	2, & t \in [0:t_n^{(1)}]\cup (t_n^{(2)}: T_{\epsilon_0});\\
	1-3\eta+c/\log^4 n, & t\geq T_{\epsilon_0},
	\end{cases}
	\end{align*}
	we have for small enough $\eta>0$,
	\begin{align*}
	\pnorm{\mathrm{u}_\ast^{(t+1)}}{\infty}&\leq \gamma_{t}\cdot \pnorm{\mathrm{u}_\ast^{(t)}}{\infty}+c\cdot \pnorm{\mu_\ast}{\infty}.
	\end{align*}
	Iterating the above estimate, we have for $t>T_{\epsilon_0}$,
	\begin{align*}
	&\pnorm{\mathrm{u}_\ast^{(t+1)}}{\infty}\leq c \pnorm{\mu_\ast}{\infty}\cdot \sum_{s \in [0:t]}\prod_{r \in [s:t]}\gamma_r\\
	&\leq c \pnorm{\mu_\ast}{\infty} \cdot 2^{c_1 \log \log n}\cdot \bigg(1+\frac{c}{\log^4 n}\bigg)^{c_1\log n} \sum_{s \in [0:t]} \bigg(1-3\eta+\frac{c}{\log^4 n}\bigg)^{(t-T_{\epsilon_0}\vee s+1)_+}\\
	&\leq (\log n)^{c_1}\cdot \pnorm{\mu_\ast}{\infty}\leq n^{-1/2}(\log n)^{c_1},
	\end{align*}
	proving the desired estimate for $\pnorm{\mathrm{u}_\ast^{(t)}}{\infty}$.
\end{proof}

\subsection{Proof of Theorem \ref{thm:single_index_cor_est}}\label{subsection:proof_single_index_cor_est}

The proof is divided into two steps. We assume without loss of generality that $t\geq t_n$. 

\noindent (\textbf{Step 1}). In this step, we prove that there exists some $c_1=c_1(\epsilon,\eta,\Lambda,\sigma_\ast)>0$ such that for $\eta \leq \sigma_\ast^{-1}\vee (3\Lambda^{2}+1)^{-1}$,
\begin{align}\label{ineq:single_index_cor_est_step1}
\abs{\Delta \hat{\tau}_\ast^{(t)} }\vee \abs{\Delta \hat{\delta}_\ast^{(t)} }\vee \hat{\Delta}_\ast^{(t)}&\leq c_1 \mathscr{B}^3(t_n)\cdot \Big(\max_{s \in [1:t_n]}  \hat{\Delta}_\ast^{(s-1)} +\abs{\E_{\pi_m} \xi_{\pi_m}}\Big)\equiv \beta_n.
\end{align}
First note that as $ \pnorm{\mathrm{u}_\ast^{(t)}}{}^2- (\alpha_\ast^{(t)}/\pnorm{\mu_\ast}{})^2 = \pnorm{\mathrm{u}_\ast^{(t)}}{}^2\cdot \sin^2(\mathrm{u}_\ast^{(t)},\mu_\ast) $, for any $t\leq t_\ast(\epsilon)$, 
\begin{align}\label{ineq:single_index_cor_est_1}
\abs{\Delta \hat{\tau}_\ast^{(t)} }\vee \abs{\Delta \hat{\delta}_\ast^{(t)} }&\leq c\Lambda\cdot \Big(\abs{\E_{\pi_m}\xi_{\pi_m}}+\abs{\hat{\alpha}_\ast^{(t)}-\alpha_\ast^{(t)} }/\pnorm{\mu_\ast}{}\nonumber\\
&\qquad + \bigabs{ \big\{ (\hat{\gamma}_\ast^{(t)})^2 - (\hat{\alpha}_\ast^{(t)}/\pnorm{\mu_\ast}{})^2\big\}^{1/2}-  \big\{ (\gamma_\ast^{(t)})^2- (\alpha_\ast^{(t)}/\pnorm{\mu_\ast}{})^2  \big\}^{1/2} }\Big)\nonumber\\
&\leq c_1\cdot \big( \hat{\Delta}_\ast^{(t)}+\abs{\E_{\pi_m}\xi_{\pi_m}}\big).
\end{align}
We now provide an estimate for $\abs{\Delta (\hat{\gamma}_\ast^{(t)})^2 }$. As
\begin{align}\label{ineq:single_index_cor_est_Delta_gamma}
\abs{\Delta (\hat{\gamma}_\ast^{(t)})^2 }&\leq (1-\eta \tau_\ast^{(t-1)})^2\cdot \abs{\Delta (\hat{\gamma}_\ast^{(t-1)})^2 }\nonumber\\
&\qquad +c_1 \cdot \big(\abs{\Delta \hat{\tau}_\ast^{(t-1)}}+\abs{\Delta \hat{\delta}_\ast^{(t-1)}}+\abs{\Delta \hat{\alpha}_\ast^{(t-1)}}\big),
\end{align}
by iterating the above inequality and using (\ref{ineq:single_index_cor_est_1}), for $\eta \leq \sigma_\ast^{-1}\vee (3\Lambda^{2}+1)^{-1}$,
\begin{align*}
\abs{\Delta (\hat{\gamma}_\ast^{(t)})^2 }&\leq  c_1 \cdot\sum_{s \in [1:t]} \bigg\{\prod_{r \in [s:t_n)} \prod_{r \in [t_n:t-1]}\bigg\}(1-\eta \tau_\ast^{(r)})^2\cdot  \big(\abs{\Delta \hat{\tau}_\ast^{(s-1)}} + \abs{\Delta \hat{\delta}_\ast^{(s-1)}}+\abs{\Delta \hat{\alpha}_\ast^{(s-1)}} \big)\nonumber\\
&\leq c_1\cdot (1-\eta\sigma_\ast)^{2(t-t_n)_+} \mathscr{B}^2(t_n)\cdot \max_{s \in [1:t_n)}  \hat{\Delta}_\ast^{(s-1)}\nonumber\\
&\qquad + c_1\cdot (1-\eta\sigma_\ast)^{2(t-t_n)_+} (t-t_n+1)_+\cdot \max_{s \in [t_n:t]} \hat{\Delta}_\ast^{(s-1)}\nonumber\\
&\qquad + c_1\cdot (1-\eta \sigma_\ast)^{2(t-t_n)_+} \big(\mathscr{B}^2(t_n)+(t-t_n+1)_+\big)\cdot  \abs{\E_{\pi_m} \xi_{\pi_m}}.
\end{align*}
This means 
\begin{align}\label{ineq:single_index_cor_est_Delta_gamma_1}
\abs{\Delta (\hat{\gamma}_\ast^{(t)})^2 }&\leq c_1\cdot  (1-\eta\sigma_\ast)^{2(t-t_n)_+}\cdot \Big[ \mathscr{B}^2(t_n) \max_{s \in [1:t_n)} \hat{\Delta}_\ast^{(s-1)}+(t-t_n+1)_+  \max_{s \in [t_n:t]} \hat{\Delta}_\ast^{(s-1)} \Big]\nonumber\\
&\qquad +c_1\cdot \big[1+ (1-\eta\sigma_\ast)^{(t-t_n)_+} \mathscr{B}(t_n)\big]^2\cdot \abs{\E_{\pi_m} \xi_{\pi_m}}.
\end{align}
We now provide an estimate for $\abs{\Delta \hat{\alpha}_\ast^{(t)} }$. As 
\begin{align}\label{ineq:single_index_cor_est_Delta_alpha}
\abs{\Delta \hat{\alpha}_\ast^{(t)}} &\leq \abs{1-\eta \tau_\ast^{(t-1)}}\cdot \abs{\Delta \hat{\alpha}_\ast^{(t-1)}}+ c_1 \cdot \big( \abs{\Delta \hat{\tau}_\ast^{(t-1)}} +\abs{ \Delta \hat{\delta}_\ast^{(t-1)}}+\abs{\Delta \hat{\gamma}_\ast^{(t)} }\big),
\end{align}
by iterating the above inequality and using the estimates in (\ref{ineq:single_index_cor_est_1}) and (\ref{ineq:single_index_cor_est_Delta_gamma_1}), for $\eta \leq \sigma_\ast^{-1}\vee (3\Lambda^{2}+1)^{-1}$,
\begin{align}\label{ineq:single_index_cor_est_Delta_alpha_1}
\abs{\Delta \hat{\alpha}_\ast^{(t)}}&\leq c_1 \sum_{s \in [1:t]}\bigg\{\prod_{r \in [s:t_n)} \prod_{r \in [t_n:t-1]}\bigg\} \,\abs{1-\eta \tau_\ast^{(r)}}\nonumber\\
&\qquad  \times  \big(\abs{\Delta \hat{\tau}_\ast^{(s-1)} } + \abs{\Delta \hat{\delta}_\ast^{(s-1)}}+\abs{\Delta \hat{\gamma}_\ast^{(s)} }\big)+ \prod_{r \in [0:t-1]}\abs{1-\eta \tau_\ast^{(r)}}\cdot \abs{\Delta \hat{\alpha}_\ast^{(0)}}\nonumber\\
&\leq c_1\cdot  (1-\eta \sigma_\ast)^{(t-t_n)_+} \cdot \Big[\mathscr{B}^3(t_n) \max_{s \in [1:t_n)} \hat{\Delta}_\ast^{(s-1)}  +(t-t_n+1)_+^2 \max_{s \in [t_n:t]}  \hat{\Delta}_\ast^{(s-1)}\Big]\nonumber\\
&\qquad + c_1\cdot \big[1+ (1-\eta\sigma_\ast)^{(t-t_n)_+} \mathscr{B}(t_n)\big]^3\cdot \abs{\E_{\pi_m} \xi_{\pi_m}}.
\end{align}
Combining (\ref{ineq:single_index_cor_est_Delta_gamma_1}) and (\ref{ineq:single_index_cor_est_Delta_alpha_1}), with $\mathfrak{p}(t) \equiv  (t+1)^2\cdot(1-\eta \sigma_\ast)^{t}$, for $t\leq  t_\ast(\epsilon)$, 
\begin{align}\label{ineq:single_index_cor_est_Delta_alpha_4}
\hat{\Delta}_\ast^{(t)}&\leq c_1\cdot \mathfrak{p}\big((t-t_n)_+\big)\cdot \Big(\max_{s \in [t_n:t]}  \hat{\Delta}_\ast^{(s-1)}+\beta_n\Big).
\end{align}
Consequently, for some $c_2=c_2(\epsilon,\eta,\Lambda,\sigma_\ast)>0$ and $\delta_0=\delta_0(\epsilon,\eta,\Lambda,\sigma_\ast)>0$, we may simplify the above recursive estimate (\ref{ineq:single_index_cor_est_Delta_alpha_4}) as
\begin{align*}
\hat{\Delta}_\ast^{(t)}&\leq \Big(c_2\bm{1}_{t\leq t_n+c_2}+(1-\delta_0)^{(t-t_n-c_2)_+}\bm{1}_{t>t_n+c_2}\Big)\cdot \Big(\max_{s \in [t_n:t]}  \hat{\Delta}_\ast^{(s-1)}+\beta_n\Big).
\end{align*}
In particular, we have
\begin{align}\label{ineq:single_index_cor_est_Delta_alpha_5}
\begin{cases}
\big(\hat{\Delta}_\ast^{(t)}+\beta_n \big)\leq c_2\cdot \big(\max_{s \in [t_n:t]} \hat{\Delta}_\ast^{(s-1)}+  \beta_n\big), & t_n\leq t\leq t_n+c_2;\\
\hat{\Delta}_\ast^{(t)}\leq \max_{s \in [t_n:t]} \hat{\Delta}_\ast^{(s-1)}+ (1-\delta_0)^{(t-t_n-c_2)_+}\cdot \beta_n, & t>t_n+c_2.
\end{cases}
\end{align}
Now iterating the above estimate (\ref{ineq:single_index_cor_est_Delta_alpha_5}) until $t=t_n$, we obtain (\ref{ineq:single_index_cor_est_step1}).

\noindent (\textbf{Step 2}). In this step, we prove the desired claims via (\ref{ineq:single_index_cor_est_Delta_alpha_5}) and apriori bounds. Using (\ref{ineq:single_index_cor_est_1}), (\ref{ineq:single_index_cor_est_Delta_gamma}) and trivial bounds for the first terms on the right hand side of the first inequality of (\ref{ineq:single_index_cor_est_Delta_alpha_1}), for $\eta \leq \sigma_\ast^{-1}\vee (3\Lambda^{2}+1)^{-1}$ and $t\leq  t_\ast(\epsilon)$,
\begin{align}\label{ineq:single_index_cor_est_step2_1}
\abs{\Delta \hat{\alpha}_\ast^{(t)}} &\leq c_1 \sum_{s \in [1:t]} 2^{t-s}\cdot \big(\abs{\Delta \hat{\tau}_\ast^{(s-1)}}+\abs{\Delta \hat{\delta}_\ast^{(s-1)}}+\abs{\Delta \hat{\gamma}_\ast^{(s)}}\big) + 2^t\cdot \abs{\Delta \hat{\alpha}_\ast^{(0)}} \nonumber\\
&\leq c_1\sum_{r \in [0:t-1]} 2^{t-r} \hat{\Delta}_\ast^{(r)} + c_12^t\cdot \big(\abs{\Delta\hat{\alpha}_\ast^{(0)}}+\abs{\E_{\pi_m} \xi_{\pi_m}}\big).
\end{align}
Using (\ref{ineq:single_index_cor_est_1}), (\ref{ineq:single_index_cor_est_Delta_gamma}) again, similarly for $\eta \leq \sigma_\ast^{-1}\vee (3\Lambda^{2}+1)^{-1}$ and $t\leq  t_\ast(\epsilon)$,
\begin{align}\label{ineq:single_index_cor_est_step2_2}
\abs{\Delta \hat{\gamma}_\ast^{(t)}} &\leq c_1 \sum_{s \in [1:t]} 2^{t-s}\cdot \big(\abs{\Delta \hat{\tau}_\ast^{(s-1)}}+\abs{\Delta \hat{\delta}_\ast^{(s-1)}}+\abs{\Delta \hat{\alpha}_\ast^{(s-1)}}\big)\nonumber\\
&\leq c_1\sum_{r \in [0:t-1]} 2^{t-r} \hat{\Delta}_\ast^{(r)} + c_12^t\cdot \big(\abs{\Delta\hat{\alpha}_\ast^{(0)}}+\abs{\E_{\pi_m} \xi_{\pi_m}}\big).
\end{align}
Combining (\ref{ineq:single_index_cor_est_step2_1})-(\ref{ineq:single_index_cor_est_step2_2}), for $\eta \leq \sigma_\ast^{-1}\vee (3\Lambda^{2}+1)^{-1}$ and $t\leq  t_\ast(\epsilon)$,
\begin{align*}
\hat{\Delta}_\ast^{(t)} 
&\leq c_1\sum_{r \in [0:t-1]} 2^{t-r} \hat{\Delta}_\ast^{(r)} + c_12^t\cdot \big(\abs{\Delta\hat{\alpha}_\ast^{(0)}}+\abs{\E_{\pi_m} \xi_{\pi_m}}\big).
\end{align*}
Consequently, for $s \in [0:t]$,
\begin{align*}
2^{-s}\hat{\Delta}_\ast^{(s)}  \leq c_1 \sum_{r \in [0:s-1]} 2^{-r} \hat{\Delta}_\ast^{(r)}  + c_1\cdot \big(\abs{\Delta\hat{\alpha}_\ast^{(0)}}+\abs{\E_{\pi_m} \xi_{\pi_m}}\big).
\end{align*}
Let $\{\omega_s>0: s \in [0:t]\}$ to be determined later on. Averaging the above display with the weights $\{\omega_s\}$ and interchanging the order the summation, we have
\begin{align*}
\sum_{s \in [0:t]} \omega_s 2^{-s}\hat{\Delta}_\ast^{(s)}
&\leq c_1\cdot \sum_{s\in [0:t-1]} \bigg(\sum_{r \in (s:t]} \omega_r\bigg)\cdot  2^{-s} \hat{\Delta}_\ast^{(s)}\\
&\qquad + c_1\cdot \sum_{s \in [0:t]} \omega_s \cdot \big(\abs{\Delta\hat{\alpha}_\ast^{(0)}}+\abs{\E_{\pi_m} \xi_{\pi_m}}\big).
\end{align*}
So, by choosing $\{\omega_s\}\subset \R_{>0}$ inductively by the relation $\omega_s\equiv \sum_{r \in (s:t]} \omega_r$ and the normalization $\sum_{r \in [0:t]} \omega_r=1$, for $t\leq  t_\ast(\epsilon)$, we have
\begin{align*}
\sum_{s \in [0:t]} \omega_s 2^{-s}\hat{\Delta}_\ast^{(s)}&\leq c_1 \sum_{s \in [0:t-1]} \omega_s 2^{-s}\hat{\Delta}_\ast^{(s)}+ c_1\cdot \big(\abs{\Delta\hat{\alpha}_\ast^{(0)}}+\abs{\E_{\pi_m} \xi_{\pi_m}}\big).
\end{align*}
Iterating the above inequality, we then arrive at
\begin{align*}
\sum_{s \in [0:t]} \omega_s 2^{-s}\hat{\Delta}_\ast^{(s)}\leq c_1^t\cdot  \big(\abs{\Delta\hat{\alpha}_\ast^{(0)}}+\abs{\E_{\pi_m} \xi_{\pi_m}}\big).
\end{align*}
The desired apriori estimate for $\hat{\Delta}_\ast^{(t)}$ follows by noting that $\omega_t\asymp 2^{-t}$. \qed

\section{Proof of Theorem \ref{thm:mean_field_dyn_approx}}\label{section:proof_mean_field_dyn_approx}

For notational convenience, we shall write $\bm{\delta}^{[t]}\equiv (\delta_s)_{s \in [t]} \in \R^t$, $\bm{\tau}^{[t]}\equiv (\tau_{r,s})_{r,s \in [t]}\in \R^{t\times t}$, $\bm{\rho}^{[t]}\equiv (\rho_{r,s})_{r,s \in [t]}\in \R^{t\times t}$,  $\Sigma_{\mathfrak{Z}}^{[t]} \in \R^{[0:t]\times [0:t]}$ and $\Sigma_{\mathfrak{W}}^{[t]} \in \R^{[1:t]\times [1:t]}$ for the covariance of $\mathfrak{Z}^{([0:t])}$ and $\mathfrak{W}^{([1:t])}$.

We first prove the following apriori estimates.
\begin{lemma}\label{lem:rho_tau_bound}
	Suppose (A2) holds for some $\Lambda\geq 2$ and $\mathfrak{p}\geq 1$, and the step sizes $\{\eta_t\}$ are bounded by $\Lambda$. The following hold for some $c_t=c_t(t)>1$:
	\begin{enumerate}
		\item $ \pnorm{\bm{\tau}^{[t]}}{\op}+ \pnorm{\bm{\rho}^{[t]}}{\op}\leq \big(\Lambda(1\vee \phi^{-1})\big)^{c_t }$.
		\item $\pnorm{\Sigma_{\mathfrak{Z}}^{[t]}}{\op}+ \phi\cdot\pnorm{\Sigma_{\mathfrak{W}}^{[t]}}{\op}+ \pnorm{\bm{\delta}^{[t]}}{}\leq \big(\Lambda L_\ast^{(0)}(1\vee \phi^{-1})\big)^{c_t }$.
		\item $\max\limits_{k \in [m], r \in [1:t]} \big(\abs{\Upsilon_{r;k}(z^{([0:r])})}+\abs{\Theta_{r;k}(z^{([0:r])})}\big)\leq \big(\Lambda (1\vee \phi^{-1})\big)^{c_t }\cdot \big(1+\pnorm{z^{([0:t])} }{}\big)$.
		\item $\max\limits_{\ell \in [n], r \in [1:t]} \abs{\Omega_{r;\ell}(w^{([1:r])})}\leq \big(\Lambda L_\ast^{(0)}(1\vee \phi^{-1})\big)^{c_t }\cdot \big(1+\pnorm{w^{([1:t])}}{}\big)$.
		\item $\max\limits_{\substack{k \in [m], \ell \in [n], r,s \in [1:t]}} \big(\abs{\partial_{(s)}\Upsilon_{r;k}(z^{([0:r])})}+\abs{\partial_{(s)}\Omega_{r;\ell}(w^{([1:r])})}\big)\leq \big(\Lambda (1\vee \phi^{-1})\big)^{c_t }$.
	\end{enumerate}
	Here $\partial_{(s)}\Upsilon_{r;k}(z^{([0:r])})\equiv \partial_{z^{(s)}}\Upsilon_{r;k}(z^{([0:r])})$, and $\partial_{(s)}\Omega_{r;\ell}(w^{([1:r])})\equiv \partial_{w^{(s)}}\Omega_{r;\ell}(w^{([1:r])})$. 
\end{lemma}
\begin{proof}
	The proof modifies that of \cite[Lemma 7.1]{han2024gradient} by carefully tracking the impact of the aspect ratio $\phi$. To ease the reading burden for cross reference, we provide below details for these arguments. Let
	\begin{align*}
	\Theta_{t}(z^{([0:t])})\equiv \mathfrak{z}^{(t)} - \frac{1}{\phi}\sum_{s \in [1:t-1]}\eta_{s-1}\rho_{t-1,s}\Upsilon_s(\mathfrak{z}^{([0:s])}).
	\end{align*}
	For a general $n\times n$ matrix $M$, let
	\begin{align*}
	\mathfrak{O}_{n+1}(M)\equiv 
	\begin{pmatrix}
	0_{1\times n} & 0\\
	M & 0_{n\times 1}
	\end{pmatrix} \in \R^{(n+1)\times (n+1)}.
	\end{align*}
	For notational consistency, we write $\mathfrak{O}_1(\emptyset)=0$. Moreover, for notational simplicity, we work with the constant step size $\eta_t\equiv \eta$.
	
	\noindent (\textbf{Step 1}). We prove the estimate in (1) in this step. First, for any $k\in [m]$ and $1\leq s\leq t$, using (S1),
	\begin{align*}
	\partial_{(s)}\Upsilon_{t;k}(z^{([0:t])})&\equiv  \partial_{11} \mathsf{L}\big(\Theta_{t;k}(z^{([0:t])}) ,\mathcal{F}(z^{(0)},\xi_k)\big)\\
	&\qquad \times \bigg(\bm{1}_{t=s}-  \frac{\eta}{\phi}\sum_{r \in [1:t-1]} \rho_{t-1,r} \partial_{(s)} \Upsilon_{r;k}(z^{([0:r])})\bigg).
	\end{align*}
	In the matrix form, with $\bm{\Upsilon}_{k}^{';[t]}(z^{([0:t])}) \equiv \big(\partial_{(s)}\Upsilon_{r;k}(z^{([0:r])})\big)_{r,s \in [1:t]}$ and $\bm{L}^{[t]}_{k}(z^{([0:t])})\equiv \mathrm{diag}\big(\big\{  \partial_{11} \mathsf{L}\big(\Theta_{s;k}(z^{([0:s])}) ,\mathcal{F}(z^{(0)},\xi_k)\big)\big\}_{s \in [1:t]}\big)$, 
	\begin{align*}
	\bm{\Upsilon}_{k}^{';[t]}(z^{([0:t])}) =  \bm{L}^{[t]}_{k}(z^{([0:t])}) -\frac{\eta}{\phi}\cdot  \bm{L}^{[t]}_{k}(z^{([0:t])}) \mathfrak{O}_{t}(\bm{\rho}^{[t-1]}) \bm{\Upsilon}_{k}^{';[t]}(z^{([0:t])}).
	\end{align*}
	Solving for $\bm{\Upsilon}_{k}^{';[t]}$ yields that
	\begin{align}\label{ineq:rho_tau_bound_Upsilon}
	\bm{\Upsilon}_{k}^{';[t]} (z^{([0:t])}) = \Big[I_t+\frac{\eta}{\phi}\cdot \bm{L}^{[t]}_{k}(z^{([0:t])}) \mathfrak{O}_{t}(\bm{\rho}^{[t-1]})\Big]^{-1} \bm{L}^{[t]}_{k}(z^{([0:t])}).
	\end{align}
	As $\bm{L}^{[t]}_{k}(z^{([0:t])}) \mathfrak{O}_{t}(\bm{\rho}^{[t-1]})$ is a lower triangular matrix with $0$ diagonal elements, $\big(\bm{L}^{[t]}_{k}(z^{([0:t])}) \mathfrak{O}_{t}(\bm{\rho}^{[t-1]})\big)^t=0_{t\times t}$, and therefore using $\pnorm{\bm{L}^{[t]}_{k}(z^{([0:t])})}{\op}\leq \Lambda^{c}$,
	\begin{align}\label{ineq:rho_tau_bound_Upsilon_1}
	\pnorm{\bm{\Upsilon}_{k}^{';[t]} (z^{([0:t])}) }{\op}&\leq \bigg(1+\sum_{r \in [1:t]} \phi^{-r}\Lambda^{cr} \pnorm{ \bm{\rho}^{[t-1]}}{\op}^r\bigg)\cdot \Lambda^c\nonumber\\
	&\leq (1\vee \phi^{-1})^t\cdot \Lambda^{c t}\cdot   \pnorm{ \bm{\rho}^{[t-1]}}{\op}^t.
	\end{align}
	Using definition of $\{\tau_{r,s}\}$, we then arrive at
	\begin{align}\label{ineq:rho_tau_bound_1}
	\pnorm{\bm{\tau}^{[t]}}{\op}\leq (1\vee \phi^{-1})^t\cdot \Lambda^{c t}\cdot   \pnorm{ \bm{\rho}^{[t-1]}}{\op}^t.
	\end{align}
	Next, for any $\ell \in [n]$ and $1\leq s\leq t$, using (S3), 
	\begin{align*}
	\partial_{(s)}\Omega_{t;\ell}(w^{([1:t])})&=\bm{1}_{t=s}+ \sum_{r \in [1:t]} (\bm{1}_{t=r}-\eta\cdot \tau_{t,r})\cdot  \partial_{(s)}\Omega_{r-1;\ell}(w^{([1:r-1])}).
	\end{align*}
	In the matrix form, with $\bm{\Omega}_{\ell}^{';[t]}(w^{([1:t])})\equiv \big(\partial_{(s)}\Omega_{r;\ell}(w^{([1:r])})\big)_{r,s \in [1:t]}$,
	\begin{align*}
	\bm{\Omega}_{\ell}^{';[t]}(w^{([1:t])}) = I_t+(I_t-\eta\cdot \bm{\tau}^{[t]})\mathfrak{O}_t\big(\bm{\Omega}_{\ell}^{';[t-1]}(w^{([1:t-1])}) \big).
	\end{align*}
	Consequently, 
	\begin{align*}
	\pnorm{\bm{\Omega}_{\ell}^{';[t]}(w^{([1:t])})}{\op}\leq 1+\Lambda\cdot (1+ \pnorm{\bm{\tau}^{[t]}}{\op})\cdot  \pnorm{\bm{\Omega}_{\ell}^{';[t-1]}(w^{([1:t-1])})}{\op}.
	\end{align*}
	Iterating the bound and using the trivial initial condition $\pnorm{\bm{\Omega}_{\ell}^{';(1)}(w^{(1)})}{\op}\leq \Lambda^c$, 
	\begin{align}\label{ineq:rho_tau_bound_Omega_1}
	\pnorm{\bm{\Omega}_{\ell}^{';[t]}(w^{([1:t])})}{\op}\leq \big(\Lambda (1+ \pnorm{\bm{\tau}^{[t]}}{\op})\big)^{c_t}.
	\end{align}
	Using the definition of $\{\rho_{r,s}\}$, we then have
	\begin{align}\label{ineq:rho_tau_bound_2}
	\pnorm{ \bm{\rho}^{[t]}}{\op}\leq \big(\Lambda (1+ \pnorm{\bm{\tau}^{[t]}}{\op})\big)^{c_t}.
	\end{align}
	Combining (\ref{ineq:rho_tau_bound_1}) and (\ref{ineq:rho_tau_bound_2}), it follows that 
	\begin{align*}
	\pnorm{\bm{\tau}^{[t]}}{\op}\leq \big(\Lambda(1\vee \phi^{-1})\big)^{c_t }\cdot   (1+ \pnorm{\bm{\tau}^{[t-1]}}{\op})^{c_t}.
	\end{align*}
	Iterating the bound and using the initial condition $\pnorm{\bm{\tau}^{[1]}}{\op}\leq \Lambda^c$ to conclude the bound for $\pnorm{\bm{\tau}^{[t]}}{\op}$. The bound for $\pnorm{\bm{\rho}^{[t]}}{\op}$ then follows from (\ref{ineq:rho_tau_bound_2}).

	\noindent (\textbf{Step 2}). In this step we note the following recursive estimates:
	\begin{enumerate}
		\item[(a)] A direct induction argument for (S1) shows that 
		\begin{align*}
		\max_{k \in [m]} \max_{r \in [1:t]} \abs{\Upsilon_{r;k}(z^{([0:r])})}\leq \big(\Lambda(1\vee \phi^{-1})\big)^{c_t }\cdot \big(1+\pnorm{z^{([0:t])} }{}\big).
		\end{align*}
		\item[(b)]  A direct induction argument for (S3) shows that 
		\begin{align*}
		\max_{\ell \in [n]} \max_{r \in [1:t]} \abs{\Omega_{r;\ell}(w^{([1:r])})}&\leq \big(\Lambda L_\ast^{(0)}(1\vee \phi^{-1})\big)^{c_t }\cdot \big(1+\pnorm{w^{([1:t])} }{}+ \pnorm{\bm{\delta}^{[t]}}{} \big).
		\end{align*}
		\item[(c)]  Using (S2), we have
		\begin{align*}
		\pnorm{\Sigma_{\mathfrak{Z}}^{[t]}}{\op}&\leq \big(\Lambda L_\ast^{(0)}(1\vee \phi^{-1})\big)^{c_t }\cdot \big(1+\pnorm{\Sigma_{\mathfrak{W}}^{[t-1]}}{\op}+ \pnorm{\bm{\delta}^{[t-1]}}{} \big),\\
		\phi\cdot \pnorm{\Sigma_{\mathfrak{W}}^{[t]}}{\op} &\leq \big(\Lambda(1\vee \phi^{-1})\big)^{c_t }\cdot \big(1+\pnorm{\Sigma_{\mathfrak{Z}}^{[t]}}{\op} \big).
		\end{align*}
	\end{enumerate}
	
	\noindent (\textbf{Step 3}). In order to use the recursive estimates in Step 2, in this step we prove the estimate for $\pnorm{\bm{\delta}^{[t]}}{}$. For $k \in [m]$, let $\mathsf{L}_{k}^{\mathcal{F}}(u_1,u_2)\equiv \mathsf{L}(u_1,\mathcal{F}(u_2,\xi_k))$. Then by assumption, the mapping $(u_1,u_2)\mapsto \partial_1 \mathsf{L}_{k}^{\mathcal{F}}(u_1,u_2)$ is $\Lambda^c$-Lipschitz on $\R^2$. By using (S1), we then obtain the estimate
	\begin{align*}
	\abs{\partial_{(0)}\Upsilon_{t;k}(z^{([0:t])})}&= \biggabs{  \partial_{11} \mathsf{L}_k^{\mathcal{F}}\big(\Theta_{t;k}(z^{([0:t])}) ,z^{(0)})\big) \cdot \bigg(-\frac{\eta}{\phi} \sum_{r \in [1:t-1]}\rho_{t-1,r} \partial_{(0)} \Upsilon_{r;k}(z^{([0:r])})\bigg)\\
		&\qquad +\partial_{12} \mathsf{L}_{k}^{\mathcal{F}}\big(\Theta_{t;k}(z^{([0:t])}) ,z^{(0)})\big)}\\
	&\leq \big(\Lambda(1\vee \phi^{-1})\big)^{c_t }\cdot \pnorm{\bm{\rho}^{[t-1]} }{\op}\cdot \Big(1+\max_{r \in [1:t-1]} \abs{\partial_{(0)} \Upsilon_{r;k}(z^{([0:r])}) }\Big).
	\end{align*}
	Invoking the proven estimate in (1) and iterating the above bound with the trivial initial condition $\abs{\partial_{(0)}\Upsilon_{1;k}(z^{([0:1])})}\leq \Lambda^c$ to conclude that $
	\abs{\partial_{(0)}\Upsilon_{t;k}(z^{([0:t])})}\leq \big(\Lambda(1\vee \phi^{-1})\big)^{c_t }$, and therefore by definition of $\delta_t$, we conclude that
	\begin{align}\label{ineq:rho_tau_bound_3}
	\pnorm{\bm{\delta}^{[t]}}{}\leq \big(\Lambda(1\vee \phi^{-1})\big)^{c_t }.
	\end{align}
	
	\noindent (\textbf{Step 4}). Now we shall use the estimate for $\pnorm{\bm{\delta}^{[t]}}{}$ in  Step 3 to run the recursive estimates in Step 2. Combining the first line of (c) and (\ref{ineq:rho_tau_bound_3}), we have
	\begin{align*}
	\pnorm{\Sigma_{\mathfrak{Z}}^{[t]}}{\op}&\leq \big(\Lambda L_\ast^{(0)}(1\vee \phi^{-1})\big)^{c_t }\cdot \big(1+\pnorm{\Sigma_{\mathfrak{W}}^{[t-1]}}{\op} \big).
	\end{align*}
	Combined with the second line of (c), we obtain
	\begin{align*}
	\pnorm{\Sigma_{\mathfrak{Z}}^{[t]}}{\op}&\leq \big(\Lambda L_\ast^{(0)}(1\vee \phi^{-1})\big)^{c_t }\cdot \Big(1+\max_{r\in [1:t-1]}\pnorm{\Sigma_{\mathfrak{Z}}^{[r]}}{\op}\Big).
	\end{align*}
	Coupled with the initial condition $\pnorm{\Sigma_{\mathfrak{Z}}^{[1]}}{\op}\leq (L_\ast^{(0)})^c$, we arrive at the estimate
	\begin{align*}
	\pnorm{\Sigma_{\mathfrak{Z}}^{[t]}}{\op}+ \phi\cdot \pnorm{\Sigma_{\mathfrak{W}}^{[t]}}{\op}+ \pnorm{\bm{\delta}^{[t]}}{}\leq \big(\Lambda L_\ast^{(0)}(1\vee \phi^{-1})\big)^{c_t }.
	\end{align*}
	The proof of (1)-(4) is complete by collecting the estimates. The estimate in (5) follows by combining (\ref{ineq:rho_tau_bound_Upsilon_1}) and (\ref{ineq:rho_tau_bound_Omega_1}) along with the estimate in (1).
\end{proof}

\begin{proof}[Proof of Theorem \ref{thm:mean_field_dyn_approx}]
	The proof is divided in several steps.
	
	\noindent (\textbf{Step 1}). Let us now define a revised state evolution for formal comparison. In particular, initialize with (i) two vectors $\tilde{\Omega}_{-1}\equiv n^{1/2}\mu_\ast \in \R^n$ and $\tilde{\Omega}_0 \equiv n^{1/2}\mu^{(0)}\in \R^n$, and (ii) a Gaussian random variable $\tilde{\mathfrak{Z}}^{(0)}\sim \mathcal{N}(0,\pnorm{\mu_\ast}{}^2)$. For $t=1,2,\ldots$, we execute the following steps:
	\begin{enumerate}
		\item[(LS1)] Let $\tilde{\Upsilon}_t: \mathbb{R}^{m\times [0:t]}\to \R^m$ be defined as follows: 
		\begin{align*}
		\tilde{\Upsilon}_t(\mathfrak{z}^{([0:t])})\equiv  \partial_1 \mathsf{L}\big(\mathfrak{z}^{(t)} ,\mathcal{F}(\mathfrak{z}^{(0)},\xi)\big)\in \R^m.
		\end{align*}
		\item[(LS2)] Let $\tilde{\mathfrak{Z}}^{([0:t])}\in \R^{[0:t]}$ be a centered Gaussian random vector whose law at iterate $t$ is determined via the correlation specification: 
		\begin{align*}
		\cov(\tilde{\mathfrak{Z}}^{(t)},\tilde{\mathfrak{Z}}^{(s)})
		& \equiv \E^{(0)}  \tilde{\Omega}_{s-1;\pi_n} \tilde{\Omega}_{t-1;\pi_n},\quad s \in [0:t].
		\end{align*}
		\item[(LS3)] Let the vector $\tilde{\Omega}_t \in \R^n$ be defined as follows:
		\begin{align*}
		\tilde{\Omega}_t\equiv  (1-\eta\tilde{\tau}_{t,t})\cdot  \tilde{\Omega}_{t-1} + \eta\tilde{\delta}_t\cdot (n^{1/2}\mu_\ast).
		\end{align*}
		Here the coefficients are defined via
		\begin{align*}
		\tilde{\tau}_{t,s} &\equiv  \E^{(0)}\partial_{\tilde{\mathfrak{Z}}^{(t)}} \tilde{\Upsilon}_{t;\pi_m}(\tilde{\mathfrak{Z}}^{([0:t])})\in \R,\quad s \in [1:t];\\
		\tilde{\delta}_t &\equiv - \E^{(0)}\partial_{\tilde{\mathfrak{Z}}^{(0)}} \tilde{\Upsilon}_{t;\pi_m}(\tilde{\mathfrak{Z}}^{([0:t])})\in \R.
		\end{align*}
	\end{enumerate}
	Because $\tilde{\tau}_{t,s}=0$ for $s\neq t$ as the function $\tilde{\Upsilon}_t$ depends on $\mathfrak{z}^{([0:t])}$ only via $\big(\mathfrak{z}^{(0)},\mathfrak{z}^{(t)}\big)$, the matrix $\tilde{\bm{\tau}}^{[t]}$ is diagonal with elements $\{\tilde{\tau}_{s,s}\}_{s \in [1:t]}$. Compared to Definition \ref{def:state_evolution}, we have the following identification:
	\begin{align*}
	n^{-1/2}\tilde{\Omega}_t=\mathrm{u}_\ast^{(t)},\quad \tilde{\tau}_{t,t} = \tau_\ast^{(t-1)},\quad \tilde{\delta}_t = \delta_\ast^{(t-1)}.
	\end{align*}
	It is easy to prove the same apriori estimates
	\begin{align}\label{ineq:gd_se_lowd_1}
	\max_{k \in [m]} \frac{\abs{\tilde{\Upsilon}_{t;k}(z^{([0:t])})} }{1+ \pnorm{z^{([0:t])}}{} } + \pnorm{\tilde{\Omega}_t}{\infty}+ \pnorm{\tilde{\bm{\tau}}^{[t]}}{\op} + \pnorm{\tilde{\Sigma}_{\mathfrak{Z}}^{[t]}}{\op}+ \pnorm{\tilde{\bm{\delta}}^{[t]}}{}  \leq \big(\Lambda L_\ast^{(0)}\big)^{c_t}.
	\end{align}

	\noindent (\textbf{Step 2}). For notational convenience, we shall write $\tilde{\d{}}\# = \#-\tilde{\#}$. For instance, $\tilde{\d{}}\Upsilon = \Upsilon-\tilde{\Upsilon}$. In this step we prove that
	\begin{align}\label{ineq:gd_se_lowd_step2}
	&\pnorm{ \tilde{\d{}} \bm{\tau}^{[t]} }{\op}+\pnorm{\tilde{\d{}}\bm{\delta}^{[t]}}{} +\pnorm{ \tilde{\d{}} \Sigma_{\mathfrak{Z}}^{[t]} }{\op}\nonumber\\
	&\qquad + \max_{\ell \in [n]} \frac{\abs{\tilde{\d{}} \Omega_{t;\ell}(w^{([1:t])})}}{ 1+\pnorm{ w^{([1:t])}}{} }\leq \big(\Lambda L_\ast^{(0)} \pnorm{\mu_\ast}{}^{-1}\big)^{c_t} \cdot \phi^{-1/c_t}.
	\end{align}
	Using the representation in (\ref{ineq:rho_tau_bound_Upsilon}) and notation used therein, and the estimates in Lemma \ref{lem:rho_tau_bound},
	\begin{align*}
	\bigpnorm{\bm{\Upsilon}_{k}^{';[t]} (z^{([0:t])}) - {\bm{L}}^{[t]}_{k}(z^{([0:t])})}{\op}&\leq \phi^{-1/c}\cdot \Lambda^{c_t}.
	\end{align*}
	Using the definition of $\{\tau_{r,s}\}$, we then conclude that
	\begin{align*}
	\bigpnorm{\bm{\tau}^{[t]} - \E^{(0)} {\bm{L}}^{[t]}_{k}(\mathfrak{Z}^{([0:t])})}{\op}&\leq \phi^{-1/c}\cdot \Lambda^{c_t}.
	\end{align*}
	This means
	\begin{align}\label{ineq:gd_se_lowd_diff_tau}
	\pnorm{ \tilde{\d{}} \bm{\tau}^{[t]} }{\op}\leq  \Lambda^{c_t}\cdot \big(\phi^{-1/c}+ \pnorm{ \tilde{\d{}} \Sigma_{\mathfrak{Z}}^{[t]} }{\op}\big).
	\end{align}
	Next, using Gaussian integration-by-parts, we have
	\begin{align}\label{def:delta_t_alternative}
	\delta_t \equiv \frac{1}{ \pnorm{\mu_\ast}{}^2}\bigg(\E^{(0)} \mathfrak{Z}^{(0)} \Upsilon_{t;\pi_m}(\mathfrak{Z}^{([0:t])})-\sum_{s \in [1:t]} \tau_{t,s}\cov(\mathfrak{Z}^{(0)},\mathfrak{Z}^{(s)}) \bigg),
	\end{align}
	and a similar representation holds for $\tilde{\delta}_t$. Here for $\mu_\ast=0$, the right hand side is interpreted as the limit as $\pnorm{\mu_\ast}{}\to 0$ whenever well-defined. Now combining the above display (\ref{def:delta_t_alternative}) and the proven estimate (\ref{ineq:gd_se_lowd_diff_tau}), along with the apriori estimates in Lemma \ref{lem:rho_tau_bound} and (\ref{ineq:gd_se_lowd_1}), we have
	\begin{align}\label{ineq:gd_se_lowd_diff_delta}
	\pnorm{\tilde{\d{}}\bm{\delta}^{[t]}}{} &\leq \big(\Lambda L_\ast^{(0)} \pnorm{\mu_\ast}{}^{-1}\big)^{c_t}\cdot \big( \pnorm{ \tilde{\d{}} \Sigma_{\mathfrak{Z}}^{[t]} }{\op}^{1/2}+\pnorm{ \tilde{\d{}} \bm{\tau}^{[t]} }{\op}\big)\nonumber\\
	&\leq \big(\Lambda L_\ast^{(0)} \pnorm{\mu_\ast}{}^{-1}\big)^{c_t}\cdot \big(\phi^{-1/c}+ \pnorm{ \tilde{\d{}} \Sigma_{\mathfrak{Z}}^{[t]} }{\op}^{1/2}\big).
	\end{align}
	Consequently, combining (\ref{ineq:gd_se_lowd_diff_tau}), (\ref{ineq:gd_se_lowd_diff_delta}) and Lemma \ref{lem:rho_tau_bound},
	\begin{align*}
	&\abs{\tilde{\d{}} \Omega_{t;\ell}(w^{([1:t])})}\\
	&\leq \big(\Lambda L_\ast^{(0)}\big)^{c_t}\cdot \Big[\abs{\tilde{\d{}} \Omega_{t-1;\ell}(w^{([1:t-1])})}+\abs{w^{(t)}}+ \pnorm{ \tilde{\d{}} \bm{\tau}^{[t]} }{\op} \big(1+\pnorm{ w^{([1:t])}}{}\big) + \pnorm{\tilde{\d{}}\bm{\delta}^{[t]}}{} \Big]\\
	&\leq \big(\Lambda L_\ast^{(0)} \pnorm{\mu_\ast}{}^{-1}\big)^{c_t}\cdot \Big[\abs{\tilde{\d{}} \Omega_{t-1;\ell}(w^{([1:t-1])})} + \abs{w^{(t)}} +\big(\phi^{-1/c}+ \pnorm{ \tilde{\d{}} \Sigma_{\mathfrak{Z}}^{[t]} }{\op}^{1/2}\big)\cdot \big(1+\pnorm{ w^{([1:t])}}{}\big) \Big].
	\end{align*}
	Iterating the bound, we obtain 
	\begin{align}\label{ineq:gd_se_lowd_diff_Omega}
	\abs{\tilde{\d{}} \Omega_{t;\ell}(w^{([1:t])})}&\leq \big(\Lambda L_\ast^{(0)} \pnorm{\mu_\ast}{}^{-1}\big)^{c_t}\cdot \Big[ \pnorm{w^{([1:t])}}{} \nonumber\\
	&\qquad +\big(\phi^{-1/c}+ \pnorm{ \tilde{\d{}} \Sigma_{\mathfrak{Z}}^{[t]} }{\op}^{1/2}\big)\cdot \big(1+\pnorm{ w^{([1:t])}}{}\big) \Big].
	\end{align}
	Combining the above display (\ref{ineq:gd_se_lowd_diff_Omega}) and the definitions for $\Sigma_{\mathfrak{Z}}^{[t]},\tilde{\Sigma}_{\mathfrak{Z}}^{[t]} $, by using the apriori estimates in Lemma \ref{lem:rho_tau_bound} and in (\ref{ineq:gd_se_lowd_1}),
	\begin{align*}
	\pnorm{ \tilde{\d{}} \Sigma_{\mathfrak{Z}}^{[t]} }{\op}&\leq \big(\Lambda L_\ast^{(0)} \pnorm{\mu_\ast}{}^{-1}\big)^{c_t}\cdot \big(\phi^{-1}+ \pnorm{ \tilde{\d{}} \Sigma_{\mathfrak{Z}}^{[t-1]} }{\op}\big)^{1/c_0}.
	\end{align*}
	Iterating the bound, we have
	\begin{align*}
	\pnorm{ \tilde{\d{}} \Sigma_{\mathfrak{Z}}^{[t]} }{\op}\leq \big(\Lambda L_\ast^{(0)} \pnorm{\mu_\ast}{}^{-1}\big)^{c_t}\cdot \phi^{-1/c_t}.
	\end{align*}
	The claim in (\ref{ineq:gd_se_lowd_step2}) is proven by combining the above display with (\ref{ineq:gd_se_lowd_diff_tau}), (\ref{ineq:gd_se_lowd_diff_delta}) and (\ref{ineq:gd_se_lowd_diff_Omega}).

	\noindent (\textbf{Step 3}). The claimed estimate follows by (\ref{ineq:gd_se_lowd_step2}) and the apriori estimate for $\pnorm{\Sigma_{\mathfrak{W}}^{[t]}}{\op}$ in Lemma \ref{lem:rho_tau_bound}.
\end{proof}

\appendix

\section{A further example}

The main body of this work has focused on applications of the meta Theorems \ref{thm:grad_des_se} and \ref{thm:grad_des_incoh}, in which algorithmic convergence to $\mu_\ast$ is expected (up to sign). We now give an example showing that the iterates $\mu^{(t)}$ need not converge to $\pm \mu_\ast$, yet still fall within the scope of these meta theorems.

This phenomenon arises naturally when the loss function is misaligned with the statistical model, and it occurs already in convex settings. Below we consider a simple case in which the loss function $\mathsf{L}:\R^2\to\R$ is chosen so that the maps $x\mapsto \mathsf{L}(x,y)$ (for each fixed $y\in\R$) are strongly
convex. For simplicity, we use a constant step size $\eta_t \equiv \eta$.

\begin{theorem}\label{thm:cvx_dyn}
	Suppose Assumption \ref{assump:abstract} holds for some $K,\Lambda\geq 2$ and $\mathfrak{p}= 1$, and furthermore $\inf_{i \in [m]}\inf_{x,y \in \R}\partial_1 \mathfrak{S}(x,y,\xi_i)\geq 1/\Lambda$. Then the fixed point equation
	\begin{align}\label{eqn:cvx_dyn_fpe}
	\begin{cases}
	\tau_\ast= \E^{(0)}   \partial_1\mathfrak{S}\big(\frac{\delta_\ast}{\tau_\ast}\cdot\iprod{\mathsf{Z}_n}{ \mu_\ast},\iprod{\mathsf{Z}_n}{ \mu_\ast},\xi_{\pi_m}\big)\\
	\delta_\ast= -\E^{(0)}   \partial_2\mathfrak{S}\big(\frac{\delta_\ast}{\tau_\ast}\cdot \iprod{\mathsf{Z}_n}{ \mu_\ast},\iprod{\mathsf{Z}_n}{ \mu_\ast},\xi_{\pi_m}\big)
	\end{cases}
	\end{align}
	has a unique solution $(\tau_\ast,\delta_\ast)\in (0,\infty)\times [0,\infty)$. 
	
	If $\eta\leq 1/(c_0\Lambda)$ holds for some large constant $c_0>1$, then 
	\begin{align*}
	\biggpnorm{\mathrm{u}_\ast^{(t)}-\frac{\delta_\ast}{\tau_\ast}\cdot\mu_\ast}{}\leq \bigg(1-\frac{\eta}{2\Lambda}\bigg)^{t}\cdot \biggpnorm{\mu^{(0)}-\frac{\delta_\ast}{\tau_\ast}\cdot\mu_\ast}{},\quad t=1,2,\ldots.
	\end{align*}
	Moreover, if the aspect ratio $\phi\geq c_0 (K^2\vee \log m)$, then there exists another constant $c_1=c_1(c_0)>0$ such that with $\Prob^{(0)}$-probability at least $1-m^{-100}$, uniformly in $t\leq m^{100}$,
	\begin{align*}
	(n\wedge \phi)^{1/2}\pnorm{\mu^{(t)}-\mathrm{u}_\ast^{(t)}}{}+\max_{i \in [m]}\abs{\iprod{X_i}{\mu^{(t)}}}&\leq \big(K\Lambda  L_{\ast}^{(0)}\log m\cdot (\eta^{-1}\wedge t)\big)^{c_1}.
	\end{align*}
\end{theorem}

We note that at the level of general theory, $\delta_\ast/\tau_\ast$ need not be $1$. For instance, consider the squared loss $\mathsf{L}(x,y)\equiv (x-y)^2/2$ regardless of the model structure $\mathcal{F}$. Then $
\mathfrak{S}(x_0,y_0,\xi_0)\equiv x_0-\mathcal{F}(y_0,\xi_0)$, and therefore 
\begin{align*}
\partial_1 \mathfrak{S}(x_0,y_0,\xi_0)=1,\quad \partial_2 \mathfrak{S}(x_0,y_0,\xi_0)=-\partial_1\mathcal{F}(y_0,\xi_0).
\end{align*}
Consequently, the fixed point equation (\ref{eqn:cvx_dyn_fpe}) reduces to 
\begin{align*}
\tau_\ast = 1,\quad \delta_\ast= \E^{(0)}   \partial_1\mathcal{F}\big(\iprod{\mathsf{Z}_n}{ \mu_\ast},\xi_{\pi_m}\big).
\end{align*}

\subsection{Proof of Theorem \ref{thm:cvx_dyn}}

\begin{lemma}\label{lem:cvx_dyn_fpe}
	Suppose the conditions in Theorem \ref{thm:cvx_dyn} hold. 
	\begin{enumerate}
		\item The fixed point equation (\ref{eqn:cvx_dyn_fpe}) has a unique solution $(\tau_\ast,\delta_\ast)\in (0,\infty)\times [0,\infty)$.
		\item If $\eta\leq 1/\Lambda$, then
		\begin{align*}
		\biggpnorm{\mathrm{u}^{(t)}-\frac{\delta_\ast}{\tau_\ast}\cdot\mu_\ast}{}\leq \Big(1-\frac{\eta}{\Lambda}\Big)^{t/2}\cdot \biggpnorm{\mu^{(0)}-\frac{\delta_\ast}{\tau_\ast}\cdot\mu_\ast}{}. 
		\end{align*}
	\end{enumerate}
\end{lemma}
\begin{proof}
	\noindent (1). Consider the function $F:\R^n \to \R$ defined by
	\begin{align*}
	F(\mu)\equiv \E^{(0)} \mathsf{L}\big(\iprod{\mathsf{Z}_n}{\mu},\mathcal{F}(\iprod{\mathsf{Z}_n}{\mu_\ast},\xi_{\pi_m})\big).
	\end{align*}
	It is easy to compute its gradient and Hessian as
	\begin{align*}
	\nabla F(\mu)&\equiv \E^{(0)} \partial_1\mathsf{L}\big(\iprod{\mathsf{Z}_n}{\mu},\mathcal{F}(\iprod{\mathsf{Z}_n}{\mu_\ast},\xi_{\pi_m})\big) \mathsf{Z}_n \in \R^n,\\
	\nabla^2 F(\mu)&\equiv \E^{(0)} \partial_1^2\mathsf{L}\big(\iprod{\mathsf{Z}_n}{\mu},\mathcal{F}(\iprod{\mathsf{Z}_n}{\mu_\ast},\xi_{\pi_m})\big) \mathsf{Z}_n\mathsf{Z}_n^\top \in \R^{n\times n}.
	\end{align*}
	Under the condition of Theorem \ref{thm:cvx_dyn}, we have $\nabla^2 F(\mu)\geq I_n/\Lambda$. This means that $F$ is strongly convex, and  its unique minimizer $\mathrm{u}_0\in \R^n$ is given by the first order condition
	\begin{align}\label{ineq:cvx_dyn_fpe_1}
	\bm{0}_n &= \nabla F(\mathrm{u}_0) = \E^{(0)} \mathfrak{S}\big(\iprod{\mathsf{Z}_n}{\mathrm{u}_0},\iprod{\mathsf{Z}_n}{\mu_\ast},\xi_{\pi_m}\big) \mathsf{Z}_n\\
	& = \E^{(0)} \partial_1 \mathfrak{S}\big(\iprod{\mathsf{Z}_n}{\mathrm{u}_0},\iprod{\mathsf{Z}_n}{\mu_\ast},\xi_{\pi_m}\big) \mathrm{u}_0 + \E^{(0)} \partial_2 \mathfrak{S}\big(\iprod{\mathsf{Z}_n}{\mathrm{u}_0},\iprod{\mathsf{Z}_n}{\mu_\ast},\xi_{\pi_m}\big) \mu_\ast,\nonumber
	\end{align}
	where the second identity follows from an application of Gaussian integration-by-parts. Comparing the above identity to (\ref{eqn:cvx_dyn_fpe}), we may set
	\begin{align*}
	\tau_\ast& = \E^{(0)} \partial_1 \mathfrak{S}\big(\iprod{\mathsf{Z}_n}{\mathrm{u}_0},\iprod{\mathsf{Z}_n}{\mu_\ast},\xi_{\pi_m}\big),\\
	\delta_\ast&=-\E^{(0)} \partial_2 \mathfrak{S}\big(\iprod{\mathsf{Z}_n}{\mathrm{u}_0},\iprod{\mathsf{Z}_n}{\mu_\ast},\xi_{\pi_m}\big),
	\end{align*}
	which proves the existence of a solution to (\ref{eqn:cvx_dyn_fpe}). For uniqueness, for any solution $(\tau_\ast,\delta_\ast)$ to (\ref{eqn:cvx_dyn_fpe}), let $\mathrm{u}_\ast\equiv (\delta_\ast/\tau_\ast)\mu_\ast$. Then a straightforward calculation shows that $\mathrm{u}_\ast$ satisfies the first order condition (\ref{ineq:cvx_dyn_fpe_1}) which therefore must satisfy $\mathrm{u}_\ast=\mathrm{u}_0$.
	
	\noindent (2). $\mathrm{u}_\ast^{(t)}$ may be interpreted as the gradient descent on the loss function $F$: with initialization $\mathrm{u}_\ast^{(0)}=\mu^{(0)}$, for $t=1,2,\ldots$,
	\begin{align}\label{ineq:cvx_dyn_fpe_2}
	\mathrm{u}_\ast^{(t)}&= \mathrm{u}_\ast^{(t-1)}-\eta\cdot  \nabla F(\mathrm{u}_\ast^{(t-1)}). 
	\end{align}
	Now we may invoke standard convex optimization theory (cf. Lemma \ref{lem:cvx_gd_generic}) to prove linear convergence of $\pnorm{\mathrm{u}_\ast^{(t)}- \mathrm{u}_\ast}{}$.
\end{proof}

\begin{lemma}\label{lem:cvx_dyn_est}
	Suppose the conditions in Theorem \ref{thm:cvx_dyn} hold. Fix two sequences of generic vectors $\{\mathrm{u}^{(\cdot)}\},\{\mathrm{v}^{(\cdot)}\} \subset \R^n$. 
	\begin{enumerate}
		\item Suppose $\eta\leq 1/(c_0\Lambda) $ holds for some large universal constant $c_0>1$. Then 
		\begin{align*}
		\mathscr{M}_{ \mathrm{u}^{(\cdot)},\mathrm{v}^{(\cdot)} }^{(t),X}\leq 1+ \min\{\Lambda/\eta,t+1\}.
		\end{align*}
		\item Suppose that $\phi\geq c_0 K^2$ for some large universal constant $c_0>1$. Then if $\eta\leq 1/(c_1\Lambda)$ for some large universal constant $c_1>1$, on an event with $\Prob^{(0)}$-probability at least $1-2e^{-n}$,
		\begin{align*}
		\mathfrak{M}_{ \mathrm{u}^{(\cdot)},\mathrm{v}^{(\cdot)} }^{(t)}(X) \leq 1+\min\{2\Lambda/\eta,t+1\}.
		\end{align*}
	\end{enumerate}
\end{lemma}
\begin{proof}
	\noindent (1). By definition of $\mathbb{M}_{\mathrm{u},\mathrm{v} }^X$ in Definition \ref{def:G_M}-(2), uniformly in $\mathrm{u},\mathrm{v} \in \R^n$,
	\begin{align}\label{ineq:cvx_dyn_est_1}
	I_n/\Lambda= \E^{(0)}X_{\pi_m} X_{\pi_m}^\top/\Lambda \leq  \mathbb{M}_{\mathrm{u},\mathrm{v} }^X\leq c_1\Lambda\cdot  \E^{(0)}X_{\pi_m} X_{\pi_m}^\top = c_1\Lambda\cdot I_n.
	\end{align}
	Consequently, by choosing $\eta\leq 1/(2c_1\Lambda)$, we have $\bigpnorm{I_n - \eta\cdot \mathbb{M}_{\mathrm{u},\mathrm{v}}^{X} }{\op}\leq 1-\eta/\Lambda$, and therefore
	\begin{align*}
	\mathscr{M}_{ \mathrm{u}^{(\cdot)},\mathrm{v}^{(\cdot)} }^{(t),X}&\leq 1+\max_{\tau \in [0:t]} \sum_{s \in [0:\tau]} (1-\eta/\Lambda)^{\tau-s+1}\leq 1+ \min\{\Lambda/\eta,t+1\}.
	\end{align*}
	\noindent (2). Similar to the argument above, with $\hat{\Sigma}\equiv m^{-1}\sum_{i \in [m]}X_iX_i^\top$, we have $
	\hat{\Sigma}/\Lambda \leq M_{\mathrm{u},\mathrm{v} }(X)\leq c_1\Lambda\cdot  \hat{\Sigma}$. Using \cite[Theorem 4.7.1]{vershynin2018high}, if $\phi\geq (c_2 K)^2$ for some large enough universal constant $c_2>1$, then on an event $E_0$ with $\Prob^{(0)}$-probability at least $1-2e^{-n}$, $\pnorm{\hat{\Sigma}-I_n}{\op}\leq 1/2$. Consequently, by choosing $\eta\leq 1/(4c_1\Lambda)$, on the event $E_0$ we have 
	\begin{align*}
	\mathfrak{M}_{ \mathrm{u}^{(\cdot)},\mathrm{v}^{(\cdot)} }^{(t)}(X)\leq 1+\max_{\tau \in [0:t]}\sum_{s \in [1:\tau]} \big(1-\eta/(2\Lambda) \big)^{\tau-s+1}\leq  1+ \min\{2\Lambda/\eta,t+1\},
	\end{align*}
	proving the claim. 	
\end{proof}

\begin{proof}[Proof of Theorem \ref{thm:cvx_dyn}]
	The existence and uniqueness of the solution to the fixed point equation (\ref{eqn:cvx_dyn_fpe}) follows directly from Lemma \ref{lem:cvx_dyn_fpe}-(1). Applying Theorem \ref{thm:grad_des_se} with the estimates in Lemma \ref{lem:cvx_dyn_est}, for $\eta\leq 1/(c\Lambda)$, with $\Prob^{(0)}$-probability at least $1- m^{-100}$, uniformly in $t\leq m^{100}$,
	\begin{align*}
	(n\wedge \phi)^{1/2}\pnorm{\mu^{(t)}-\mathrm{u}_\ast^{(t)}}{}&\leq \big(K\Lambda  L_{\ast}^{(0)}\log m\cdot (\eta^{-1}\wedge t)\big)^{c_1}\\
	&\qquad \times \Big(1+\max_{s \in [1:t-1]} n^{1/2}\pnorm{\mathrm{u}_\ast^{(s)}}{\infty}\vee \pnorm{\mathrm{u}_X^{(s)}}{}\Big)^{c_1}.
	\end{align*}
	It is easy to prove by (\ref{def:u_ast}), (\ref{def:theo_grad_des}) and strong convexity that $n^{1/2}\pnorm{\mathrm{u}_\ast^{(t)}}{\infty}\vee \pnorm{\mathrm{u}_X^{(t)}}{}\leq (K\Lambda L_\ast^{(0)}\eta^{-1}\wedge t)^{c_1}$. The first claim now follows by Lemma \ref{lem:cvx_dyn_fpe}-(2). The second claim follows directly from Theorem \ref{thm:grad_des_incoh} and a similar uniform control of $\pnorm{\mu_{[-i]}^{(t)}}{}$ via (\ref{def:grad_des_loo}) and strong convexity. 
\end{proof}

\section{Technical tools}

We need the cumulant expansion in the following form.

\begin{lemma}\label{lem:cum_exp}
	Let $\kappa_0(Z)\equiv 1$ and let for $m=1,2,\ldots,$
	\begin{align*}
	\kappa_m(Z)\equiv (\sqrt{-1})^m\cdot \frac{\partial^m }{\partial t^m} \log \E \exp\big(\sqrt{-1}t Z\big)\big|_{t=0}.
	\end{align*}
	Then for any $\ell \in \N$ and $f\in C^{\ell+1}(\R)$, 
	\begin{align*}
	\biggabs{\E\big[Z f(Z)\big]-\sum_{m=0}^\ell \frac{\kappa_{m+1}(Z)}{m!}\E\big[f^{(m)}(Z)\big]}\leq \sum_{m=0}^{\ell+1} \abs{\kappa_{m}(Z)} \int_0^1 \bigabs{\E \big[f^{(\ell+1)}(Zs)Z^{\ell+2-m}\big]}\,\d{s}.
	\end{align*}
	In particular, if $\E Z=0$ and $\sigma^2\equiv \E Z^2$,
	\begin{align*}
	\bigabs{\E\big[Z f(Z)\big]-\sigma^2 \E\big[f'(Z)\big]}\leq \sup_{s \in [0,1]}\Big\{\bigabs{\E \big[f''(Zs) Z^3\big]}+\sigma^2 \bigabs{\E \big[f''(Zs) Z\big]}\Big\}.
	\end{align*}
\end{lemma}
\begin{proof}
	The proof uses the method in \cite[Appendix A]{he2018isotropic}. First, it is well-known that for any polynomial $P$ with $\mathrm{deg}(P)\leq \ell$, for all $0\leq m\leq \ell$,
	\begin{align}\label{ineq:cum_exp_1}
	\E\big[Z\cdot P(Z)\big]=\sum_{m=0}^\ell \frac{\kappa_{m+1}(Z)}{m!}\E\big[P^{(m)}(Z)\big].
	\end{align}
	See, e.g., \cite[Eq. (A.4)]{he2018isotropic} for a detailed proof of the above display. Now using Taylor expansion for $f(z)\equiv P_{\ell;f}(z)+R_{\ell;f}(z)$ with the integral remainder form in that $
	P_{\ell;f}(z)\equiv \sum_{m=0}^\ell \frac{f^{(m)}(0)}{m!}z^m$ and $R_{\ell;f}(z)\equiv \int_0^z \frac{f^{(\ell+1)}(t)}{\ell!}(z-t)^\ell\,\d{t}$, we have
	\begin{align}\label{ineq:cum_exp_2}
	f^{(m)}(z)-P_{\ell;f}^{(m)}(z)=\int_0^z \frac{f^{(\ell+1)}(t)}{(\ell-m)!}(z-t)^{\ell-m}\,\d{t}.
	\end{align}
	Combining the above two displays (\ref{ineq:cum_exp_1})-(\ref{ineq:cum_exp_2}), we have
	\begin{align*}
	&\biggabs{\E\big[Z f(Z)\big]-\sum_{m=0}^\ell \frac{\kappa_{m+1}(Z)}{m!}\E\big[f^{(m)}(Z)\big]}\\
	& = \biggabs{\E\big[Z \big(f(Z)-P_{\ell;f}(Z)\big)\big]- \sum_{m=0}^\ell \frac{\kappa_{m+1}(Z)}{m!}\E\big[\big(f^{(m)}(Z)-P_{\ell;f}^{(m)}(Z)\big)\big]}\\
	&\leq \biggabs{\E \int_0^1 \frac{f^{(\ell+1)}(Zs)}{\ell!}(1-s)^{\ell}Z^{\ell+2}\,\d{s}}\\
	&\qquad\qquad +\sum_{m=0}^\ell \frac{\abs{\kappa_{m+1}(Z)}}{m!} \biggabs{\E \int_0^1 \frac{f^{(\ell+1)}(Zs)}{(\ell-m)!}(1-s)^{\ell-m}Z^{\ell+1-m}\,\d{s}  }\\
	&\leq \int_0^1 \bigabs{\E \big[f^{(\ell+1)}(Zs)Z^{\ell+2}\big]} \,\d{s}+\sum_{m=1}^{\ell+1} \abs{\kappa_{m}(Z)} \int_0^1 \bigabs{\E \big[f^{(\ell+1)}(Zs)Z^{\ell+2-m}\big]}\,\d{s}.
	\end{align*}
	The claim follows by the definition of $\kappa_0(Z)=1$. 
\end{proof}

We also need the following Lindeberg's principle due to \cite{chatterjee2006generalization}.

\begin{lemma}\label{lem:lindeberg}
	Let $X=(X_1,\ldots,X_n)$ and $Y=(Y_1,\ldots,Y_n)$ be two random vectors in $\R^n$ with independent components such that $\E X_i^\ell=\E Y_i^\ell$ for $i \in [n]$ and $\ell=1,2$. Then for any $f \in C^3(\R^n)$,
	\begin{align*}
	\bigabs{\E f(X) - \E f(Y)}&\leq \max_{U_i \in \{X_i,Y_i\}}\biggabs{\sum_{i=1}^n\E U_i^3 \int_0^{1} \partial_i^3 f(X_{[1:(i-1)]},tU_i, Y_{[(i+1):n]} )(1-t)^2\,\d{t}}.
	\end{align*}
\end{lemma}

The following is a standard convergence result for strongly convex functions from convex optimization theory (cf. \cite{boyd2004convex}) . We include some details for the convenience of the readers.

\begin{lemma}\label{lem:cvx_gd_generic}
Suppose $F: \R^n\to \R$ is a $C^2$ strongly convex function with $\inf_{x \in \R^n}\lambda_{\min}(\nabla^2 F(x))\geq \mu>0$ and $\sup_{x \in \R^n}\pnorm{\nabla^2 F(x)}{\op}\leq L$. Consider the gradient descent: with initialization $x^{(0)}\in \R^n$ and step size $\eta$,
\begin{align*}
x^{(t)}=x^{(t-1)}-\eta \nabla F(x^{(t-1)}),\quad t = 1,2,\ldots.
\end{align*}
Then with $x_\ast \in \R^n$ being the global minimizer of $F$, if $\eta\leq 1/L$, we have
\begin{align*}
\pnorm{x^{(t)}-x_\ast}{}^2\leq \big(1-\eta\mu\big)^t\cdot \pnorm{x^{(0)}-x_\ast}{}^2,\quad t=1,2,\ldots.
\end{align*}

\end{lemma}
\begin{proof}
By the gradient update rule,  
\begin{align}\label{ineq:cvx_gd_generic_1}
&\pnorm{x^{(t)}- x_\ast}{}^2  = \pnorm{x^{(t-1)}-\eta\cdot  \nabla F(x^{(t-1)})-x_\ast}{}^2\nonumber\\
& = \pnorm{x_\ast^{(t-1)}- x_\ast}{}^2-2\eta\cdot \bigiprod{\nabla F(x^{(t-1)})}{x^{(t-1)}- x_\ast}+\eta^2\cdot \pnorm{ \nabla F(x^{(t-1)})}{}^2.
\end{align}
For the second term on the last display of (\ref{ineq:cvx_gd_generic_1}), by strong convexity of $F$, 
\begin{align}\label{ineq:cvx_gd_generic_2}
\bigiprod{\nabla F(x)}{x- x_\ast}\geq F(x)-F(x_\ast)+ (\mu/2) \pnorm{x-x_\ast}{}^2.
\end{align}
For the third term on the last display of (\ref{ineq:cvx_gd_generic_1}), we will use the generic estimate
\begin{align}\label{ineq:cvx_gd_generic_3}
\pnorm{\nabla F(x)}{}^2\leq 2L\cdot \big(F(x)-F(x_\ast)\big).
\end{align}
To see (\ref{ineq:cvx_gd_generic_3}), by the optimality of $x_\ast$ and the fact that $\pnorm{\nabla^2 F(x)}{\op}\leq L$,  using Taylor expansion of $F$ at $x- L^{-1}\nabla F(x)$ up to the second order,
\begin{align*}
F(x_\ast)\leq F \big(x- L^{-1}\nabla F(x)\big)\leq F(x)- L^{-1} \pnorm{\nabla F(x)}{}^2+(2L)^{-1} \pnorm{\nabla F(x)}{}^2,
\end{align*}
proving (\ref{ineq:cvx_gd_generic_3}). Now plugging (\ref{ineq:cvx_gd_generic_2}) and (\ref{ineq:cvx_gd_generic_3}) into (\ref{ineq:cvx_gd_generic_1}) with $x=x^{(t-1)}$, if $\eta\leq 1/L$,
\begin{align*}
\pnorm{x^{(t)}- x_\ast}{}^2 &\leq (1-\eta\mu)\cdot \pnorm{x^{(t-1)}- x_\ast}{}^2-2(\eta-L \eta^2)\cdot  \big(F(x^{(t-1)})-F(x_\ast)\big)\\
&\leq (1-\eta\mu)\cdot \pnorm{x^{(t-1)}- x_\ast}{}^2,
\end{align*}
as desired. 
\end{proof}

\section*{Acknowledgments}
The research of Q. Han is partially supported by NSF grant DMS-2143468.

\bibliographystyle{alpha}
\bibliography{mybib}

\end{document}